\documentclass{article}

\usepackage[numbers,sort&compress]{natbib}
\setcitestyle{square,comma}

\usepackage[main, final]{neurips_2026}

\usepackage[utf8]{inputenc} 
\usepackage[T1]{fontenc}    
\usepackage{hyperref}       
\usepackage{url}            
\usepackage{booktabs}       
\usepackage{amsfonts}       
\usepackage{nicefrac}       
\usepackage{microtype}      
\usepackage{xcolor}         
\usepackage{wrapfig}
\usepackage{graphicx}
\usepackage{subcaption}
\usepackage{multirow}
\usepackage{amsmath}
\usepackage{amssymb}
\usepackage{mathtools}
\usepackage{amsthm}
\usepackage{enumitem}
\usepackage{thm-restate}
\usepackage[most]{tcolorbox}
\usepackage{cleveref}
\usepackage{algorithm}
\usepackage{algorithmic}
\usepackage{tabularx}
\usepackage{caption}
\usepackage{titletoc}

\hypersetup{
  colorlinks=true,
  citecolor=blue,
  linkcolor=blue,
  urlcolor=blue
}

\usepackage{pifont}
\newcommand{\cmark}{\ding{51}}
\newcommand{\xmark}{\ding{55}}

\ifodd 1

\newcommand{\com}[1]{\textbf{\color{red}(comment: #1)}} 
\newcommand{\res}[1]{\textbf{\color{magenta}(RESPONSE: #1)}} 
\else

\newcommand{\com}[1]{}
\newcommand{\res}[1]{}
\fi

\theoremstyle{plain}
\newtheorem{theorem}{Theorem}[section]
\newtheorem{lemma}[theorem]{Lemma}
\newtheorem{proposition}[theorem]{Proposition}
\newtheorem{corollary}[theorem]{Corollary}

\theoremstyle{definition}
\newtheorem{definition}[theorem]{Definition}
\newtheorem{assumption}[theorem]{Assumption}
\newtheorem{remark}[theorem]{Remark}

\tcbset{
  graythm/.style={
    enhanced,
    breakable,
    colback=gray!12,
    colframe=gray!12,
    boxrule=0pt,
    arc=6pt,
    left=6pt,
    right=6pt,
    top=2pt,
    bottom=2pt,
    before skip=4pt,
    after skip=4pt
  }
}

\newcolumntype{Y}[1]{%
  >{\hsize=#1\hsize
    \linewidth=\hsize
    \raggedright\arraybackslash}X%
}

\title{Self-Consuming Generative Models with Co-Evolving Human Preferences}

\author{%
 Xiukun Wei$^{1*}$, \quad Tian Xie$^{1*}$, \quad Ding Zhu$^1$, \quad Xueru Zhang$^1$ \\
  $^1$The Ohio State University\\
\texttt{\{wei.1418, xie.1379, zhu.3723, zhang.12807\}@osu.edu}\\
}

\begin{document}

\renewcommand{\thefootnote}{\fnsymbol{footnote}}
\footnotetext[1]{Equal contribution.}
\renewcommand{\thefootnote}{\arabic{footnote}}

\maketitle

\begin{abstract}
    Generative models are increasingly trained in self-consuming iterative loops, where users curate preferred samples from model-generated candidates and the curated samples are used to train future generations of the model. Prior work has largely assumed fixed user preferences, but in practice exposure to model outputs gradually reshapes what users perceive as desirable, creating a feedback loop in which model distributions and user preferences co-evolve. We take a first step toward understanding the long-term behavior of such coupled dynamics. We show that when training relies entirely on user-curated synthetic data, iterative curation amplifies initial biases and drives the system toward one of multiple \textit{singleton} equilibria in which the instance holding an initial advantage eventually dominates.  In contrast, injecting reference data into training at a sufficiently large rate fundamentally changes the dynamics and yields a \textit{unique globally attracting} equilibrium. Building on this insight, we study how reference-data injection can be used to control long-term outcomes, and propose an efficient algorithm that jointly selects a reference distribution and its mixing weight to steer the coupled system toward equilibria that preserve desired attributes while minimizing data collection costs.
\end{abstract}

\section{Introduction}

Modern generative models are increasingly trained on data that contains synthetic outputs produced by earlier versions of themselves, creating a \textit{self-consuming} training loop. Recent studies have shown that such loops can lead to model collapse, training instability, and bias amplification~\citep{AICollapse,Satbility,CollapseInevitable,StatisticalAnalysisLM,GoMad,TheoreticalPerspective,wang2026observations}, and have proposed various mitigation strategies~\citep{cai2026stabilizing,self_correction,CollapseInevitable}. In many practical settings, the synthetic data reused for retraining is further filtered through human curation~\citep{Verification}: model-generated samples are presented to users, who express preferences by selecting, rating, editing, reposting, or otherwise amplifying certain outputs over others, and the retained samples are subsequently incorporated into future training datasets. Prior work has shown that repeated retraining on such human-curated synthetic data can implicitly steer models toward distributions aligned with user preferences~\citep{Curation,AdversarialCuration}.  

However, existing works largely treat the preference signal as  \textbf{fixed} or \textbf{exogenous}: users are assumed to consistently favor the same attributes, and human curation simply filters samples accordingly. In practice, this assumption is often violated because users and the model interact in a closed loop: the model itself reshapes the environment from which future preferences are formed, as its outputs shape what users are exposed to, which in turn shapes what they find desirable. Such model-influenced preference shifts are supported by recent empirical evidence: LLM-assisted writing systematically alters the tone, vocabulary, and meaning of human-produced text~\citep{Distort}; exposure to biased AI writing assistants shifts users' subsequent opinions on the underlying topic~\citep{jakesch2023co}; and AI-assisted creation more broadly homogenizes the diversity of content users produce~\citep{doshi2024generative}. In each case, synthetic outputs are not merely consumed by future models; they reshape the preference landscape that guides future curation, creating a closed feedback loop in which the model and its users co-evolve.

\begin{figure}[t]
\begin{center}
\centerline{\includegraphics[width=\textwidth]{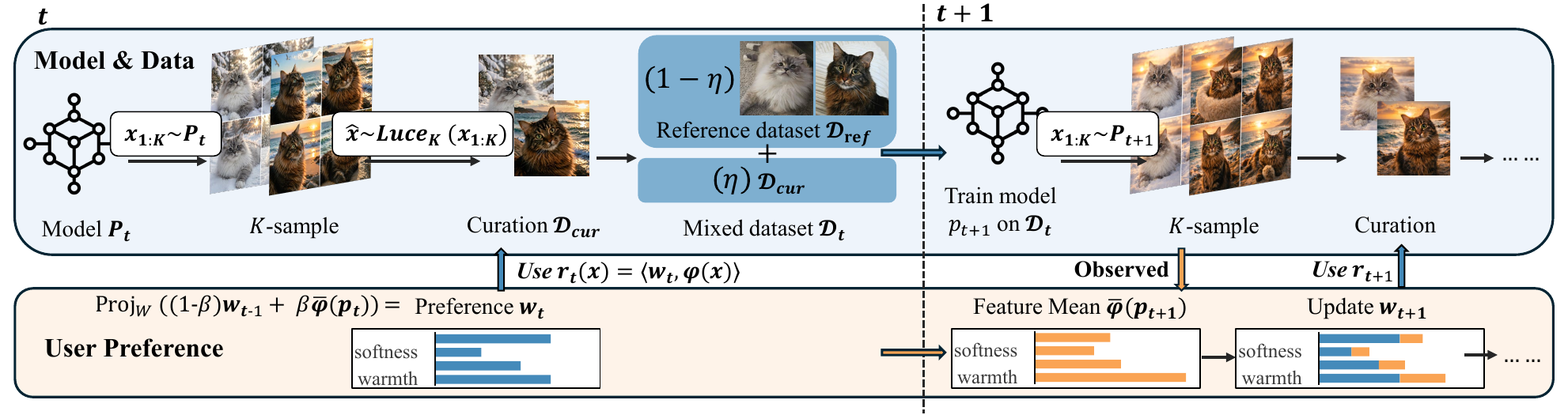}}
\caption{Overview of the coupled model-user dynamics.
\textbf{Top row (model update):} the current model $p_t$ generates $K$ candidates, of which one is retained via preference-guided $K$-way Luce choice; the curated set $\mathcal{D}_{\mathrm{cur}}$ is mixed with a reference dataset $\mathcal{D}_{\mathrm{ref}}$ at rate $\eta$ to form $\mathcal{D}_t$, on which $p_{t+1}$ is trained.
\textbf{Bottom row (preference update):} the user assesses how each candidate aligns with the current preference $w_t$ via the reward  $r_t(x)=\langle w_t,\varphi(x)\rangle$, where $\varphi$ is the feature map, and the preference vector then drifts toward the mean feature direction $\bar{\varphi}(p_{t+1})$ of the updated model.}
\vskip -0.3in
\label{Fig:framework}
\end{center}
\end{figure}

This paper sets out to study the long-term behavior of coupled user-model dynamics. We formulate a dynamical system over the joint state $(p_t,w_t)$ in which the model distribution $p_t$ and the underlying user preference vector $w_t$ co-evolve over discrete time steps, as illustrated in Fig.~\ref{Fig:framework}. At each round $t$, the model generates $K$ candidates and the user curates one according to their current preferences, with samples whose features align more closely with $w_t$ selected with higher probability. The model is then retrained on the curated synthetic samples, optionally mixed with a fixed reference distribution, while the user preferences simultaneously drift toward the feature direction of the updated model. 

Our analysis shows that the system exhibits qualitatively different long-run behavior depending on whether reference data is included in training. 
In the purely synthetic regime, where retraining relies entirely on user-curated synthetic samples, repeated curation amplifies local reward advantages: any item that is even slightly preferred under the current reward eventually dominates, driving the system toward one of multiple singleton equilibria.
In contrast, when a sufficient fraction of reference data is mixed in at each round, the coupled update becomes a strict contraction on the joint state space under an explicit condition on the mixing weight, implying a unique globally attracting equilibrium that every trajectory converges to geometrically, regardless of initialization.

The uniqueness of the equilibrium under reference-data mixing opens a natural design question: how should the reference distribution and mixing weight be jointly chosen so that a specified level of diversity is preserved at equilibrium at minimum data-collection cost?
We formalize this as a constrained optimization problem, which is bilevel and intractable since the constraint involves the equilibrium of a nonlinear coupled map with no closed form. To tackle this, we identify a worst-case relaxation that yields a linear sufficient condition on the reference distribution, decoupling the bilevel structure and reducing the problem to a family of linear programs indexed by the mixing weight. Experiments verify the theoretical results and demonstrate the effectiveness of the proposed algorithm.

\begin{table*}[htbp]
\centering
\footnotesize
\setlength{\tabcolsep}{3pt}
\renewcommand{\arraystretch}{1.15}
\caption{Comparison with prior work on self-consuming generative models with curation.}
\label{tab:comparison}
\begin{tabularx}{\textwidth}{
  @{} l Y{0.75} c Y{1.50} Y{0.85} Y{0.90} @{}
}
\toprule
\textbf{Work} 
& \textbf{Preference} 
& \textbf{\# Models} 
& \textbf{Core conclusion} 
& \textbf{Risk} 
& \textbf{Role of real data} \\
\midrule
\citet{Curation}
& Fixed
& Single
& Curation implicitly optimizes a fixed preference signal.
& Bias amplification
& Stabilizer \\
\citet{AdversarialCuration}
& Fixed
& Single
& Adversarial curation can systematically misalign the model.
& Malicious feedback
& Improves robustness \\
\citet{HeterogeneousHumanCuration}
& Heterogeneous fixed
& Single
& Reference mixing yields convergence and stability under noisy heterogeneous curation.
& Reward perturbation instability
& Stabilizer and regularizer \\
\citet{WhenandHow}
& Fixed
& Multi
& Benign curation may backfire under model interaction.
& Cross-model coupling
& Convergence aid\\
\midrule
\textbf{This work}
& \textbf{Co-evolving}
& \textbf{Single}
& \textbf{Preferences co-evolve with the model, altering equilibrium structure.}
& \textbf{Preference drift; diversity collapse}
& \textbf{Stabilizer; Rreference data design} \\
\bottomrule
\end{tabularx}
\end{table*}
Tbl.~\ref{tab:comparison} summarizes how our paper extends the existing line of self-consuming generative models with human curation. In Appx. \ref{sec:relatedwork}, we discuss more related work, and our contributions can be summarized as follows:
\begin{itemize}[leftmargin=*,topsep=0cm,itemsep=0cm]
    \item We propose a dynamic model that captures the co-evolution of user preferences and the model under self-consuming iterative training (Sec.~\ref{sec:problem}).  
    \item We theoretically analyze the convergence of the coupled dynamics and characterize the properties of its equilibria, both with and without reference data mixed into the retraining loop (Sec.~\ref{sec:long-term}).     
    \item We formulate equilibrium steering as a constrained optimization problem and show that a worst-case sufficient condition reduces it to a family of linear programs, providing a computationally lightweight algorithm for jointly choosing the reference distribution and mixing weight (Sec.~\ref{sec:mitigation}).
    \item We empirically verify the theoretical results and the effectiveness of the proposed algorithm on both image generation (CFM and DDPM on CIFAR-10) and text generation (GPT-2 on AG~News). (Sec.~\ref{sec:experiments})
    
\end{itemize}
\section{Problem Formulation}\label{sec:problem}

Consider a platform that iteratively trains generative models using data curated by human users. Let $\mathcal{X} \subset \mathbb{R}^m$ be the instance space and let
$p_t \in \mathcal{P}(\mathcal{X})$ denote the output distribution of the generative model at round $t$ of the retraining loop. 
We present the model and results in the discrete setting with finite $\mathcal{X} = \{x_1,\dots,x_n\}$, motivated by the fact that in many platforms, outputs fall into a finite number of distinguishable categories (e.g., topic or object classes); the extension to continuous spaces is deferred to Appx.~\ref{sec:continuous}.

\textbf{User preference \& human curated data.} At each round $t$, the platform presents $K \ge 2$ synthetic samples $\{x_1, \dots, x_K\}$ drawn from the current model $p_t$ to users, who then select a preferred sample $\hat{x}$ that is used to train the next generation of models. Specifically, let $\varphi:\mathcal{X} \to \mathbb{R}^d$ be a fixed feature map satisfying $\|\varphi(x)\| = 1$, which extracts features of an instance $x \in \mathcal{X}$ relevant to user preferences as a unit-norm representation\footnote{If $x$ is text, $\varphi(x)$ could encode tone, topic, or style; if $x$ is an image, it may encode color, sharpness, or semantic content. See Appx.~\ref{sec:example} for complete worked examples.}. User preferences over these features at round $t$ are characterized by a vector $w_t \in \mathcal{W} := \{w \in \mathbb{R}^d : \|w\| \le 1\}$. Given this preference vector, the reward of $x$ at round $t$ is defined as
$
    r_{w_t}(x) := r(x;w_t) := \langle w_t,\varphi(x)\rangle
\in[-1,1]$, which measures the alignment between the instance features and the user’s preference direction. Based on these rewards, users select $\hat{x}$ from the candidate set $\{x_1, \dots, x_K\}$ according to the $K$-way Luce rule:
\begin{equation}\label{eq:luce}
    \Pr(\hat x = x_k \mid x_1,\dots,x_K)
    =
    \frac{e^{\tau \cdot r_{w_t}(x_k)}}{\sum_{j=1}^K e^{\tau \cdot r_{w_t}(x_j)}}
\end{equation}
where $\tau>0$ controls the sharpness of user choice. When $K=2$, this reduces to Bradley-Terry model \citep{Luce,BradleyTerry}. We denote this sampling process by $\hat{x} \sim \mathrm{Luce}_K(x_1, \dots, x_K)$.

\textbf{Recursive training loop.} Given a fixed reference (real data) distribution $p_{\mathrm{ref}}\in\mathcal{P}(\mathcal{X})$ and user curated synthetic data, the platform updates its model at round $t+1$ using either only user-curated synthetic data ($\eta=1$) or a mixture of synthetic and reference data $\eta\in(0,1)$: 
\begin{equation}
    p_{t+1}
    \in
    \arg\max_{p\in\mathcal{P}(\mathcal{X})}
    (1-\eta)\,\mathbb{E}_{x\sim p_{\mathrm{ref}}}[\log p(x)]
    +
    \eta
    \mathbb{E}_{\substack{x_1,\dots,x_K\sim p_t \\ \hat x\sim \mathrm{Luce}_K(x_1,\dots,x_K)}}
    [\log p(\hat x)]
\label{eq:retrain}
\end{equation}
where mixing parameter $\eta\in[0,1]$ controls the fraction of synthetic data used in training. Following \citet{Curation}, we show that the retraining step~\eqref{eq:retrain} admits a multiplicative reweighting of $p_t$ (proof in Appx.~\ref{subsec:prof_explicit-update}):
\begin{equation} 
\label{eq:data-update}
        p_{t+1}(x)
        =
        (1-\eta)p_{\mathrm{ref}}(x)
        +
        \eta p_t(x) H^K_{p_t,w_t}(x),
  \end{equation}
  where $H^K_{p_t,w_t}(x)$ is referred to as the \textit{$K$-way curation operator}, which quantifies the amplification of $x$ induced by $K$-way preference-based selection from $p_t$.
  \begin{definition}[Curation operator]
  \label{def:curation-operator}
The $K$-way curation operator $H^K_{p_t}:\mathcal{X}\to \mathbb{R}_+$ is defined as
\[
H^K_{p_t,w_t}(x):=
            \mathbb{E}_{x_1,\dots,x_{K-1}\sim p_t}\left[\frac{K e^{\tau \cdot r_{w_t}(x)}}{e^{\tau \cdot r_{w_t}(x)}+\sum_{i=1}^{K-1}e^{\tau \cdot  r_{w_t}(x_i)}}\right]
\]
\end{definition}
Moreover, for any fixed $x\in \mathcal{X}$, this reweighting converges, as $K \to \infty$, to an exponential tilt of $p_t$, i.e.,
    $
    H^K_{p_t,w_t}(x)\xrightarrow{K\to\infty}
    \frac{e^{\tau \cdot r_{w_t}(x)}}{\mathbb{E}_{z\sim p_t}[e^{\tau \cdot r_{w_t}(z)}]}.
    $

\textbf{Preference drift \& coupled dynamics.} When users are exposed to model-generated content on the platform, their preferences may evolve over time. To capture this effect, we model the preference vector $w_t$ as being gradually steered toward the feature distribution induced by the model. Specifically, define $\bar{\varphi}(p) := \mathbb{E}_{x \sim p}[\varphi(x)]$ as the mean feature vector under distribution $p$. After retraining at round $t+1$, user preferences are updated according to
\begin{equation}
    w_{t+1}
    =
    \operatorname{Proj}_{\mathcal{W}}\left((1-\beta)w_t + \beta \bar\varphi(p_{t+1})\right)
\label{eq:pref-update}
\end{equation}
where $\beta\in(0,1]$ is the preference adaptation rate and $\operatorname{Proj}_{\mathcal{W}}$ denotes Euclidean projection onto $\mathcal{W}$.

Eq.~\eqref{eq:data-update} and \eqref{eq:pref-update} together define a closed-loop dynamical system over the state $(p_t, w_t)$. We define an \textbf{equilibrium} $(p_K^\star, w_K^\star)$ as a fixed point of the joint dynamics, i.e., a state that remains unchanged under data generation and preference adaptation. Formally, $(p_K^\star, w_K^\star)$ satisfies 
\begin{equation}
    p^\star_K(x)
    =
    (1-\eta)p_{\mathrm{ref}}(x)
    +
    \eta p^\star_K(x) H^K_{p^\star_K,w^\star_K}(x)
    \qquad
    w^\star_K
    =
    \operatorname{Proj}_{\mathcal{W}}\left((1-\beta)w^\star_K + \beta \bar\varphi(p^\star_K)\right)
\label{eq:equilibrium}
\end{equation}

\textbf{Objective.} We study the evolution of the coupled dynamics~\eqref{eq:data-update}-\eqref{eq:pref-update} and characterize the long-run behavior of the model $p_t$ and user preferences $w_t$ (Sec.~\ref{sec:long-term}). Specifically, we ask: under what conditions does the system converge and admit an equilibrium; how do the choice set size $K$, the synthetic-data weight $\eta$, and the geometry of the feature map $\varphi$ govern this behavior; and when do the dynamics collapse toward degenerate outcomes versus stabilize under reinjection of reference data.

Beyond stability analysis, we investigate the reference distribution $p_{\mathrm{ref}}$ as a steering mechanism. We ask how $p_{\mathrm{ref}}$ and the synthetic-data weight $\eta$ can be jointly chosen so that the resulting equilibrium preserves desired attributes while keeping the cost of intervention small (Sec.~\ref{sec:mitigation}).
\section{Long-Term Behavior of User--Model Dynamics}
\label{sec:long-term}
Next, we study the evolution of $(p_t, w_t)$. We begin with the purely synthetic curation setting ($\eta = 1$), and then extend to the setting in which each round mixes synthetic and reference samples ($\eta \in (0,1)$).

\subsection{Retraining with purely synthetic user-curated data}
\label{subsec:pure}
When \(\eta=1\), update \eqref{eq:data-update} reduces to $
    p_{t+1}(x)=p_t(x)H^K_{p_t,w_t}(x)
$. In this case, the dynamics are driven entirely by repeated preference-guided reweighting of the current distribution via the curation operator $H^K_{p_t,w_t}$. Before analyzing the full dynamics, we first study $H^K_{p_t,w_t}$ and identify its key properties.

\begin{restatable}[Basic properties of curation operator]
{proposition}{PureBasic}
\label{prop:pure-basic}
  For any $p \in \mathcal{P}(\mathcal{X})$, $w \in \mathcal{W}$, and $K \ge 2$,  the curation operator $H^K_{p,w}$ satisfies the following properties:
    \begin{enumerate}[leftmargin=*,topsep=-0.1cm,itemsep=-0.05cm]
        \item \textbf{Strict monotonicity.} For any $x_i,x_j\in\mathcal X$, $r_w(x_i)>r_w(x_j) \Longleftrightarrow H^K_{p,w}(x_i)>H^K_{p,w}(x_j)$

        \item \textbf{Normalization.} $\sum_{i=1}^np(x_i)H^K_{p,w}(x_i)=1$.

        \item \textbf{Odds contraction.} If $r_w(x_i)-r_w(x_j)\ge r_\Delta>0$, then
        $$
            \frac{H^K_{p,w}(x_j)}{H^K_{p,w}(x_i)}
            \le
            \lambda_K(r_\Delta)
            :=
            1-\frac{(1-e^{-\tau \cdot r_\Delta})(K-1)}{e^{2\tau}+K-1},
            $$
        with $\lambda_K(r_\Delta)\in(0,1)$ and $ \lambda_K(r_\Delta)\xrightarrow{K\to\infty}e^{-\tau \cdot r_\Delta}$
    \end{enumerate}
\end{restatable}
See proof in Appx.~\ref{subsec:prof_pure-basic}. \textit{Strict monotonicity} implies that $H^K_{p,w}$ preserves the reward ordering induced by user preferences, ensuring higher-reward instances always receive higher selection weight. Moreover, \textit{odds contraction} implies lower-reward items are disproportionately suppressed relative to higher-reward ones, leading to an increasingly sharp separation in their selection probabilities. As a result, curation does not merely favor better items at each round; instead, it compounds their advantage over time by repeatedly amplifying higher-reward instances and pushing weaker ones further behind.

\textbf{Equilibrium and convergence analysis.}
Building on Prop.~\ref{prop:pure-basic}, we next analyze the dynamics. Since each retraining step shifts mass toward higher-reward
instances at the expense of lower-reward ones, so any distribution that places positive
mass on instances with differing rewards is unstable.

\begin{tcolorbox}[graythm]
\begin{restatable}[Equilibrium characterization under purely synthetic setting]{theorem}{PureStructure}
\label{thm:pure-structure}
When $\eta=1$, a distribution $p\in\mathcal{P}(\mathcal{X})$ is a fixed point of 
 the
retraining step  under preference $w$ if and only if all instances in its support share the
same reward:
$$
p(x)>0,\; p(x')>0 \;\Longrightarrow\; r_w(x)=r_w(x').
$$
Consequently, for every instance $x_i\in \mathcal{X}$, the pair
$\left(
p^\star=\delta_{x_i},w^\star=\varphi(x_i)\right)$
is an equilibrium: the model has collapsed to a point mass on $x_i$ while user
preferences have locked onto its features.
\end{restatable}
\end{tcolorbox}
See proof in Appx.~\ref{subsec:prof_pure-structure}, where $\delta_{x_i}$ denotes the
Dirac delta centered at ${x_i}$. Thm.~\ref{thm:pure-structure} shows that training exclusively on user-curated synthetic data admits
$n$ distinct singleton equilibria, one centered at each instance $x_i$, in which the model concentrates all mass on a single instance and user preferences align entirely with its features. Since every instance supports a singleton equilibrium, the natural question is: \textit{which of these equilibria are locally attracting?} The next theorem shows that any instance with a sufficient initial advantage will eventually dominate,  a phenomenon we term "winner lock-in".

\begin{tcolorbox}[graythm]
\begin{restatable}[Local convergence to singleton equilibria]{theorem}{PureLockin}
\label{thm:pure-lockin}
    Fix $x_i\in\mathcal{X}$ and suppose $\varphi(x_j)\neq \varphi(x_i)$ for all $j\neq i$. Define the reward gap
    $
    \gamma_i
    =
    \min_{j\neq i}
    \langle \varphi(x_i),\,\varphi(x_i)-\varphi(x_j)\rangle
    >0
    $ and let $\lambda_K(\cdot)$ be the odds-contraction factor from Prop.~\ref{prop:pure-basic}. If the initial state satisfies
$$
    1-p_0(x_i)\le \varepsilon_i,
    \qquad
    \|w_0-\varphi(x_i)\|\le \frac{\gamma_i}{4}
    $$
for some $\varepsilon_i\in\left(0,\min\left\{\frac12,\ 1-\lambda_K(\frac{\gamma_i}{2}),\ \frac{\gamma_i}{8}\right\}\right)
    $, then for all $t\ge 0$,
    $$
    1-p_t(x_i)\le \left(\frac{\lambda_K(\frac{\gamma_i}{2})}{1-\varepsilon_i}\right)^t\bigl(1-p_0(x_i)\bigr).
    $$
    Consequently, $p_t(x_i)\xrightarrow{t\to\infty} 1$ and $w_t \xrightarrow{t\to\infty} \varphi(x_i).
    $
\end{restatable}
\end{tcolorbox}

See proof in Appx.~\ref{subsec:prof_pure-lockin}. Thm.~\ref{thm:pure-lockin} shows that if the initial distribution and preferences are sufficiently biased toward $x_i$, this bias is self-reinforcing: the mass on all other instances $1-p_t(x_i)$ decays geometrically and $p_t$ converges to the singleton equilibrium $\delta_{x_i}$.  Reward gap $\gamma_i$ 
is the minimum reward advantage that $x_i$ holds over any competitor $x_j$ when preferences $w$ are fully aligned with $\varphi(x_i)$: a larger gap expands the set of initial conditions from which convergence to  $\delta_{x_i}$ is guaranteed.

\begin{remark}[Role of the choice set size $K$]
Thm.~\ref{thm:pure-lockin} holds for every $K \ge 2$, but $K$ governs the speed of convergence. Since $\lambda_K$  decreases with $K$, larger choice sets make curation more selective and accelerate the decay of mass on competing instances. In the limit $K \to \infty$, the update approaches exponential tilting, the most aggressive form of preference-guided reweighting.
\end{remark}

Note that the condition $\varphi(x_j) \neq \varphi(x_i)$ in Thm.~\ref{thm:pure-lockin} is necessary. When two distinct instances share the same feature direction, curation cannot distinguish between them on the basis of rewards alone, and transferring mass from $x_i$ to such a duplicate leaves both the reward geometry and the induced preference direction unchanged. Local stability therefore fails, as we show below. 

\begin{restatable}[Failure of local stability]{corollary}{PureDuplicate}
\label{cor:pure-duplicate}
    Fix $x_i\in\mathcal{X}$. If there exists $x_j$ with $j \neq i$ such that $\varphi(x_j)=\varphi(x_i)$, then the singleton equilibrium $\delta_{x_i}$ is not locally asymptotically stable.
\end{restatable}
See proof in Appx.~\ref{subsec:prof_pure-duplicate}. Thm.~\ref{thm:pure-structure}, Thm.~\ref{thm:pure-lockin} and Cor.~\ref{cor:pure-duplicate} together characterize the long-term behavior of the dynamics when models are retrained exclusively on human-curated synthetic data: it admits $n$ singleton equilibria, one for each instance; under generic feature separation, any instance with a sufficient initial advantage in model mass and preference alignment locks in as the unique long-run outcome; and when two instances share the same feature direction, the local stability fails.

\subsection{Retraining on mixed synthetic and reference data}
\label{subsec:reinjection}
We next consider $\eta\in(0,1)$, where each retraining round mixes curated synthetic data with a fixed reference distribution $p_{\mathrm{ref}}$ according to~\eqref{eq:data-update}. Unlike  purely synthetic setting, $p_{\mathrm{ref}}$ prevents the dynamics from collapsing to a single instance and instead drives them toward a globally stable equilibrium.

To analyze convergence and establish contraction of the dynamics, the following lemma quantifies how perturbations in the model distribution $p$ and user preference $w$ propagate through the retraining step $p \to (1-\eta)p + \eta p H^K_{p,w}$ and serves as the key technical tool for the convergence results below.
 
\begin{restatable}[Sensitivity of the retraining step]{lemma}{ReinjLip}
\label{prop:reinj-lipschitz} Let $
d_{\mathrm{TV}}(p,q):=\frac12\sum_{x\in\mathcal X}|p(x)-q(x)|
$ denote total variation distance.
    For any $p,q\in\mathcal P(\mathcal X)$, $w,v\in\mathcal W$, and $K\geq 2$, the following hold:
    \begin{align*}
   \resizebox{\hsize}{!}{$   d_{\mathrm{TV}}\left(p H_{p,w}^K,qH_{q,w}^K\right) \le L_p d_{\mathrm{TV}}(p,q);~~~
d_{\mathrm{TV}}\left(pH_{p,w}^K,p H_{p,v}^K\right) \le L_w\|w-v\|; ~~\|\bar\varphi(p)-\bar\varphi(q)\|\le 2 d_{\mathrm{TV}}(p,q)$}
    \end{align*}
    where
    $
    L_p:=2e^{2\tau}(2+e^{2\tau})$ and
    $
    L_w:=2\tau e^{4\tau}
    $.
\end{restatable}

See proof in Appx.~\ref{subsec:prof_reinj-lipschitz}. Lemma~\ref{prop:reinj-lipschitz} implies that the retraining step is Lipschitz continuous in both $p$ and $w$: the first bound controls how perturbations in $p$ propagate through the curation step; the second controls the sensitivity to changes in $w$; and the third shows that distributional perturbations translate to bounded perturbations in the mean feature vector entering the preference update~\eqref{eq:pref-update}.

\textbf{Equilibrium and convergence analysis.} 
The coupled system~\eqref{eq:data-update}-\eqref{eq:pref-update} is difficult to analyze directly: the model update depends on the current preference $w_t$, while the preference update depends on the updated distribution $p_{t+1}$. We therefore begin by fixing $w$ and studying the model update in isolation. This decoupling breaks the circularity and allows us to characterize the long-term model distribution for any given $w$. Lem.~\ref{thm:frozen-w} below shows that, unlike the purely synthetic setting which admits $n$ singleton equilibria with long-run outcome determined by initial conditions (Thm.~\ref{thm:pure-structure}), mixing with $p_{\mathrm{ref}}$ yields a \textbf{unique} long-run distribution, provided $\eta$ is small enough to make the update contractive.

\begin{tcolorbox}[graythm]
\begin{restatable}[Unique equilibrium under frozen-$w$]{lemma}{FrozenW}
\label{thm:frozen-w} Fix $w\in\mathcal W$, and consider the the frozen-$w$ model update
$p_{t+1}(x)
=(1-\eta)p_{\mathrm{ref}}(x)+\eta p(x)H^K_{p,w}(x)$. If $\eta L_p < 1$, the update is a strict contraction on  $(\mathcal P(\mathcal X),d_{\mathrm{TV}})$. Consequently, there exists a unique equilibrium $p_w^{K,\eta}$ and every trajectory converges to it geometrically:
$$
d_{\mathrm{TV}}(p_t,p_w^{K,\eta})
\le
(\eta L_p)^t\,d_{\mathrm{TV}}(p_0,p_w^{K,\eta}).
$$
\end{restatable}
\end{tcolorbox}

See proof in Appx.~\ref{subsec:prof_frozen-w}. In general, the equilibrium $p_w^{K,\eta}$
must be computed iteratively. In the large-$K$ limit, however, it admits a closed-form expression.

\begin{restatable}[Closed-form equilibrium]{proposition}{KinftyClosedForm}
\label{prop:kinfty-closed-form}
As \(K\to\infty\), the frozen-\(w\) equilibrium takes the form
$$
p_w^{\infty,\eta}(x)
=
(1-\eta)p_{\mathrm{ref}}(x)\,\frac{Z_w^\eta}{Z_w^\eta-\eta {e^{\tau \cdot r_w(x)}}}.
$$
where $Z_w^\eta := \mathbb E_{z\sim p_w^{\infty,\eta}}[e^{\tau \cdot r_w(x)}]$ is the unique scalar satisfying
$$
\sum_{i=1}^n
(1-\eta)\,p_{\mathrm{ref}}(x_i)\,
\frac{e^{\tau \cdot r_w(x_i)}}{Z_w^\eta-\eta e^{\tau \cdot r_w(x_i)}}
=1, ~~~~~Z_w^\eta>\eta\max_i e^{\tau \cdot r_w(x_i)}.
$$
\end{restatable}
See proof in Appx.~\ref{subsec:prof_frozen-w}. The closed form reveals how the equilibrium balances three forces: $p_{\mathrm{ref}}$
provides a baseline, synthetic-data weight $\eta$ amplifies the preference-induced exponential tilt $e^{\tau r_w(\cdot)}$ and concentrates mass on high-reward instances, while the constant $Z_w^\eta$ ensures the result is a valid probability distribution. As $\eta\to 0$, the equilibrium recovers $p_{\mathrm{ref}}$; as $\eta\to 1$, the tilt dominates and mass concentrates on the highest-reward instance, aligned with results of the purely synthetic setting.

With the frozen-$w$ equilibrium characterized, we now return to the full coupled system in which both the model distribution and user preferences evolve simultaneously. 

\begin{tcolorbox}[graythm]
\begin{restatable}[Convergence of coupled dynamics]{theorem}{GlobalReinj}
\label{thm:global-reinj}
Consider the coupled dynamics \eqref{eq:data-update}-\eqref{eq:pref-update}. If
$
\eta(L_p+2L_w)<1
$, then there exists
$ \xi \in \left(
\frac{\eta L_w}{\beta(1-2\eta L_w)}
,
\frac{1-\eta L_p}{2\beta\eta L_p}\right)
$ such that the coupled update is a strict contraction on
$\left(\mathcal P(\mathcal X)\times\mathcal W,d_\xi\right)$, where
$
d_\xi\bigl((p,w),(q,v)\bigr)
:=
d_{\mathrm{TV}}(p,q)+\xi\|w-v\|.
$
Consequently, there exists a unique equilibrium $(p_K^\star,w_K^\star)$ satisfying \eqref{eq:equilibrium}, and every trajectory converges to it geometrically.
\end{restatable}
\end{tcolorbox}
See proof in Appx.~\ref{subsec:prof_global-reinj}. The metric $d_\xi$ combines distributional and preference distances into a single metric on $\mathcal{P}(\mathcal{X})\times \mathcal{W}$; the window condition on $\xi$ ensures both coordinates contract simultaneously, and the contraction condition $\eta(L_p + 2L_w) < 1$ ensures the reference mixing damps both feedback loops sufficiently for the Banach fixed-point theorem to apply. Compared with the frozen-preference condition $\eta L_p<1$ in Lem.~\ref{thm:frozen-w}, the extra $L_w$ term captures the stability cost of preference drift.

Thm.~\ref{thm:global-reinj} highlights a key difference from the purely synthetic setting: whereas pure self-consumption leads to path-dependent collapse driven by initial conditions, mixing with $p_{\mathrm{ref}}$ yields a unique equilibrium to which all trajectories converge, regardless of initialization. The resulting equilibrium depends entirely on $p_{\mathrm{ref}}$ and $\eta$, which raises a natural design question:  \emph{how should $p_{\mathrm{ref}}$ and $\eta$ be chosen so that the equilibrium preserves desired attributes?} We address this in Sec.~\ref{sec:mitigation}.

\section{Steering Equilibria via Reference Data Mixing}
\label{sec:mitigation}

The uniqueness of the equilibrium under reference data mixing, established in Thm.~\ref{thm:global-reinj}, motivates the use of $p_{\mathrm{ref}}$
and $\eta$ as \textbf{design variables}: by choosing them appropriately, one can steer the coupled dynamics toward a desired equilibrium. Specifically, denote $\left(p_K^\star(p_{\mathrm{ref}},\eta),\, w_K^\star(p_{\mathrm{ref}},\eta)\right)$ as the unique equilibrium attained under reference distribution $p_{\mathrm{ref}}$ and mixing weight $\eta$, the corresponding mean feature vector is $\bar{\varphi}(p_K^\star(p_{\mathrm{ref}},\eta)) := \mathbb{E}_{x \sim p_K^\star(p_{\mathrm{ref}},\eta)}[\varphi(x)]$. 

To specify which attributes should be preserved at equilibrium, suppose the platform identifies $L$  features of interest, each represented by a \textit{unit-norm} direction $v_\ell \in \mathbb{R}^d$ in the feature space. For example, $v_\ell$  might point toward the feature vectors of instances belonging to a particular genre, topic, or demographic group. The quantity $\bigl\langle v_\ell,\, \bar\varphi(p_K^\star) \bigr\rangle$ then measures how well the equilibrium distribution represents feature $\ell$ on average. Given thresholds $\theta_1, \dots, \theta_L > 0$ specifying the minimum acceptable representation for each feature, we say that $(p_{\mathrm{ref}}, \eta)$ satisfies the \textbf{preservation constraint} if $\bigl\langle v_\ell,\, \bar\varphi(p_K^\star(p_{\mathrm{ref}},\eta)) \bigr\rangle\geq \theta_\ell, \forall \ell \in \{1,\cdots, L\}$. For example, in an image generation platform with three categories, setting $L=1$, $v_1 =(\underbrace{1}_{\text{bird}},\,\underbrace{0}_{\text{cat}},\,\underbrace{0}_{\text{dog}})$ and $\theta_1 = 0.05$ requires that bird content constitutes at least $5\%$ of the equilibrium output.

Let $N$ denote the number of training samples used at each training round. The reference distribution   $p_{\mathrm{ref}} = (p_{\mathrm{ref},1}, \cdots,p_{\mathrm{ref},n})$ specifies the allocation of reference samples across instances: at each round,  $p_{\mathrm{ref},i}\cdot (1-\eta)\cdot N$ fresh samples of instance $x_i$ are collected and added to training. Collecting these samples is not free. Rare or specialized instances (e.g., those representing underrepresented languages or domain-specific content) require more effort to source, label, and verify than common ones. We capture this heterogeneity by assigning each instance $x_i$ a per-sample cost $c_i > 0$, so that the \textbf{total collection cost} per round is $J(p_{\mathrm{ref}},\eta):=(1-\eta)\cdot N\cdot \sum_{i=1}^n c_ip_{\mathrm{ref},i} $

\textbf{Design objective.} Our goal is to find $(p_{\mathrm{ref}},\eta)$ that minimizes the total collection cost subject to the preservation constraint at the long-run equilibrium. This is naturally formulated as a bilevel problem: the outer level optimizes over $(p_{\mathrm{ref}}, \eta)$, and the inner level computes the equilibrium that determines whether the preservation constraint is met:
\begin{equation}
\label{eq:bilevel}
\min_{p_{\mathrm{ref}} \in \mathcal{P}(\mathcal{X}),\; \eta \in ({0,1})} \quad
 J(p_{\mathrm{ref}},\eta) \quad
\text{s.t.} \quad \bigl\langle v_\ell,\, \bar\varphi(p_K^\star(p_{\mathrm{ref}}, \eta)) \bigr\rangle \ge \theta_\ell,
\quad \ell = 1, \dots, L
\end{equation}

However, \eqref{eq:bilevel} is difficult to solve directly: evaluating the preservation constraint requires computing $p_K^\star(p_{\mathrm{ref}}, \eta)$, which is the fixed point of a nonlinear coupled map and has no closed form. To tackle this, we propose a worst-case relaxation of the preservation constraint that depends directly on $p_{\mathrm{ref}}$ and $\eta$, bypassing the need to compute the equilibrium explicitly. Enforcing this relaxed condition decouples the bilevel structure and reduces~\eqref{eq:bilevel} to a collection of linear programs indexed by $\eta$.

\begin{tcolorbox}[graythm]
\begin{restatable}[Sufficient condition for preservation]{proposition}{ProtSuff}
\label{prop:protection-suff}
Fix $\eta\in (0,\frac{1}{(L_p+2 L_w)})$. If $\forall \ell \in \{1,\dots,L\}$
\begin{equation*}
\sum_{i=1}^n p_{\mathrm{ref},i}\,\langle v_\ell,\varphi(x_i)\rangle
\;\ge\;
\frac{\theta_\ell-\eta\cdot\min_{i=1,\cdots,n}\langle v_\ell,\varphi(x_i)\rangle}{1-\eta},
\end{equation*}

then the equilibrium satisfies the preservation constraint 
$\bigl\langle v_\ell,\bar{\varphi}(p_K^\star(p_{\mathrm{ref}},\eta))\bigr\rangle
\ge
\theta_\ell
$, $\forall \ell$.
\end{restatable}
\end{tcolorbox}

\begin{wrapfigure}[13]{r}{0.56\textwidth}
\vspace{-24pt}
\begin{minipage}{0.56\textwidth}
\begin{algorithm}[H]
\caption{Cost-Optimal Reference Data Mixing}
\label{alg:ref-design}
\begin{algorithmic}[1]
\REQUIRE Data size $N$, feature vectors $\{\varphi(x_i)\}_{i=1}^n$, directions $\{v_\ell\}_{\ell=1}^L$, thresholds $\{\theta_\ell\}_{\ell=1}^L$, costs $\{c_i\}_{i=1}^n$ 
\ENSURE Reference distribution $p_{\mathrm{ref}}^\star$, mixing weight $\eta^\star$
\STATE $J_{\min}\gets+\infty$
\FOR{$\eta$ on a grid in $(0,\eta_{\max})$}
    \STATE Solve LP~\eqref{eq:inner-LP} $\to$ $p_{\mathrm{ref}}(\eta)$, $C^\star(\eta)$
    \IF{LP feasible and $J(\eta)<J_{\min}$}
        \STATE $J_{\min}\gets J(\eta)$, \; $\eta^\star\gets\eta$, \; $p_{\mathrm{ref}}^\star\gets p_{\mathrm{ref}}(\eta)$
    \ENDIF
\ENDFOR
\RETURN $p_{\mathrm{ref}}^\star,\,\eta^\star$
\end{algorithmic}
\end{algorithm}
\end{minipage}
\vspace{-5pt}
\end{wrapfigure}

See proof in Sec.~\ref{subsec:prof_protection-suff}. 
Replacing the equilibrium constraint in~\eqref{eq:bilevel} by the sufficient condition
of Prop.~\ref{prop:protection-suff} eliminates the bilevel structure entirely. Since
$\eta$ appears only in the right-hand side and in the cost factor $(1-\eta)$, fixing
$\eta$ decouples the problem into a linear program over $p_{\mathrm{ref}}$.

\textbf{Algorithm. }The complete design procedure combines an outer search over $\eta$ with an inner linear program (LP)
solve for $p_{\mathrm{ref}}$, as summarized in Alg.~\ref{alg:ref-design}. The search range is $(0,\eta_{\max})$,
where $\eta_{\max}$ upper-bounds mixing weight. Thm.~\ref{thm:global-reinj} provides a conservative range $(0,\frac{1}{(L_p + 2L_w)})$ for $\eta$, but in practice, $\eta_{\max}$ can be set larger since the Lipschitz constants are worst-case. The procedure is performed entirely offline prior to deployment and remains computationally lightweight: the inner problem is a linear program with $n$ variables and $L+1$ constraints, and the outer problem is one-dimensional.

\underline{\textit{Inner problem: optimal reference distribution for fixed $\eta$.}} For any fixed $\eta\in(0,\bar\eta)$, Prop.~\ref{prop:protection-suff}  reduces the design task to the following linear program:
\begin{equation}
\label{eq:inner-LP}
C^\star(\eta)
\;:=\;
\min_{p_{\mathrm{ref}}\in\mathcal{P}(\mathcal{X})}
\sum_{i=1}^n p_{\mathrm{ref},i}\,c_i
~~~~
\text{s.t.}
~~~~
\sum_{i=1}^n p_{\mathrm{ref},i}\,\langle v_\ell,\varphi(x_i)\rangle
\;\ge\;
\frac{\theta_\ell-\eta\min_{i}\langle v_\ell,\varphi(x_i)\rangle}{1-\eta},~
\forall \ell
\end{equation}

The solution $p_{\mathrm{ref}}(\eta)$ is the cheapest reference distribution that guarantees the preservation constraint at equilibrium for the given $\eta$. By standard linear programming theory, an optimal solution always exists with at most $L+1$ strictly positive entries,  meaning the optimal $p_{\mathrm{ref}}$ concentrates on at most  $L+1$ instances. The platform therefore needs to actively source only a small number of distinct instance types, regardless of how large $\mathcal{X}$ is. Each inner problem can be solved efficiently offline using standard LP solvers such as simplex or interior-point methods.

\underline{\textit{Outer problem: selecting the mixing weight.}}
Given  $C^\star(\eta)$, the per-round collection cost is $J(\eta) := (1 - \eta) N  C^\star(\eta)$. The outer problem
\begin{equation}
\label{eq:outer-search}
\eta^\star
\in
\arg\min_{\eta\in(0,\bar\eta)}
(1-\eta)\,N\,C^\star(\eta).
\end{equation}
balances two competing effects. A larger $\eta$ means fewer reference samples per round, reducing the direct cost $(1-\eta)\,N\,C^\star(\eta)$;  but it simultaneously tightens the preservation constraint, since the right-hand side of~\eqref{eq:inner-LP} grows with $\eta$, forcing $C^\star(\eta)$ to increase. The optimal $\eta^\star$ resolves this tension, and since the problem is one-dimensional, it can be found efficiently by a simple grid search.

\section{Experiments}
\label{sec:experiments}
This section aims to empirically illustrate our theoretical results in a realistic training condition, where finite samples, imperfect model training, and a concrete model architecture introduce errors absent from the abstract analysis\footnote{Please see our code at \url{https://github.com/osu-srml/Model-Reward-Coevolution}}.

\textbf{Datasets.} We conduct experiments on two datasets (i) Image dataset: \textbf{CIFAR-10} \cite{cifar10}: It contains 60,000 images from 10 classes $\{\text{airplane}:=0, \text{automobile}:=1, \text{bird}:=2, \text{cat}:=3, \text{deer}:=4, \text{dog}:=5, \text{frog}:=6, \text{horse}:=7, \text{ship}:=8, \text{truck}:=9 \}$. (ii) Text dataset: \textbf{AG~News}~\citep{AGNews}: It contains 127,600 texts from 4 classes $\{\text{world}:=0, \text{sports}:=1, \text{business}:=2, \text{science/technology}:=3 \}$

\textbf{Models.} We experiment with three generative architectures: OT-CFM~\citep{FlowMatching,Improvingflow-based} and DDPM~\citep{ddpm,ImprovedDDpm} for image generation,
and GPT-2~\citep{gpt2} for text generation. We present OT-CFM results in this section; DDPM and GPT-2 results appear in Appx.~\ref{subsec:ddpm} and~\ref{subsec:cpt2}, respectively. All models are adopted from public GitHub repositories and are used with their default hyperparameters. Each model is initialized from its released pretrained checkpoint, keeping architectures and optimization settings unchanged.

\textbf{Feature map $\varphi$ and preference $w_t$.} For CIFAR-10, we use a pretrained ResNet-56~\citep{ResNet} classifier (accuracy $\geq 93\%$). The feature map $\varphi(x) \in \mathbb{R}^{10}$ is the normalized softmax output, where $\varphi(x)_c$ represents the predicted probability of class $c$. The preference vector $w_t \in \mathbb{R}^{10}$ is updated via Eq.~\eqref{eq:pref-update} with projection onto the unit ball.
For AG~News, we use a pretrained BERT~\citep{bert} classifier with 4-dimensional normalized softmax features; see Appx.~\ref{subsec:cpt2} for details.

\textbf{Iterative Retraining.} For OT-CFM on CIFAR-10, each round $t$ proceeds as follows. The current model generates 10,000 candidate images, from which 5,000 are selected via $K$-way Luce curation with reward $r_t(x) = \langle w_t, \varphi(x) \rangle$. The curated samples are mixed with reference data at ratio $\eta$ to form a training set of size 5,000, on which the model is fine-tuned for 1,000 steps. After training, 10,000 fresh samples are generated for evaluation: we compute per-class proportions and update $w_t$. We run each experiment for 20 rounds. We set $\tau = 5.0$, $K = \infty$, and $\beta = 0.01$ for the main experiments. The effect of each parameter is discussed in Appx.~\ref{subsubsec:cfm_para}. Due to the computational cost of training and sampling from generative models, we report a single run for each setting, with the random seed varying across rounds. To verify reproducibility, Appx.~\ref{subsec:ddpm} reports three independent runs under the same setting, confirming that the trajectories are consistent.

\textbf{Different initializations, different outcomes ($\eta=1$).} Under pure synthetic training, the system exhibits initialization-dependent lock-in. When $w_0 =[1,0,0,0,\dots]$ (prefer airplane), curation progressively favors airplane images. The model's airplane proportion rises from $\sim$10\% to $>$95\% over 20 rounds as shown in Fig.~\ref{Fig:main}(a), while all other classes collapse below 5\%. When $w_0 = [0,0,0,1,\dots]$ (prefer cat), the same dynamics drive the model toward cat instead
(Fig.~\ref{Fig:main}(c)).
In both cases, the curation signal is strong enough ($\tau = 5.0$, $K=\infty$) to steer the entire generated distribution. Appx.~\ref{subsubsec:cfm_para} varies $\beta$ and shows that preference adaptation mainly changes the speed of preference drift, while the qualitative lock-in behaviors remain consistent.

\begin{figure}[thbp!]
\begin{center}
\centerline{\includegraphics[width=\textwidth]{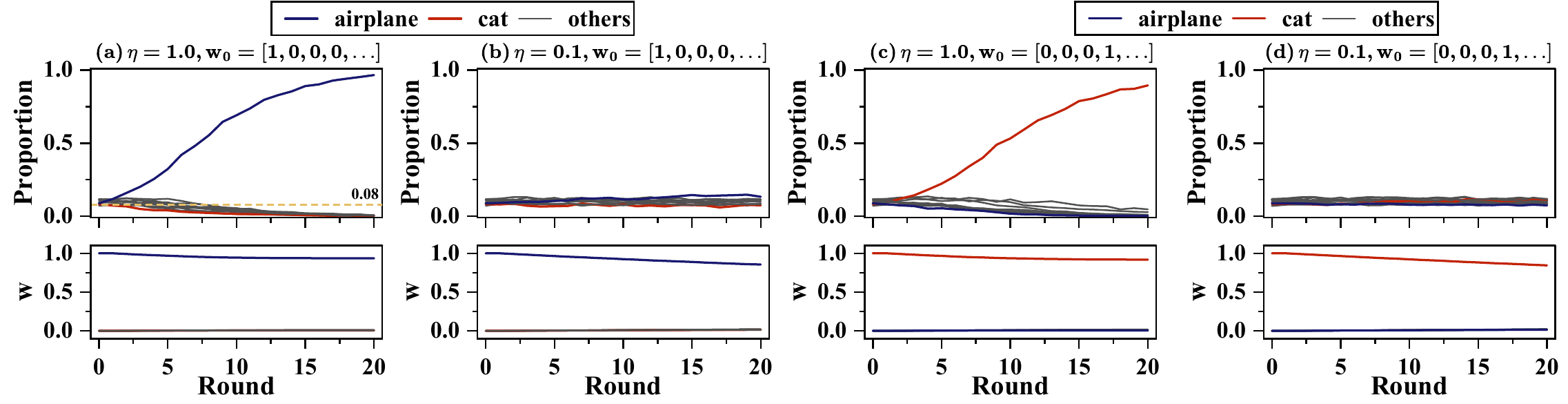}}
\caption{Co-evolution dynamics under different initializations and mixing ratios. \textbf{Top}: per-class generation proportions. \textbf{Bottom}: preference vector $w_t$.(a)~The model locks into airplane, reaching $>$95\% by round 20.
(b)~Uniform reference data stabilizes the distribution near equilibrium.
(c)~A different initialization leads to a different outcome.
(d)~The system converges to a similar equilibrium, independent of initialization.}
\vskip -0.2in
\label{Fig:main}
\end{center}
\end{figure}

\textbf{Enough reference data yields a unique equilibrium ($\eta < 1$).} When 90\% of each training batch comes from the uniform CIFAR-10 training set ($\eta = 0.1$), regardless of whether $w_0 =[1,0,0,0,\dots]$ (Fig.~\ref{Fig:main}(b)) or $w_0 = [0,0,0,1,\dots]$ (Fig.~\ref{Fig:main}(d)), the class proportions remain close to uniform across all 20 rounds. The two runs converge to a similar equilibrium, consistent with the uniqueness guarantee of Thm.~\ref{thm:global-reinj}: for sufficiently small $\eta$, the reference distribution stabilizes the system to a single fixed point independent of initialization.

\begin{figure}[t]
\centering
\begin{minipage}[c]{0.67\textwidth}
\centering
  \captionof{table}{
    Comparison of reference data strategies.
    $C^*$: per-sample cost; $J$: total cost $(1-\eta) \cdot N \cdot C^*$.
    Last four columns: proportion of each protected class at round 20. $\theta_{\ell} = 0.08$ for all protected classes.
  }
  \label{tab:alg1-compare}
  \small
  \begin{tabular}{llcccccc}
    \toprule
    $p_{\mathrm{ref}}$ & $\eta$ & $C^*$ & $J$
      & frog & horse & ship & truck \\
    \midrule
    Uniform & 0.50 & 16.4 & 41,000
      & .067 & .077 & .059 & .067 \\
    Uniform & 0.20 & 16.4 & 65,600
      & .112 & .113 & .097 & .105\\
    \textbf{Ours} & \textbf{0.67} & \textbf{26.81} & \textbf{44,043}
      & \textbf{.084} & \textbf{.113} & \textbf{.088} & \textbf{.092} \\
    \bottomrule
  \end{tabular}
\end{minipage}
\hfill
\begin{minipage}[c]{0.28\textwidth}
    \centering
    \includegraphics[width=\linewidth]{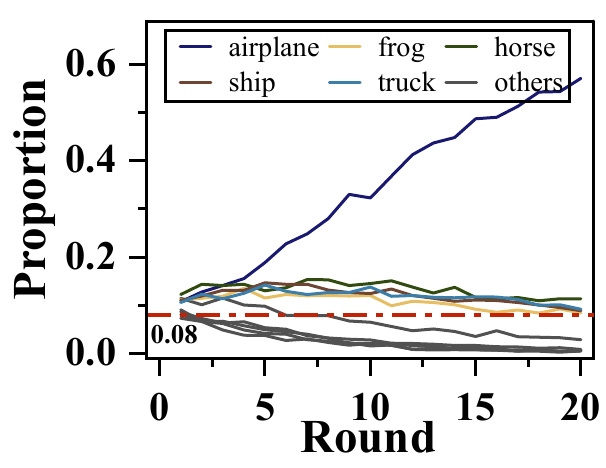}
    \caption{Class proportions under designed reference.}
    \label{Fig:cfm_pre1}
\end{minipage}
\vskip -0.2in
\end{figure}

\textbf{Reference data design.} We evaluate the effectiveness of Alg.~\ref{alg:ref-design} for preserving minority classes under co-evolution. We set the goal is to guarantee that the last four classes (frog, horse, ship, truck) each maintain a proportion of at least $\theta = 0.08$ at equilibrium. We assign heterogeneous costs reflecting class rarity: cheap classes (airplane, automobile: $c=1$) are abundant, while protected classes are expensive (frog: 30, horse: 28, ship: 27, truck: 25). The preference is initialized to $w_0  =[1,0,0,0,\dots]$ (airplane).

Tbl.~\ref{tab:alg1-compare} compares our designed reference with uniform reference and Fig.~\ref{Fig:cfm_pre1} shows the class proportions under  our designed reference. Our designed reference achieves two advantages simultaneously. First, it uses a higher synthetic ratio ($\eta = 0.67$) than either uniform baseline, reducing the total amount of reference data required. Second, at round 20, all protected classes exceed the threshold $\theta_{\ell} = 0.08$, suggesting that the designed reference distribution steers the trajectory toward a feasible equilibrium. The partial LP search results over a uniform grid of 50 points in $\eta \in [0.2, 0.9]$ and results under a different setting are provided in Appx.~\ref{subsubsec:cfm_pref}.

\section{Discussion} Our analysis focuses on a basic co-evolutionary setting in which a single aggregate preference state interacts with a self-consuming generative model through preference-guided curation. The same framework naturally extends to richer forms of preference dynamics. In particular, we consider nonlinear gap-response updates that include saturating and state-dependent gated responses; the model-update branch remains unchanged, while the stability analysis becomes response-dependent. Beyond this, stochastic individual-level adaptation provides a natural micro-level interpretation of the deterministic update studied here, while heterogeneous user populations lead to coupled group-specific preferences interacting through a shared generative distribution. We discuss these extensions and the resulting theoretical questions in Appx.~\ref{sec:discussion}.

\section{Conclusion}
We studied self-consuming generative retraining in which the model distribution and user preferences co-evolve. The system was formulated as a coupled dynamic over $(p_t, w_t)$, with the model updated by reference-mixed user-curated data and the preference vector drifting toward the model's mean feature. Under purely synthetic curation $(\eta = 1)$, the dynamics admit singleton equilibria, and any instance with a sufficient initial advantage locks in. With reference mixing satisfying $\eta(L_p + 2L_w) < 1$, the coupled update becomes a strict contraction with a unique, globally attracting equilibrium. Building on this, we  further formulate equilibrium steering as a bilevel problem and reduce it to a family of linear programs indexed by $\eta$. Experiments empirically validate our analysis and confirm the effectiveness of our design.
Several directions remain open: relaxing the exact observability of $\varphi$ and $w_t$ to partially observed quantities, extending the contraction analysis to nonlinear rewards, and generalizing the dynamics to multi-platform settings with heterogeneous user populations.

\begin{ack}
This material is based upon work supported by the U.S. National Science Foundation under Grant Nos. IIS-2202699, IIS-2416895, IIS-2301599, CMMI-2301601, DMS-2529302, and OIA-2616279. The authors declare no competing interests.
\end{ack}

\small
\bibliographystyle{plainnat}
\bibliography{Reference}


\appendix
\newpage
\appendix
\startcontents[appendix]

\section*{Appendix Contents}
\hypersetup{linkcolor=black}
\printcontents[appendix]{}{1}{
    \setcounter{tocdepth}{2}
}
\hypersetup{linkcolor=blue}
\newpage
\section{Related Work}
\label{sec:relatedwork}

\paragraph{Self-consuming} When generative models are repeatedly trained on their own outputs, the training distribution can drift from the original data distribution. Early empirical studies demonstrated that this loop leads to progressive loss of distributional coverage and tail information \citep{GoMad, AICollapse, CollapseInevitable}, with \citet{amplification, xie2024automating} showing that feedback loops can amplify existing dataset biases or even introduce complex feature distribution shifts. Subsequent theoretical work characterized the conditions under which collapse occurs and the rate at which errors accumulate \citep{Satbility, TheoreticalPerspective, Tails, StatisticalAnalysisLM, RateCollapse}. Several mitigation strategies have been proposed, including retaining real data in the training loop \citep{GoMad, Satbility, CollapseInevitable}, filtering generated samples via verification or latent-space methods \citep{Verification, cai2026stabilizing}, and incorporating self-correction during generation \citep{self_correction}. Recent work has further broadened the study of generative feedback from isolated single-model loops to settings with multiple interacting models, networked generative ecosystems, and strategic interactions among generative platforms \citep{netowrkedestability, market, WhenandHow}. Our work builds on this line by studying a setting where the training loop is further shaped by preference-guided curation and where the preference signal itself co-evolves with the model.

\paragraph{Self-consuming with curated data}
A growing body of work has studied self-consuming retraining with human curation. \citet{Curation} show that, under a fixed reward signal, human curation acts as an implicit preference-optimization mechanism. \citet{AdversarialCuration} further show that when the same curation process is noisy or adversarially manipulated, the retraining loop can become systematically misaligned from genuine user preferences. 
Work on multi-model self-consuming systems shows that even benign human curation need not improve alignment once models interact through recycled data: its effect can be amplified, dampened, or even reversed by cross-model coupling \citep{WhenandHow}. 
More recently, \citet{HeterogeneousHumanCuration} studied self-consuming generative models with heterogeneous human curation, proving convergence and stability results across pure synthetic and reference-mixed retraining regimes. 
Related iterative alignment~\citet{AlignmentCollapse} studies the co-evolution of a policy and a learned reward-model estimate around a fixed underlying human preference. Their Stackelberg decomposition characterizes how gradient-based policy optimization can corrupt the reward-model estimate, and their FPO intervention regularizes the policy update to account for this influence.
However, these formulations still treat the underlying human preference as exogenous. In many deployed settings, model outputs can reshape what users read, compare, and upload, thereby altering the preferences that guide future curation. Our work studies this endogenous feedback loop by allowing model retraining and user preferences to co-evolve. Tbl.~\ref{tab:comparison} summarizes the key differences.

\paragraph{Performative prediction.}
Performative prediction (PP) studies learning problems where the deployed model changes the data distribution it will later face~\citep{perdomo2020performative}. This feedback effect appears in many human-facing ML systems: applicants may strategically modify their features, users may change their participation or engagement behavior, and adversaries may adapt to deployed classifiers. A central difficulty in PP is that the learner is no longer optimizing over a fixed distribution. Previous works mainly focused on the optimization in discriminative PP settings. Several works develop algorithms and convergence guarantees for finding performatively stable solutions. \citet{perdomo2020performative} introduced repeated risk minimization and repeated gradient descent under sensitivity and convexity assumptions. Subsequent work studied stochastic and online variants, including greedy and lazy deployment schemes~\citep{mendler2020stochastic,mendler2022anticipating}. State-dependent performative prediction further generalizes the framework by allowing the next distribution to depend not only on the deployed model but also on the current population state~\citep{brown2022performative,li2022state}. More recent work has moved beyond the strongly convex setting. For example, \citet{mofakhami2023performative} considered neural-network-based PP under structural assumptions, and \citet{li2024stochastic} established convergence to stationary performatively stable points for smooth non-convex losses, though with variance-dependent error terms. Another line of work directly optimizes the performative risk or seeks performatively optimal solutions under additional assumptions on the distribution map~\citep{izzo2021learn,miller2021outside,ray2022decision,liu2023two,zheng2024profl}. Others study PP under weaker convexity assumptions, distributional robustness, privacy, or welfare-oriented objectives~\citep{zhao2022optimizing,zhu2023online,kim2022making,li2024clipped,jia2024distributionally,xue2024distributionally, xie2025sprint, jin2026addressing}. These works show that the choice of solution concept is not only technical: it determines whether the learner aims for equilibrium-like stability, direct performative optimality, or robustness to uncertainty in the distributional response.  

A related perspective on endogenous user responses comes from preference dynamics under personalized recommendations \cite{PreferenceDynamics,zhang2019group} and strategic behavior \cite{zhang2020fair,zhang2022fairness,xie2024non,jin2022incentive,xie2024learning}. \citet{PreferenceDynamics} study how users' preferences evolve through interaction with a recommendation platform. Their setting is closely related on the preference side, but the platform selects from a fixed catalog without generative retraining on user-curated outputs. Our framework additionally closes the self-consuming training loop: model outputs shape user preferences, the evolving preferences guide curation, and the retained outputs are used to retrain the generator. The resulting dynamics involve both the generative distribution and the underlying preference, bringing curation-induced generator collapse and reference-data mixing into the analysis of preference evolution.

\newpage
\section{Complete Instantiation Examples}
\label{sec:example}
This appendix provides concrete examples that instantiate every symbol in our framework across several domains. We give two discrete examples (Appx.~\ref{subsec:ex-discrete}) and one continuous examples (Appx.~\ref{subsec:ex-continuous}).

\subsection{Discrete Examples}
\label{subsec:ex-discrete}
\paragraph{Example A: Image Generation with Object Classes}
Consider an image generation platform with $n = 10$ object categories.
\begin{itemize}[leftmargin=*,itemsep=0pt,topsep=0pt]
    \item \textbf{Sample space:} $\mathcal{X} = \{x_1, \dots, x_{10}\}$, representing the types: airplane, automobile, bird, cat, deer, dog, frog, horse, ship, truck.
    \item \textbf{Distribution space:} $\mathcal{P}(\mathcal{X})$ is the probability simplex over 10 types.
    \item \textbf{Feature map:} $\varphi: \mathcal{X} \to \mathbb{R}^{10}$ is the softmax output of a pretrained classifier, normalized to $\|\varphi(x_i)\| = 1$.
\end{itemize}

\textbf{Concrete feature vectors.}
After normalization, the feature vectors for three representative types are:
\begin{align*}
    \varphi(x_{\text{airplane}}) &= \frac{(0.95,\; 0.02,\; 0.01,\; 0.00,\; 0.00,\; 0.00,\; 0.00,\; 0.01,\; 0.01,\; 0.00)}{\|\cdot\|} \\
    &\approx (0.999,\; 0.021,\; 0.011,\; 0.000,\; 0.000,\; 0.000,\; 0.000,\; 0.011,\; 0.011,\; 0.000), \\[4pt]
    \varphi(x_{\text{cat}}) &= \frac{(0.01,\; 0.01,\; 0.03,\; 0.90,\; 0.01,\; 0.02,\; 0.00,\; 0.01,\; 0.00,\; 0.01)}{\|\cdot\|} \\
    &\approx (0.011,\; 0.011,\; 0.033,\; 0.998,\; 0.011,\; 0.022,\; 0.000,\; 0.011,\; 0.000,\; 0.011), \\[4pt]
    \varphi(x_{\text{frog}}) &= \frac{(0.00,\; 0.00,\; 0.02,\; 0.01,\; 0.02,\; 0.01,\; 0.92,\; 0.01,\; 0.00,\; 0.01)}{\|\cdot\|} \\
    &\approx (0.000,\; 0.000,\; 0.022,\; 0.011,\; 0.022,\; 0.011,\; 0.999,\; 0.011,\; 0.000,\; 0.011).
\end{align*}
Each vector is dominated by one coordinate, reflecting high classifier confidence.
 
\textbf{Initial distribution.}
$p_0 = (0.10,\; 0.10,\; \dots,\; 0.10)$ (uniform across all types).
 
\textbf{User preference.}
Suppose the user prefers airplane and ship content. The initial preference vector is
$$
w_0 = \frac{(0.6,\; 0,\; 0,\; 0,\; 0,\; 0,\; 0,\; 0,\; 0.4,\; 0)}{\|(0.6,\; 0,\; 0,\; 0,\; 0,\; 0,\; 0,\; 0,\; 0.4,\; 0)\|}
\approx (0.832,\; 0,\; 0,\; 0,\; 0,\; 0,\; 0,\; 0,\; 0.555,\; 0),
$$
with $\|w_0\| = 1$.
 
\textbf{Reward computation.}
The reward $r_0(x) = \langle w_0, \varphi(x)\rangle$ at round $t = 0$:
\begin{align*}
    r_0(x_{\text{airplane}}) &= 0.832 \times 0.999 + 0.555 \times 0.011 + \cdots \approx 0.837, \\
    r_0(x_{\text{cat}}) &= 0.832 \times 0.011 + 0.555 \times 0.000 + \cdots \approx 0.009, \\
    r_0(x_{\text{frog}}) &= 0.832 \times 0.000 + 0.555 \times 0.000 + \cdots \approx 0.000.
\end{align*}
Airplane receives reward $\approx 0.84$, while cat and frog receive near-zero reward. With curation strength $\tau > 0$, the $K$-way Luce rule strongly favors airplane. Over successive rounds, $p_t$ concentrates on airplane, $w_t$ drifts further toward the airplane direction, and frog and horse vanish (Thm.~\ref{thm:pure-lockin}).
 
\textbf{Protection.}
To prevent frog and horse from disappearing, the designer sets $L = 2$:
\begin{align*}
    v_1 &= (\underbrace{0}_{0},\;\underbrace{0}_{1},\;\underbrace{0}_{2},\;\underbrace{0}_{43},\;\underbrace{0}_{4},\;\underbrace{0}_{5},\;\underbrace{1}_{6},\;\underbrace{0}_{7},\;\underbrace{0}_{8},\;\underbrace{0}_{9}), \quad \theta_1 = 0.05 \quad \text{(frog $\geq 5\%$)}, \\
    v_2 &= (\underbrace{0}_{0},\;\underbrace{0}_{1},\;\underbrace{0}_{2},\;\underbrace{0}_{6},\;\underbrace{0}_{4},\;\underbrace{0}_{5},\;\underbrace{0}_{6},\;\underbrace{1}_{7},\;\underbrace{0}_{8},\;\underbrace{0}_{9}), \quad \theta_2 = 0.05 \quad \text{(horse $\geq 5\%$)}.
\end{align*}
The constraint $\langle v_1, \bar\varphi(p^\star)\rangle \geq 0.05$ ensures that frog-type images constitute at least $5\%$ of the equilibrium output in the feature-weighted sense.
 
\textbf{Cost.}
If real-world images are equally easy to collect across all types, we set $c_i = 1$ uniformly. In practice, some types are naturally rarer. For instance, high-quality frog images may be scarce compared to airplane or cat images, making them more expensive to collect. Setting $c_{\text{frog}} = 5$ reflects this scarcity.

\paragraph{Example B: Text Generation with Writing Style}

Consider a text generation platform where outputs are categorized by writing style.
\begin{itemize}[leftmargin=*,itemsep=0pt,topsep=0pt]
    \item \textbf{Sample space:} $\mathcal{X} = \{x_1, \dots, x_5\}$, representing five style types: formal, casual, poetic, technical, humorous.
    \item \textbf{Distribution space:} $\mathcal{P}(\mathcal{X}) = \Delta_5$.
    \item \textbf{Feature map:} $\varphi: \mathcal{X} \to \mathbb{R}^5$ is the softmax output of a fine-tuned style classifier, normalized to $\|\varphi(x_i)\| = 1$.
\end{itemize}
 
\textbf{Concrete feature vectors.}
\begin{align*}
    \varphi(x_{\text{casual}}) &\approx (0.033,\; 0.997,\; 0.033,\; 0.022,\; 0.055) / \|\cdot\|, \\
    \varphi(x_{\text{poetic}}) &\approx (0.055,\; 0.033,\; 0.993,\; 0.022,\; 0.055) / \|\cdot\|, \\
    \varphi(x_{\text{technical}}) &\approx (0.044,\; 0.022,\; 0.011,\; 0.995,\; 0.011) / \|\cdot\|.
\end{align*}
 
\textbf{User preference and reward.}
Most users prefer casual content:
$$
w_0 = \frac{(0.1,\; 0.8,\; 0.0,\; 0.0,\; 0.1)}{\|(0.1,\; 0.8,\; 0.0,\; 0.0,\; 0.1)\|}
\approx (0.124,\; 0.990,\; 0.000,\; 0.000,\; 0.124).
$$
Then:
\begin{align*}
    r_0(x_{\text{casual}}) &= \langle w_0, \varphi(x_{\text{casual}})\rangle \approx 0.99, \\
    r_0(x_{\text{poetic}}) &= \langle w_0, \varphi(x_{\text{poetic}})\rangle \approx 0.04, \\
    r_0(x_{\text{technical}}) &= \langle w_0, \varphi(x_{\text{technical}})\rangle \approx 0.02.
\end{align*}
Casual text dominates curation. Over rounds, poetic and technical writing vanish.
 
\textbf{Protection.}
Preserve poetic and technical writing:
\begin{align*}
     v_1 &= (0,\; 0,\; 1,\; 0,\; 0), \quad \theta_1 = 0.08 \quad \text{(poetic $\geq 8\%$)}, \\
    v_2 &= (0,\; 0,\; 0,\; 1,\; 0), \quad \theta_2 = 0.05 \quad \text{(technical $\geq 5\%$)}.
\end{align*}
 
\textbf{Cost.}
$c_{\text{poetic}} = 3$ (high-quality poetic text requires expert curation), $c_{\text{casual}} = 1$ (abundant online).

\subsection{Continuous Examples}
\label{subsec:ex-continuous}
\paragraph{Example C: Image Generation with CLIP Embedding}
Consider an image generation platform where outputs live in a continuous space.
\begin{itemize}[leftmargin=*,itemsep=0pt,topsep=0pt]
    \item \textbf{Sample space:} $\mathcal{X} \subset \mathbb{R}^m$ is a compact subset of image space.
    \item \textbf{Distribution space:} $\mathcal{P}(\mathcal{X})$ is the set of all Borel probability measures on $\mathcal{X}$.
    \item \textbf{Feature map:} $\varphi: \mathcal{X} \to \mathbb{R}^{512}$ is a pretrained CLIP image encoder, with $\|\varphi(x)\| = 1$. Feature dimension $d = 512$. Semantically similar images have nearby embeddings.
\end{itemize}
 
\textbf{Concrete feature vectors.}
Three representative images and their (truncated) CLIP embeddings:
\begin{align*}
    \varphi(x_{\text{sunset}}) &= (0.12,\; -0.08,\; 0.43,\; \dots)_{512}, \quad \text{(vibrant, warm tones)} \\
    \varphi(x_{\text{portrait}}) &= (-0.05,\; 0.31,\; 0.15,\; \dots)_{512}, \quad \text{(human subject, neutral)} \\
    \varphi(x_{\text{mono}}) &= (-0.22,\; -0.14,\; -0.35,\; \dots)_{512}, \quad \text{(dark, monochrome)}
\end{align*}
each normalized to $\|\varphi\| = 1$.
 
\textbf{User preference and reward.}
Users consistently select vibrant, colorful images. After 5 rounds, the preference vector has drifted to
$$
w_5 = (0.15,\; -0.06,\; 0.38,\; \dots)_{512}, \quad \|w_5\| = 0.73.
$$
The reward for each image:
\begin{align*}
    r_5(x_{\text{sunset}}) &= \langle w_5, \varphi(x_{\text{sunset}})\rangle \approx 0.71, \\
    r_5(x_{\text{portrait}}) &= \langle w_5, \varphi(x_{\text{portrait}})\rangle \approx 0.18, \\
    r_5(x_{\text{mono}}) &= \langle w_5, \varphi(x_{\text{mono}})\rangle \approx -0.41.
\end{align*}
Sunset images receive high reward; monochrome images receive negative reward. Curation suppresses monochrome content, $p_t$ concentrates on vibrant images, $w_t$ drifts further from the monochrome direction.
 
\textbf{Mean feature vector.}
The equilibrium mean feature vector is
$$
\bar\varphi(p^\star) = \int_{\mathcal{X}} \varphi(x)\, p^\star(dx) \in \mathbb{R}^{512}.
$$
As $p_t$ concentrates on vibrant images, $\bar\varphi(p_t)$ moves toward the vibrant cluster in CLIP space and away from the monochrome direction.
 
\textbf{Protection.}
To preserve monochrome imagery, the designer collects 500 reference monochrome images and computes their mean CLIP embedding:
\[
v_1 = \frac{\frac{1}{500}\sum_{j=1}^{500} \varphi(x_j^{\text{mono}})}{\left\|\frac{1}{500}\sum_{j=1}^{500} \varphi(x_j^{\text{mono}})\right\|} \in \mathbb{R}^{512}, \qquad \theta_1 = 0.05.
\]
The constraint $\langle v_1, \bar\varphi(p^\star)\rangle \geq 0.05$ ensures the equilibrium retains at least $0.05$ projection onto the monochrome direction in CLIP space.
 
\textbf{Cost.}
$c(x)$ varies across CLIP space: reference images in the monochrome region are expensive to source (niche artistic photography), while vibrant natural images are cheap (abundant on the web).
\section{Proofs}
\label{sec:prof}
\subsection{Proof of Explicit Retraining Update}
\label{subsec:prof_explicit-update}
\begin{tcolorbox}[graythm]
\begin{lemma}[Explicit retraining update]
\label{lem:explicit-update}
    Let $p_{t+1}$ be defined by \eqref{eq:retrain}. Then, 
    \begin{equation}
        p_{t+1}(x)     =   (1-\eta)p_{\mathrm{ref}}(x)   +   \eta p_t(x) H^K_{p_t,w_t}(x),
    \end{equation}   
        where
        $$    H^K_{p_t,w_t}(x):=\mathbb{E}_{x_1,\dots,x_{K-1}\sim p_t}\left[\frac{K e^{\tau \cdot r_{w_t}(x)}}{e^{\tau \cdot r_{w_t}(x)}+\sum_{i=1}^{K-1}e^{\tau \cdot  r_{w_t}(x_i)}}\right]. $$  
        Moreover, for every fixed $x\in \mathcal{X}$,  $$  H^K_{p_t,w_t}(x)\xrightarrow{K\to\infty}  \frac{e^{\tau r_t(x)}}{\mathbb{E}_{z\sim p_t}[e^{\tau r_t(z)}]}.  $$
  
\end{lemma}
\end{tcolorbox}
\begin{proof}
First, by minimization of the cross-entropy, we know that for any distribution $q \in \mathcal{P}(\mathcal{X})$, $\arg\max_{p\in\mathcal{P}(\mathcal{X})}\mathbb{E}_{x\sim q}[\log p(x)] = q$

\paragraph{Case 1: $\eta=1$.}
When $\eta=1$, the objective in \eqref{eq:retrain} reduces to
$$
p_{t+1}
=
\arg\max_{p\in\mathcal{P}(\mathcal{X})}
\mathbb{E}_{\substack{
x_1,\dots,x_K\sim p_t\\
\hat x\sim \mathrm{Luce}_K(x_1,\dots,x_K)}}
[\log p(\hat x)]
$$
Let $p_t^{\mathrm{cur}}$ denote the distribution of the retained sample $\hat x$. By the population MLE argument above, $p_{t+1}=p_t^{\mathrm{cur}}$

To characterize $p_t^{\mathrm{cur}}$, fix round $t$ and sample$X_1,\dots,X_K \overset{\mathrm{i.i.d.}}{\sim} p_t$. Let $i_K\in[K]$ be selected according to
$$
\Pr(i_K=i\mid x_1,\dots,x_K)
=
\frac{e^{\tau r_t(x_i)}}{\sum_{j=1}^K e^{\tau r_t(x_j)}}.
$$
Define the retained sample by $\hat X:=X_{i_K}$. Then $p_t^{\mathrm{cur}}$ is the law of $\hat X$. By symmetry of the $K$ candidates,
$$
\begin{aligned}
p_{t+1}(x)=p_t^{\mathrm{cur}}(x)
&=
\sum_{i=1}^{K}\int_{y_j,\;j\neq i}
p_t(y_1,\dots,y_{i-1},x,y_{i+1},\dots,y_K)\,
\Pr(i_K=i\mid x,y_j,\;j\neq i)
\prod_{j\neq i}dy_j \\
&=
K\int_{y_1,\dots,y_{K-1}}
p_t(y_1,\dots,y_{K-1},x)\,
\Pr(i_K=K\mid y_1,\dots,y_{K-1},x)\,
dy_1\cdots dy_{K-1} \\
&=
p_t(x)\,K\int_{y_1,\dots,y_{K-1}}
\frac{e^{\tau r_t(x)}}{e^{\tau r_t(x)}+\sum_{i=1}^{K-1}e^{\tau r_t(y_i)}}
\prod_{i=1}^{K-1}p_t(y_i)\,dy_1\cdots dy_{K-1}.
\end{aligned}
$$
Therefore,
$$
p_{t+1}(x)=p_t(x)\,H^K_{p_t,w_t}(x),
$$
where
$$
\begin{aligned}
H^K_{p_t,w_t}(x)
&:=
K\int_{y_1,\dots,y_{K-1}}
\frac{e^{\tau r_t(x)}}{e^{\tau r_t(x)}+\sum_{i=1}^{K-1}e^{\tau r_t(y_i)}}
\prod_{i=1}^{K-1}p_t(y_i)\,dy_1\cdots dy_{K-1} \\
&=
\mathbb{E}_{X_1,\dots,X_{K-1}\overset{\mathrm{i.i.d.}}{\sim}p_t} \left[
\frac{K e^{\tau r_t(x)}}{e^{\tau r_t(x)}+\sum_{i=1}^{K-1}e^{\tau r_t(X_i)}}
\right].
\end{aligned}
$$

Now consider the large-$K$ limit. Since $r_t(x)\in[-1,1]$, we have
$$
e^{-\tau}\le e^{\tau r_t(x)}\le e^\tau
\qquad\text{for all }x\in\mathcal{X}.
$$
Rewrite the denominator as
$$
\frac{e^{\tau r_t(x)}}{K}
+
\frac{K-1}{K}\cdot \frac1{K-1}\sum_{i=1}^{K-1}e^{\tau r_t(X_i)}.
$$
By the strong law of large numbers,
$$
\frac1{K-1}\sum_{i=1}^{K-1}e^{\tau r_t(X_i)}
\xrightarrow[]{\mathrm{a.s.}}
\mathbb{E}_{z\sim p_t}[e^{\tau r_t(z)}].
$$
Hence, for every fixed $x\in\mathcal{X}$,
$$
\frac{e^{\tau r_t(x)}}
{\frac{e^{\tau r_t(x)}}{K}
+\frac{K-1}{K}\cdot \frac1{K-1}\sum_{i=1}^{K-1}e^{\tau r_t(X_i)}}
\xrightarrow[]{\mathrm{a.s.}}
\frac{e^{\tau r_t(x)}}{\mathbb{E}_{z\sim p_t}[e^{\tau r_t(z)}]}.
$$
Therefore,
$$
H^K_{p_t,w_t}(x)\xrightarrow{K\to\infty}
\frac{e^{\tau r_t(x)}}{\mathbb{E}_{z\sim p_t}[e^{\tau r_t(z)}]}.
$$
Consequently, in the large-$K$ limit,
$$
p_{t+1}(x)
\xrightarrow{K\to\infty}
p_t(x)\frac{e^{\tau r_t(x)}}{\mathbb{E}_{z\sim p_t}[e^{\tau r_t(z)}]}
$$

Moreover, for every $K\ge2$,
$$
0\le H^K_{p_t,w_t}(x)
\le
\frac{K e^\tau}{e^{-\tau}+(K-1)e^{-\tau}}
\le 2e^{2\tau}.
$$
Hence the family $\{H^K_{p_t,w_t}(x)\}_{K\ge2}$ is uniformly bounded. By dominated convergence,
$$
\int_{\mathcal{X}}\left|
H^K_{p_t,w_t}(x)-\frac{e^{\tau r_t(x)}}{\mathbb{E}_{z\sim p_t}[e^{\tau r_t(z)}]}
\right|p_t(dx)\to 0.
$$
This proves the total-variation convergence of the curated distribution to its large-$K$ exponential-reweighting limit.

\paragraph{Case 2: $\eta<1$.}
When $\eta<1$, the objective in \eqref{eq:retrain} can be written as
$$
(1-\eta)\mathbb{E}_{x\sim p_{\mathrm{ref}}}[\log p(x)]
+
\eta \mathbb{E}_{x\sim p_t^{\mathrm{cur}}}[\log p(x)]
=
\mathbb{E}_{x\sim q_t}[\log p(x)],
$$
where
$$
q_t:=(1-\eta)p_{\mathrm{ref}}+\eta\,p_t^{\mathrm{cur}}.
$$
Applying the same population MLE argument yields
$$
p_{t+1}=q_t=(1-\eta)p_{\mathrm{ref}}+\eta\,p_t^{\mathrm{cur}}.
$$

For finite $K$, since
$$
p_t^{\mathrm{cur}}(x)=p_t(x)\,H^K_{p_t,w_t}(x),
$$
we obtain
$$
\begin{aligned}
p_{t+1}(x)
&=
(1-\eta)p_{\mathrm{ref}}(x)+\eta\,p_t^{\mathrm{cur}}(x) \\
&=
(1-\eta)p_{\mathrm{ref}}(x)
+
\eta\,p_t(x)\,H^K_{p_t,w_t}(x) \\
&=
(1-\eta)p_{\mathrm{ref}}(x)
+
\eta\,p_t(x)\,
\mathbb{E}_{X_1,\dots,X_{K-1}\overset{\mathrm{i.i.d.}}{\sim}p_t} \left[
\frac{K e^{\tau r_t(x)}}{e^{\tau r_t(x)}+\sum_{i=1}^{K-1}e^{\tau r_t(X_i)}}
\right].
\end{aligned}
$$

As $K\to\infty$,
$$
H^K_{p_t,w_t}(x)\to
\frac{e^{\tau r_t(x)}}{\mathbb{E}_{z\sim p_t}[e^{\tau r_t(z)}]},
$$
and therefore
$$
p_{t+1}(x)
\xrightarrow{K\to\infty}
(1-\eta)p_{\mathrm{ref}}(x)
+
\eta\,p_t(x)\frac{e^{\tau r_t(x)}}{\mathbb{E}_{z\sim p_t}[e^{\tau r_t(z)}]}.
$$
\end{proof}

\subsection{Proof of Propostion~\ref{prop:pure-basic}}
\label{subsec:prof_pure-basic}
\begin{tcolorbox}[graythm]
\PureBasic*
\end{tcolorbox}

\begin{proof}
Write $$a(x) := e^{\tau r_w(x)}$$
For a fixed  $X_1, \dots, X_{K-1} \overset{\text{iid}}{\sim} p$, let $$S:= \sum_{l=1}^{K-1} a(X_l) \ge 0$$
Since $a(X_\ell)>0$ and $K\ge 2$, we have $S>0$

So $$
H_{p,w}^K(x)=
\mathbb{E}_{X_1,\dots,X_{K-1}\overset{\mathrm{i.i.d.}}{\sim}p} \left[
\frac{K e^{\tau r_w(x)}}{e^{\tau r_w(x)}+\sum_{i=1}^{K-1}e^{\tau r_w(X_i)}}
\right]
=\mathbb{E}_{X_1,\dots,X_{K-1}\overset{\mathrm{i.i.d.}}{\sim}p} \frac{Ka(x)}{a(x) + S}$$
We prove the three claims in order.

\paragraph{Part 1: strict monotonicity.}
Define
$$
f_S(a):=\frac{Ka}{a+S}
$$
Then
$$
f_S'(a)=\frac{K(a+S)-Ka}{(a+S)^2}
=\frac{KS}{(a+S)^2}>0
$$
Hence $f_S(a)$ is strictly increasing in $a\in(0,\infty)$.

On the other hand,
$$
a'(x)
=
\frac{d}{dr_w(x)} e^{\tau r_w(x)}
=
\tau e^{\tau r_w(x)}
>0
$$
so $a(x)$ is strictly increasing in $r_w(x)$.

Therefore, for every fixed sample path,
$$
r_w(x_i)>r_w(x_j)
\quad\Longrightarrow\quad
a(x_i)>a(x_j)
\quad\Longrightarrow\quad
f_S(a(x_i))>f_S(a(x_j))
$$
That is,
$$
r_w(x_i)>r_w(x_j)
\quad\Longrightarrow\quad
\frac{K e^{\tau r_w(x_i)}}{e^{\tau r_w(x_i)}+S}
>
\frac{K e^{\tau r_w(x_j)}}{e^{\tau r_w(x_j)}+S}
$$
Taking expectation over $X_1,\dots,X_{K-1}$ yields
$$
r_w(x_i)>r_w(x_j)
\quad\Longrightarrow\quad
H^K_{p,w}(x_i)>H^K_{p,w}(x_j)
$$

Conversely, since $H^K_{p,w}(x)$ is a strictly increasing function of $r_w(x)$, it also reflects
strict order, and hence
$$
r_w(x_i)>r_w(x_j)
\quad\Longleftrightarrow\quad
H^K_{p,w}(x_i)>H^K_{p,w}(x_j)
$$

For $K=\infty$, we have
$$
H_{p,w}^\infty(x)=\frac{e^{\tau r_w(x)}}{\mathbb{E}_{z\sim p}[e^{\tau r_w(z)}]}
$$
where ${\mathbb{E}_{z\sim p}[e^{\tau r_w(z)}]}$ is independent of $x$. Thus strict monotonicity is immediate.
 
\paragraph{Part 2: normalization.}
For sample$X_1,\dots,X_K \overset{\mathrm{i.i.d.}}{\sim} p_t$, let $i_K\in[K]$ be selected according to
$$
\Pr(i_K=i\mid x_1,\dots,x_K)
=
\frac{e^{\tau r_t(x_i)}}{\sum_{j=1}^K e^{\tau r_t(x_j)}}
$$
Define the retained sample by $\hat X:=X_{i_K}$.
By the derivation in Lemma~\ref{lem:explicit-update},
the law of $\hat X$ is precisely
$$
\Pr(\hat X=x_i)=p(x_i)H^K_{p,w}(x_i)
$$
Therefore,
$$
\sum_{i=1}^n p(x_i)H^K_{p,w}(x_i)
=
\sum_{i=1}^n \Pr(\hat X=x_i)
=
1
$$

For $K=\infty$, we have
$$
H_{p,w}^\infty(x)=\frac{e^{\tau r_w(x)}}{\mathbb{E}_{z\sim p}[e^{\tau r_w(z)}]}
$$
so
$$
\sum_{i=1}^n p(x_i)H_{p,w}^\infty(x_i)
=
\frac{\sum_{i=1}^n p(x_i){e^{\tau r_w(x_i)}})}{\mathbb{E}_{z\sim p}[e^{\tau r_w(z)}]}
=
1
$$
 
\paragraph{Part 3: explicit odds contraction.}
Write
$$
a(x_i):=e^{\tau r_w(x_i)}
\qquad
a(x_j):=e^{\tau r_w(x_j)}
$$
If $r_w(x_i)-r_w(x_j)\ge r_\Delta>0$, then
$$
a(x_j) \le e^{-\tau r_\Delta}a(x_i)
$$
and therefore
$$
a(x_i)-a(x_j) \ge (1-e^{-\tau r_\Delta})a(x_i)
$$

Let
$$S:= \sum_{l=1}^{K-1} a(X_l) \ge 0$$
For any $S\ge 0$,
$$
\frac{a(x_j)/(a(x_j)+S)}{a(x_i)/(a(x_i)+S)}
=
\frac{a(x_j)(a(x_i)+S)}{a(x_i)(a(x_j)+S)}
=
1-\frac{S \cdot (a(x_i)-a(x_j))}{a(x_i)(a(x_j)+S)}
$$
So,
\begin{align}
    \frac{H^K_{p,w}(x_j)}{H^K_{p,w}(x_i)}&=
\frac{\mathbb{E}_{X_1,\dots,X_{K-1}\overset{\mathrm{i.i.d.}}{\sim}p} \left[
\frac{K e^{\tau r_w(x_j)}}{e^{\tau r_w(x_j)}+\sum_{i=1}^{K-1}e^{\tau r_w(X_i)}}\right]}
{\mathbb{E}_{X_1,\dots,X_{K-1}\overset{\mathrm{i.i.d.}}{\sim}p} \left[
\frac{K e^{\tau r_w(x_i)}}{e^{\tau r_w(x_i)}+\sum_{i=1}^{K-1}e^{\tau r_w(X_i)}}\right]} \notag\\
&=1-
\mathbb{E}_{X_1,\dots,X_{K-1}\overset{\mathrm{i.i.d.}}{\sim}p} \frac{S \cdot (a(x_i)-a(x_j))}{a(x_i)(a(x_j)+S)} 
\label{eq:ratio-identity-clean}
\end{align}

Since $r_w(x)\in[-1,1]$, each $a(x)\in[e^{-\tau},e^{\tau}]$.

Using $a(x_i)-a(x_j) \ge (1-e^{-\tau r_\Delta})a(x_i)
$ and $a(x_j)\le e^\tau$, we obtain
$$
\frac{S \cdot (a(x_i)-a(x_j))}{a(x_i)(a(x_j)+S)}
\ge
(1-e^{-\tau r_\Delta})\frac{S}{e^\tau+S}
$$
Since $a(X_\ell)\ge e^{-\tau}$ so
$$
S \ge (K-1)e^{-\tau}
$$
Because the map $s\mapsto \frac{s}{(e^\tau+s)}$ is increasing,
$$
\mathbb{E}_{X_1,\dots,X_{K-1}\overset{\mathrm{i.i.d.}}{\sim}p} \frac{S}{e^\tau+S} 
\ge
\frac{(K-1)e^{-\tau}}{e^\tau+(K-1)e^{-\tau}}
=
\frac{K-1}{e^{2\tau}+K-1}
$$
Substituting this into \eqref{eq:ratio-identity-clean} yields
$$
\frac{H^K_{p,w}(x_j)}{H^K_{p,w}(x_i)}
\le
1-(1-e^{-\tau r_\Delta})\frac{K-1}{e^{2\tau}+K-1}
=: \lambda_K(r_\Delta)
$$

For $K=\infty$, the claim is immediate:
$$
\frac{H_{p,w}^\infty(x_j)}{H_{p,w}^\infty(x_i)}
=
\frac{\frac{e^{\tau r_w(x_j)}}{\mathbb{E}_{z\sim p}[e^{\tau r_w(z)}]}}{\frac{e^{\tau r_w(x_i)}}{\mathbb{E}_{z\sim p}[e^{\tau r_w(z)}]}}
=
\frac{a(x_j)}{a(x_i)}
\le e^{-\tau r_\Delta}
$$

Finally,
$$
\lambda_K(r_\Delta)\xrightarrow{K\to\infty} e^{-\tau r_\Delta}
$$

And for $K\ge 2$, $r_\Delta>0$,
$$
0<\lambda_K(r_\Delta)<1.
$$
This completes the proof.
\end{proof}

\subsection{Proof of Theorem~\ref{thm:pure-structure}}
\label{subsec:prof_pure-structure}
\begin{tcolorbox}[graythm]
\PureStructure*
\end{tcolorbox}
\begin{proof}
We prove the two claims in turn.

\paragraph{Part 1: Data fixed-point characterization.}
\emph{($\Rightarrow$)} Suppose $p$ is a fixed point of the data update at preference $w$, i.e., for all $x\in\mathcal X$,
$$
p(x)H^K_{p,w}(x)=p(x)
$$
Then for every $x$ with $p(x)>0$,
$$
H^K_{p,w}(x)=1
$$
Now let $x,x'\in\mathcal X$ satisfy $p(x)>0$ and $p(x')>0$. If
$$
r_w(x)\neq r_w(x')
$$
then by Proposition~\ref{prop:pure-basic} (strict monotonicity),
$$
r_w(x)>r_w(x')
\quad\Longrightarrow\quad
H^K_{p,w}(x)>H^K_{p,w}(x')
$$
or vice versa.

But both $x$ and $x'$ lie in the support of $p$, so we already showed
$$
H^K_{p,w}(x)=H^K_{p,w}(x')=1
$$
a contradiction. Therefore,
$$
p(x)>0,\ p(x')>0
\quad\Longrightarrow\quad
r_w(x)=r_w(x')
$$

\emph{($\Leftarrow$)} Conversely, suppose that for all $x,x'\in\mathcal X$.
$$
p(x)>0,\ p(x')>0
\quad\Longrightarrow\quad
r_w(x)=r_w(x')
$$
Let $c$ denote the common reward value on the support of $p$, i.e., for every $x$ with $p(x)>0$.
$$
r_w(x)=c
$$
Then for every such $x$,
$$
e^{\tau r_w(x)}=e^{\tau c}
$$

Hence, for any $x$ in the support of $p$,
$$
H^K_{p,w}(x)
=
\mathbb E\left[
\frac{K e^{\tau c}}{e^{\tau c}+\sum_{\ell=1}^{K-1}e^{\tau c}}
\right]
=
1
$$
Therefore,
$$
p(x)H^K_{p,w}(x)=p(x)
$$
so $p$ is a fixed point of the data update.

For $K=\infty$, the same conclusion holds since
$$
H_{p,w}^\infty(x)
=
\frac{e^{\tau c}}{\mathbb E_{z\sim p}[e^{\tau r_w(z)}]}
=
1
$$
for every $x$ with $p(x)>0$.

\paragraph{Part 2: Singleton equilibria.}
Fix $i\in\{1,\dots,n\}$, and define $p^\star$ by
$$
p^\star(x_i)=1,
\qquad
p^\star(x_j)=0\ \text{for all }j\neq i,
$$
and let
$$
w^\star=\varphi(x_i)
$$
Since $p^\star$ is supported on a single point, the condition in Part 1 is trivially satisfied, so $p^\star$ is a fixed point of the data update.

Moreover,
$$
\bar\varphi(p^\star)=\varphi(x_i)=w^\star
$$
Therefore,
$$
\operatorname{Proj}_{\mathcal W}\!\left((1-\beta)w^\star+\beta\bar\varphi(p^\star)\right)
=
\operatorname{Proj}_{\mathcal W}(\varphi(x_i))
=
\varphi(x_i)
=
w^\star
$$
where we used that $\|\varphi(x_i)\|=1$ and $\mathcal W$ is the closed unit ball. 

Hence $(p^\star,w^\star)$ satisfies both parts of \eqref{eq:equilibrium}, and is therefore an equilibrium of the pure coupled dynamics.
\end{proof}

\subsection{Proof of Theorem~\ref{thm:pure-lockin}}
\label{subsec:prof_pure-lockin}
\begin{tcolorbox}[graythm]
\PureLockin*
\end{tcolorbox}
\begin{proof}
Let
$$
q_t:=1-p_t(x_i)= \sum_{j \neq i} p_t(x_j)
$$
denote the off-winner mass.

Since $\|\varphi(x)\|=1$ for all $x$ and
$\varphi(x_j)\neq \varphi(x_i)$ for every $j\neq i$, we have
$$
\gamma_i
=
\min_{j\neq i}
\langle \varphi(x_i),\,\varphi(x_i)-\varphi(x_j)\rangle
=
\min_{j\neq i}\bigl(1-\langle \varphi(x_i),\varphi(x_j)\rangle\bigr)
>0
$$

Fix $K$, $\lambda_K(\frac{\gamma_i}{2})\in(0,1)$, where $\lambda_K(\cdot)$ is the odds-contraction factor from
Proposition~\ref{prop:pure-basic}.

Choose
$$
\varepsilon_i\in\left(0,\min\left\{\frac12,\ 1-\lambda_K(\frac{\gamma_i}{2}),\ \frac{\gamma_i}{8}\right\}\right),
$$
and define
$$
U_i
:=
\left\{
(p,w):
1-p(x_i)\le \varepsilon_i,\ 
\|w-\varphi(x_i)\|\le \frac{\gamma_i}{4}
\right\}.
$$
We show that $U_i$ is forward invariant and that $q_t$ decays geometrically.

\paragraph{Uniform reward gap inside $U_i$.}
Suppose $(p_t,w_t)\in U_i$. 
Using $r_{w_t}(x)=\langle w_t,\varphi(x)\rangle$, we have
$$
r_{w_t}(x_i)-r_{w_t}(x_j)
=
\langle w_t,\varphi(x_i)-\varphi(x_j)\rangle
$$
Writing
$$
w_t=\varphi(x_i)+\bigl(w_t-\varphi(x_i)\bigr)
$$
we obtain
$$
r_{w_t}(x_i)-r_{w_t}(x_j)
=
\langle \varphi(x_i),\varphi(x_i)-\varphi(x_j)\rangle
+
\langle w_t-\varphi(x_i),\varphi(x_i)-\varphi(x_j)\rangle
$$
for every $j\neq i$.

Hence
$$
r_{w_t}(x_i)-r_{w_t}(x_j)
\ge
\gamma_i-\|w_t-\varphi(x_i)\|\|\varphi(x_i)-\varphi(x_j)\|
$$
Since $\|\varphi(x_i)-\varphi(x_j)\|\le 2$ and
$\|w_t-\varphi(x_i)\|\le \gamma_i/4$, we obtain
$$
r_{w_t}(x_i)-r_{w_t}(x_j)\ge \gamma_i-2\cdot \frac{\gamma_i}{4}=\frac{\gamma_i}{2}
$$

\paragraph{Odds contraction.}
By Proposition~\ref{prop:pure-basic}, for every $j\neq i$,
$$
\frac{H^K_{p_t,w_t}(x_j)}{H^K_{p_t,w_t}(x_i)}
\le
\lambda_K(\frac{\gamma_i}{2})
$$
Since the pure data update is
$$
p_{t+1}(x)=p_t(x)H^K_{p_t,w_t}(x),
$$
we have
$$
\frac{p_{t+1}(x_j)}{p_{t+1}(x_i)}
=
\frac{p_t(x_j)}{p_t(x_i)}
\cdot
\frac{H^K_{p_t,w_t}(x_j)}{H^K_{p_t,w_t}(x_i)}
\le
\lambda_K(\frac{\gamma_i}{2}) \frac{p_t(x_j)}{p_t(x_i)}
$$
Summing over all \(j\neq i\) gives
$$
\sum_{j \neq i } \frac{p_{t+1}(x_j)}{p_{t+1}(x_i)} \le \sum_{j \neq i } \lambda_K(\frac{\gamma_i}{2}) \frac{p_t(x_j)}{p_t(x_i)}
$$
that is,
$$
\frac{q_{t+1}}{1-q_{t+1}}
\le
\lambda_K(\frac{\gamma_i}{2}) \frac{q_t}{1-q_t}
$$
Since $q_t\le \varepsilon_i$, this implies
$$
q_{t+1}
\le
\frac{\lambda_K(\frac{\gamma_i}{2})}{1-q_t}q_t
\le
\frac{\lambda_K(\frac{\gamma_i}{2})}{1-\varepsilon_i}q_t
$$
Define
$$
\bar\lambda_K:=\frac{\lambda_K(\frac{\gamma_i}{2})}{1-\varepsilon_i}
$$
Because $\varepsilon_i \leq 1-\lambda_K(\frac{\gamma_i}{2})$, we have $\bar\lambda_K\in(0,1]$. Thus
$$
q_{t+1}\le \bar\lambda_K q_t
$$
Iterating yields
$$
q_t\le \bar\lambda_K^t q_0
$$
that is,
$$
1-p_t(x_i)\le \bar\lambda_K^t\bigl(1-p_0(x_i)\bigr)
$$

\paragraph{Forward invariance of $U_i$.}
From the geometric bound,
$$
q_{t+1}\le \bar\lambda_K q_t\le \bar\lambda_K \varepsilon_i\le \varepsilon_i
$$
So the first defining condition of \(U_i\) is preserved.

Next, since
$$
\bar\varphi(p_{t+1})-\varphi(x_i)
=
\sum_{j\neq i}p_{t+1}(x_j)\bigl(\varphi(x_j)-\varphi(x_i)\bigr),
$$
we obtain
$$
\|\bar\varphi(p_{t+1})-\varphi(x_i)\|
\le
\sum_{j\neq i}p_{t+1}(x_j)\|\varphi(x_j)-\varphi(x_i)\|
\le
2q_{t+1}
$$
Using the preference update,
$$
w_{t+1}
=
\operatorname{Proj}_{\mathcal W}\!\left((1-\beta)w_t+\beta\bar\varphi(p_{t+1})\right)
$$
and
$$
\|w_{t+1}-\varphi(x_i)\|
\le
(1-\beta)\|w_t-\varphi(x_i)\|
+
\beta\|\bar\varphi(p_{t+1})-\varphi(x_i)\|
$$
Hence
$$
\|w_{t+1}-\varphi(x_i)\|
\le
(1-\beta)\frac{\gamma_i}{4}
+
2\beta q_{t+1}
\le
(1-\beta)\frac{\gamma_i}{4}
+
2\beta\varepsilon_i
$$
Since \(\varepsilon_i\le \gamma_i/8\), we get
$$
\|w_{t+1}-\varphi(x_i)\|
\le
(1-\beta)\frac{\gamma_i}{4}
+
\beta\frac{\gamma_i}{4}
=
\frac{\gamma_i}{4}
$$
So the second defining condition of $U_i$ is also preserved. Therefore $U_i$ is forward invariant.

\paragraph{Convergence of \(w_t\).}
We already proved \(q_t\to 0\). Using again
$$
\|\bar\varphi(p_{t+1})-\varphi(x_i)\|\le 2q_{t+1},
$$
the preference recursion implies
$$
\|w_{t+1}-\varphi(x_i)\|
\le
(1-\beta)\|w_t-\varphi(x_i)\|+2\beta q_{t+1}
$$
Unrolling this recursion gives
$$
\|w_t-\varphi(x_i)\|
\le
(1-\beta)^t\|w_0-\varphi(x_i)\|
+
2\beta\sum_{s=0}^{t-1}(1-\beta)^{t-1-s}q_{s+1}
$$
Since $q_s\to 0$ geometrically, the right-hand side tends to $0$. Hence
$$
w_t\to \varphi(x_i)
$$
Together with $q_t\to 0$, this implies
$$
p_t(x_i)\xrightarrow{t\to\infty} 1,
\qquad
w_t \xrightarrow{t\to\infty} \varphi(x_i).
$$

\end{proof}

\subsection{Proof of Corollary~\ref{cor:pure-duplicate}}
\label{subsec:prof_pure-duplicate}
\begin{tcolorbox}[graythm, breakable]
\PureDuplicate*
\end{tcolorbox}

\begin{proof}
Assume there exists $j\neq i$ such that
$$
\varphi(x_j)=\varphi(x_i)
$$
For any $\varepsilon\in(0,1)$, define
$$
p^\varepsilon(x_\ell)=
\begin{cases}
1-\varepsilon, & \ell=i,\\
\varepsilon, & \ell=j,\\
0, & \text{otherwise},
\end{cases}
\qquad
w^\varepsilon:=\varphi(x_i)
$$
Since $\varphi(x_j)=\varphi(x_i)$, we have
$$
r_{w^\varepsilon}(x_j)=r_{w^\varepsilon}(x_i),
$$
so all points in the support of $p^\varepsilon$ have the same reward.

By Theorem~\ref{thm:pure-structure}, \(p^\varepsilon\) is a fixed point of the data update.

Moreover,
$$
\bar\varphi(p^\varepsilon)
=
(1-\varepsilon)\varphi(x_i)+\varepsilon\varphi(x_j)
=
\varphi(x_i)
=
w^\varepsilon,
$$
so the preference update also fixes $w^\varepsilon$. 

Hence $(p^\varepsilon,w^\varepsilon)$ is an equilibrium for every $\varepsilon\in(0,1)$.

As $\varepsilon\to 0$, the equilibria $(p^\varepsilon,w^\varepsilon)$ converge to the singleton equilibrium at $x_i$.

Therefore this singleton equilibrium is not isolated, and hence cannot be locally asymptotically stable.
\end{proof}

\subsection{Proof of Proposition~\ref{prop:reinj-lipschitz}}
\label{subsec:prof_reinj-lipschitz}
\begin{tcolorbox}[graythm]
\ReinjLip*
\end{tcolorbox}
\begin{proof}
Recall that
$$
a(x):=e^{\tau r_w(x)}
$$
Since $r_w(x)\in[-1,1]$, we have
$$
e^{-\tau}\le a(x)\le e^\tau
$$

For finite $K$,
$$
H_{p,w}^K(x)
=
\mathbb E_{X_1,\dots,X_{K-1}\overset{\mathrm{iid}}{\sim}p}
\!\left[
\frac{K a(x)}{a(x)+\sum_{\ell=1}^{K-1}a(X_\ell)}
\right]
$$
For $K=\infty$,
$$
H_{p,w}^\infty(x)
=
\frac{a_(x)}{\mathbb E_{z\sim p}[a(z)]}
$$

We prove the three bounds in turn.

\paragraph{Part 1: Lipschitz continuity in $p$.}
Fix $w\in\mathcal W$.
$$
2d_{\mathrm{TV}}(pH^K_{p,w},\;qH^K_{q,w})
=
\sum_{x\in\mathcal X}
\left|
p(x)H^K_{p,w}(x)-q(x)H^K_{q,w}(x)
\right|
$$
For each \(x\in\mathcal X\),
$$
p(x)H^K_{p,w}(x)-q(x)H^K_{q,w}(x)
=
(p(x)-q(x))H^K_{p,w}(x)
+
q(x)\bigl(H^K_{p,w}(x)-H^K_{q,w}(x)\bigr)
$$
Hence
$$
2d_{\mathrm{TV}}(pH^K_{p,w},\;qH^K_{q,w})
\le I+II,
$$
where
$$
I:=\sum_{x\in\mathcal X}|p(x)-q(x)|\,H^K_{p,w}(x),
\qquad
II:=\sum_{x\in\mathcal X}q(x)\,\bigl|H^K_{p,w}(x)-H^K_{q,w}(x)\bigr|.
$$

We first bound $I$. Since for all $x,\ p,\ w,\ K,$
$$
0\le H^K_{p,w}(x)\le 2e^{2\tau}
$$
we have
$$
I
\le
2e^{2\tau}\sum_{x\in\mathcal X}|p(x)-q(x)|
=
4e^{2\tau}d_{\mathrm{TV}}(p,q)
$$

Next we bound $II$. Fix $x\in\mathcal X$. For finite $K$, define
$$
g_x(y_1,\dots,y_{K-1})
:=
\frac{K a_w(x)}{a_w(x)+\sum_{\ell=1}^{K-1}a_w(y_\ell)}
$$
Then
$$
H^K_{p,w}(x)=\mathbb E_{p^{K-1}}[g_x],
\qquad
H^K_{q,w}(x)=\mathbb E_{q^{K-1}}[g_x]
$$

We compare these two expectations by replacing
the coordinate distributions one at a time.
Define, for $\ell = 0, 1, \dots, K-1$,
$$
    M_\ell
    := \mathbb{E}\bigl[
        g_x(Y_1,\dots,Y_{K-1})
    \bigr],
    \qquad
    Y_j \sim
    \begin{cases}
        q & \text{if } j \le \ell, \\
        p & \text{if } j > \ell.
    \end{cases}
$$
Then $M_0 = H^K_{p,w}(x)$, $M_{K-1} = H^K_{q,w}(x)$, and
$$
H^K_{p,w}(x) - H^K_{q,w}(x) = \sum_{\ell=1}^{K-1}(M_{\ell-1} - M_\ell)
$$
The difference $M_{\ell-1} - M_\ell$ isolates the effect
of switching coordinate $\ell$ from $p$ to $q$, with all
other coordinates unchanged.

In $M_{\ell-1} - M_\ell$, condition on all coordinates
except $y_\ell$.
Write $S_{-\ell} := \sum_{m \ne \ell} a(y_m)$ for their
total
As a function of $y_\ell$ alone, the integrand becomes
$$
h(y_\ell):= \frac{K\,a(x)}{a(x) + a(y_\ell) + S_{-\ell}}.
$$
Since $h$ is decreasing in $a(y_\ell)$ and
$a(y_\ell) \in [e^{-\tau},\,e^{\tau}]$:
$$
\sup_{y} h(y) - \inf_{y} h(y)  = K\,a(x)
      \cdot\frac{e^{\tau} - e^{-\tau}}
           {(a(x)+e^{-\tau}+S_{-\ell})
            (a(x)+e^{\tau}+S_{-\ell})}
$$
We bound the numerator and denominator separately:

\emph{Numerator:} Using $a(x) \le e^{\tau}$ and $e^{\tau}-e^{-\tau} \le 2e^{\tau}$.

$$K\,a(x)\,(e^{\tau}-e^{-\tau})
\le K\,e^{\tau}\cdot 2e^{\tau} = 2K\,e^{2\tau}$$  

\emph{Denominator:} each factor satisfies
    $$
    a(x)+e^{\pm\tau}+S_{-\ell}
    \ge e^{-\tau}+e^{-\tau}+(K-2)\,e^{-\tau} = K\,e^{-\tau}
    $$
    so the product is $\ge K^2\,e^{-2\tau}$.

Therefore
$$
\sup_{y} h(y) - \inf_{y} h(y)
    \le \frac{2K\,e^{2\tau}}{K^2\,e^{-2\tau}}
    = \frac{2e^{4\tau}}{K}.
$$
This bound holds for every realization of the
conditioned coordinates, since it uses only the
worst-case $S_{-\ell} \ge (K-2)\,e^{-\tau}$.

Now, for any bounded function $h$ and any two
distributions $p, q$,
$$
    \bigl|H^K_{p,w}(x) - H^K_{q,w}(x)\bigr|
    \le (\sup_{y} h(y) - \inf_{y} h(y)) \cdot d_{\mathrm{TV}}(p,q).
$$
(This follows from the total-variation duality:
center $h$ so that $\|h - \text{const}\|_\infty
\le (\sup_{y} h(y) - \inf_{y} h(y))/2$, then apply
$|\mathbb{E}_p[f]-\mathbb{E}_q[f]|
\le 2\|f\|_\infty\,d_{\mathrm{TV}}(p,q)$.)

So:
$$
 \bigl|H^K_{p,w}(x) - H^K_{q,w}(x)\bigr|
    \le \sum_{\ell=1}^{K-1}
        \frac{2e^{4\tau}}{K}\,d_{\mathrm{TV}}(p,q)
    = \frac{(K-1)\cdot 2e^{4\tau}}{K}\,d_{\mathrm{TV}}(p,q)
    < 2e^{4\tau}\,d_{\mathrm{TV}}(p,q)
$$
This is a pointwise bound on $x$, uniform in $K$.

Hence
$$
II
\le
2e^{4\tau}d_{\mathrm{TV}}(p,q).
$$

Combining the bounds on $I$ and $II$, we get
$$
2d_{\mathrm{TV}}(pH^K_{p,w},\;qH^K_{q,w})
\le
\bigl(4e^{2\tau}+2e^{4\tau}\bigr)d_{\mathrm{TV}}(p,q).
$$
Equivalently,
$$
d_{\mathrm{TV}}(pH^K_{p,w},\;qH^K_{q,w})
\le
L_p\,d_{\mathrm{TV}}(p,q),
$$
where
$$
L_p:=2e^{2\tau}(2+e^{2\tau})
$$

For \(K=\infty\), the same inequality remains valid. Indeed,
$$
H_{p,w}^\infty(x)
=
\frac{a_(x)}{\mathbb E_{z\sim p}[a(z)]}
$$
and the dependence on $p$ is simpler.

\paragraph{Part 2: Lipschitz continuity in $w$.}
Fix $p\in\mathcal P(\mathcal X)$. For finite $K$, use the same samples
\(X_1,\dots,X_{K-1}\overset{\mathrm{iid}}{\sim}p\), and write
$$
S_w:=\sum_{\ell=1}^{K-1}a_w(X_\ell),
\qquad
S_v:=\sum_{\ell=1}^{K-1}a_v(X_\ell).
$$
Then for any \(x\in\mathcal X\),
\begin{align*}
\left|
\frac{K a_w(x)}{a_w(x)+S_w}
-
\frac{K a_v(x)}{a_v(x)+S_v}
\right|
&\le
\left|
\frac{K a_w(x)}{a_w(x)+S_w}
-
\frac{K a_v(x)}{a_v(x)+S_w}
\right| \\
&\quad+
\left|
\frac{K a_v(x)}{a_v(x)+S_w}
-
\frac{K a_v(x)}{a_v(x)+S_v}
\right|
\end{align*}

For the first term, since
$$
|a_w(x)-a_v(x)|
\le
\tau e^\tau\|w-v\|
$$
we have
$$
\left|
\frac{K a_w(x)}{a_w(x)+S_w}
-
\frac{K a_v(x)}{a_v(x)+S_w}
\right|
=
\frac{K S_w |a_w(x)-a_v(x)|}{(a_w(x)+S_w)(a_v(x)+S_w)}.
$$
Using
$$
S_w\le (K-1)e^\tau,
\qquad
a_w(x)+S_w\ge Ke^{-\tau},
\qquad
a_v(x)+S_w\ge Ke^{-\tau},
$$
we obtain
$$
\left|
\frac{K a_w(x)}{a_w(x)+S_w}
-
\frac{K a_v(x)}{a_v(x)+S_w}
\right|
\le
\tau e^{4\tau}\|w-v\|.
$$

For the second term,
$$
|S_v-S_w|
\le
\sum_{\ell=1}^{K-1}|a_v(X_\ell)-a_w(X_\ell)|
\le
(K-1)\tau e^\tau\|w-v\|.
$$
Therefore
$$
\left|
\frac{K a_v(x)}{a_v(x)+S_w}
-
\frac{K a_v(x)}{a_v(x)+S_v}
\right|
=
\frac{K a_v(x)|S_v-S_w|}{(a_v(x)+S_w)(a_v(x)+S_v)}
\le
\tau e^{4\tau}\|w-v\|.
$$

Combining the two bounds and taking expectation gives
$$
|H^K_{p,w}(x)-H_{p_v}^K(x)|
\le
2\tau e^{4\tau}\|w-v\|.
$$
Hence
$$
d_{\mathrm{TV}}(pH^K_{p,w},\;pH_{p_v}^K)
\le
L_w\,\|w-v\|,
$$
where
$$
L_w:=2\tau e^{4\tau}.
$$

For \(K=\infty\), the same bound follows directly from
$$
H_{p,w}^\infty(x)
=
\frac{a_(x)}{\mathbb E_{z\sim p}[a(z)]}
$$

\paragraph{Part 3: Lipschitz continuity of $\bar\varphi$.}
Let $u\in\mathbb R^d$ satisfy $\|u\|=1$. Since $\|\varphi(x)\|=1$,
$$
|\langle u,\varphi(x)\rangle|\le 1
$$
Therefore
$$
\left|\left\langle u,\bar\varphi(p)-\bar\varphi(q)\right\rangle\right|
=
\left|
\sum_{x\in\mathcal X}\langle u,\varphi(x)\rangle\bigl(p(x)-q(x)\bigr)
\right|
\le
2d_{\mathrm{TV}}(p,q).
$$
Taking the supremum over all unit vectors \(u\) yields
$$
\|\bar\varphi(p)-\bar\varphi(q)\|
\le
2d_{\mathrm{TV}}(p,q).
$$
This completes the proof.
\end{proof}

\subsection{Proofs of Lemma~\ref{thm:frozen-w} and Proposition~\ref{prop:kinfty-closed-form}}
\label{subsec:prof_frozen-w}
\begin{tcolorbox}[graythm]
\FrozenW*
\end{tcolorbox}
\begin{proof}
Fix $w\in\mathcal W$. For any $p,q\in\mathcal P(\mathcal X)$, define
$$
T_w^K(p)(x)
:=
(1-\eta)p_{\mathrm{ref}}(x)
+
\eta\,p(x)H^K_{p,w}(x)
$$
We show that $T_w^K$ is a contraction on $(\mathcal P(\mathcal X),d_{\mathrm{TV}})$.

Let $p,q\in\mathcal P(\mathcal X)$. Then
\begin{align*}
d_{\mathrm{TV}}(T_w^K(p),T_w^K(q))
&=
d_{\mathrm{TV}}\!\left(
(1-\eta)p_{\mathrm{ref}}+\eta\,pH^K_{p,w},\,
(1-\eta)p_{\mathrm{ref}}+\eta\,qH^K_{q,w}
\right) \\
&=
\eta\, d_{\mathrm{TV}}(pH^K_{p,w},\;qH^K_{q,w}).
\end{align*}
By Proposition~\ref{prop:reinj-lipschitz},
$$
d_{\mathrm{TV}}(pH^K_{p,w},\;qH^K_{q,w})
\le
L_p\,d_{\mathrm{TV}}(p,q).
$$
Therefore,
$$
d_{\mathrm{TV}}(T_w^K(p),T_w^K(q))
\le
\eta L_p\, d_{\mathrm{TV}}(p,q).
$$
If $\eta L_p<1$, then $T_w^K$ is a strict contraction.

Since $\mathcal P(\mathcal X)$ is complete under $d_{\mathrm{TV}}$ in the finite-state setting, the Banach fixed-point theorem implies that $T_w^K$ admits a unique fixed point, denoted $p_w^{K,\eta}$, satisfying
$$
p_w^{K,\eta}(x)
=
(1-\eta)p_{\mathrm{ref}}(x)
+
\eta\,p_w^{K,\eta}(x)H_{(p_w^{K,\eta})_w}^K(x)
$$
Moreover, for any initialization $p_0\in\mathcal P(\mathcal X$, the iterates satisfy
$$
d_{\mathrm{TV}}(p_t,p_w^{K,\eta})
\le
(\eta L_p)^t\,d_{\mathrm{TV}}(p_0,p_w^{K,\eta}),
$$
which proves geometric convergence.
\end{proof}

\begin{tcolorbox}[graythm]
\KinftyClosedForm*
\end{tcolorbox}

\begin{proof}
Recall that
$$
a_w(x):=e^{\tau r_w(x)}
$$
$$
H_{p,w}^\infty(x)
=
\frac{a_w(x)}{\mathbb E_{z\sim p_w}[a_w(z)]}
$$
For fixed $w$, the frozen-$w$ update therefore becomes
$$
p_{t+1}(x_i)
=
(1-\eta)p_{\mathrm{ref}}(x_i)
+
\eta\,p_t(x_i)\frac{a_w(x_i)}{\mathbb E_{z\sim p_t}[a_w(z)]},
$$

Let $p_w^{\infty,\eta}$ be a frozen-$w$ fixed point, then the fixed-point equation is
$$
p_w^{\infty,\eta}
=
(1-\eta)p_{\mathrm{ref}}
+
\eta p_w^{\infty,\eta} \frac{a_w(x)}{\mathbb E_{z\sim p_w^{\infty,\eta}}[a_w(z)]}
$$
Rearranging,
$$
p_w^{\infty,\eta}\left(1-\eta\frac{a_w(x)}{\mathbb E_{p_w^{\infty,\eta}}[a_w(z)]}\right)
=
(1-\eta)p_{\mathrm{ref}},
$$
so
$$
p_w^{\infty,\eta}
=
\frac{(1-\eta)p_{\mathrm{ref}}}{1-\eta \frac{a_w(x)}{\mathbb E_{z\sim p_w^{\infty,\eta}}[a_w(z)]}}
=
(1-\eta)p_{\mathrm{ref}}\,\frac{{\mathbb E_{z\sim p_w^{\infty,\eta}p}[a_w(z)]}}{{\mathbb E_{z\sim p_w^{\infty,\eta}}[a_w(z)]}-\eta {a_w(x)}}.
$$
That is 
$$
p_w^{\infty,\eta}
=
(1-\eta)p_{\mathrm{ref}}\,\frac{{\mathbb E_{z\sim p_w^{\infty,\eta}}[e^{\tau r_w(x)}]}}{{\mathbb E_{z\sim p_w^{\infty,\eta}}[e^{\tau r_w(x)})]}-\eta {e^{\tau r_w(x)}}}.
$$

It remains to determine ${\mathbb E_{z\sim p_w^{\infty,\eta}}[a_w(z)]}$. 
Substituting the formula above, we obtain
$$
{\mathbb E_{z\sim p_w^{\infty,\eta}}[a_w(z)]}
=
\sum_{i=1}^n
(1-\eta)p_{\mathrm{ref}}\,\frac{{\mathbb E_{z\sim p_w^{\infty,\eta}p}[a_w(z)]}}{{\mathbb E_{z\sim p_w^{\infty,\eta}}[a_w(z)]}-\eta {a_w(x_i)}} a_w(x_i)
$$
Since ${\mathbb E_{z\sim p_w^{\infty,\eta}}[a_w(z)]}$, dividing both sides by ${\mathbb E_{z\sim p_w^{\infty,\eta}}[a_w(z)]}$ yields
$$
\sum_{i=1}^n
(1-\eta)p_{\mathrm{ref}}\,\frac{a_w(x_i)}{{\mathbb E_{z\sim p_w^{\infty,\eta}}[a_w(z)]}-\eta {a_w(x_i)}}
=
1
$$

We now show that this equation has a unique solution on
$$
(\eta \max_i a_w(x_i),\infty).
$$
Define
$$
F(Z_w^{\eta})
:=
\sum_{i=1}^n
(1-\eta)\,p_{\mathrm{ref}}(x_i)\,
\frac{a_w(x_i)}{Z_w^{\eta}-\eta a_w(x_i)}.
$$
For $Z_w^{\eta}>\eta\max_i a_w(x_i)$, all denominators are positive, and
$$
F'(Z_w^{\eta})
=
-
\sum_{i=1}^n
(1-\eta)\,p_{\mathrm{ref}}(x_i)\,
\frac{a_w(x_i)}{(Z_w^{\eta}-\eta a_w(x_i))^2}
<0.
$$
Thus $F$ is strictly decreasing on $(\eta\max_i a_w(x_i),\infty)$. Moreover, $F(Z_w^{\eta})\to\infty$ as $Z_w^{\eta}\downarrow \eta\max_i a_w(x_i)$,
while
$$
F(Z_w^{\eta})\to 0
\qquad
\text{as } Z_w^{\eta}\to\infty.
$$
Therefore the equation $F(Z_w^{\eta})=1$ has a unique solution when 
$
Z_w^{\eta}>\eta\max_i a_w(x_i).
$

Substituting this unique $Z_w^{\eta}$ into the formula yields the unique frozen-$w$ fixed point.
\end{proof}

\subsection{Proof of Theorem~\ref{thm:global-reinj}}
\label{subsec:prof_global-reinj}

\begin{tcolorbox}[graythm]
\GlobalReinj*
\end{tcolorbox}
\begin{proof}
Let $(p,w)$ and $(q,v)$ be two states, then:
$$
p_{t+1}(x)
=
(1-\eta)p_{\mathrm{ref}}(x)
+
\eta\,p_{t}(x)H_{p_{t},w}^K(x),
\qquad
q_{t+1}
=
(1-\eta)p_{\mathrm{ref}}(x)
+
\eta\,q_{t}(x)H_{q_{t},v}^K(x),
$$
$$
w_{t+1}
=
\operatorname{Proj}_{\mathcal W}\!\left((1-\beta)w_{t}+\beta\bar\varphi(p_{t+1})\right),
\qquad
v_{t+1}
=
\operatorname{Proj}_{\mathcal W}\!\left((1-\beta)v_{t}+\beta\bar\varphi(q_{t+1})\right)
$$

We first bound the data coordinate. By Proposition~\ref{prop:reinj-lipschitz},
\begin{align*}
d_{\mathrm{TV}}(p_{t+1},p_{t+1})
&=
d_{\mathrm{TV}}\!\left(
(1-\eta)p_{\mathrm{ref}}+\eta\,p_t H_{p_{t},w}^K,,\,
(1-\eta)p_{\mathrm{ref}}+\eta\,q_t H_{q_{t},v}^K,
\right) \\
&=
\eta\,d_{\mathrm{TV}}(p_t H_{p_{t},w}^K,\;q_t H_{q_{t},v}^K) \\
&\le
\eta\,d_{\mathrm{TV}}(p_t H_{p_{t},w}^K,\;q_t H_{q_{tw}}^K)
+
\eta\,d_{\mathrm{TV}}(q_t H_{q_{t},w}^K,\;q_t H_{q_{t},v}^K,) \\
&\le
\eta L_p\,d_{\mathrm{TV}}(p_t,q_t)+\eta L_{w}\,\|{w_t}-{v_t}\|.
\end{align*}

Next we bound the preference coordinate. Using the nonexpansiveness of Euclidean projection,
\begin{align*}
\|{w_{t+1}}-{v_{t+1}}\|
&\le
\left\|
(1-\beta)(w_t-v_t)+\beta\bigl(\bar\varphi(p_{t+1})-\bar\varphi(q_{t+1}\bigr)
\right\| \\
&\le
(1-\beta)\|w_t-v_t\|
+
\beta\|\bar\varphi(p_{t+1})-\bar\varphi(q_{t+1})\|.
\end{align*}
Applying Proposition~\ref{prop:reinj-lipschitz} again,
$$
\|\bar\varphi(p_{t+1})-\bar\varphi(q_{t+1})\|
\le
2\,d_{\mathrm{TV}}(p_{t+1},q_{t+1}),
$$
and hence
$$
\|{w_{t+1}}-{v_{t+1}}\|
\le
(1-\beta)\|w_t-v_t\|+2\beta\,d_{\mathrm{TV}}(p_{t+1},q_{t+1}).
$$
Substituting the bound on $d_{\mathrm{TV}}(p_{t+1},q_{t+1})$, we obtain
$$
\|{w_{t+1}}-{v_{t+1}}\|
\le
2\beta\eta L_p\,d_{\mathrm{TV}}(p_t,q_t)
+
\bigl(1-\beta+2\beta\eta L_w\bigr)\|w_t-v_t\|.
$$

Now multiply this inequality by $\xi$ and add it to the data-coordinate bound. This gives
\begin{align*}
d_\xi\bigl((p_{t+1},w_{t+1}),(q_{t+1},v_{t+1})\bigr)
&=
d_{\mathrm{TV}}(p_{t+1},q_{t+1})+\xi\|w_{t+1}-v_{t+1}\| \\
&\le
\eta L_p\,d_{\mathrm{TV}}(p_t,q_t)+\eta L_{w}\,\|{w_t}-{v_t}\|+\xi \Bigl((1-\beta)\|w_t-v_t\|+2\beta\,d_{\mathrm{TV}}(p_{t+1},q_{t+1}))\\
&\le
\Bigl(\eta L_p+2\beta\eta L_p\,\xi\Bigr)d_{\mathrm{TV}}(p,q)+\xi \Bigl(\frac{\eta L_w}{\xi}+(1-\beta+2\beta\eta L_w)\Bigr)\|w-v\|.
\end{align*}

Therefore the map is a contraction under $d_\xi$ as soon as both coefficients are strictly smaller than the corresponding weights in $d_\xi$:
$$
\left\{
\begin{aligned}
\eta L_p+2\beta\eta L_p\,\xi < 1,\\
\frac{\eta L_w}{\xi} +(1-\beta+2\beta\eta L_w) < 1.
\end{aligned}
\right.
$$
so,
$$
\left\{
\begin{aligned}
\xi<\frac{1-\eta L_p}{2\beta\eta L_p}\\
\xi>\frac{\eta L_w}{\beta(1-2\eta L_w)}
\end{aligned}
\right.
$$
Thus any $\xi$ in the window
$$
\frac{\eta L_w}{\beta(1-2\eta L_w)}
<
\xi
<
\frac{1-\eta L_p}{2\beta\eta L_p}
$$
makes the coupled update a strict contraction.

It remains to check that this window is nonempty. Since all quantities are positive,
\begin{align*}
\frac{\eta L_w}{\beta(1-2\eta L_w)}
<
\frac{1-\eta L_p}{2\beta\eta L_p}
\quad
&\Longleftrightarrow\quad
2\eta^2L_pL_w<(1-2\eta L_w)(1-\eta L_p) \\
&\Longleftrightarrow\quad
0<1-\eta L_p-2\eta L_w \\
&\Longleftrightarrow\quad
\eta(L_p+2L_w)<1.
\end{align*}
So under the stated condition the interval is nonempty, and we may choose such a $\xi$.

Finally, $\mathcal P(\mathcal X)\times\mathcal W$ is complete under $d_\xi$ in the finite-state setting, since $\mathcal P(\mathcal X)$ is complete under $d_{\mathrm{TV}}$ and $\mathcal W$ is closed in finite-dimensional Euclidean space. 

By the Banach fixed-point theorem, the coupled reinjected dynamics admit a unique equilibrium $(p_K^\star,w_K^\star)$, and every trajectory converges to it geometrically.
\end{proof}

\subsection{Proof of Proposition~\ref{prop:protection-suff}}
\label{subsec:prof_protection-suff}

\begin{tcolorbox}[graythm, breakable]
\ProtSuff*
\end{tcolorbox}

\begin{proof}
Fix $\eta\in(0,\bar\eta)$, and let
$$
(p_K^\star,w_K^\star)
:=
(p_K^\star(\mu,\eta),\,w_K^\star(\mu,\eta))
$$
denote the unique equilibrium from Theorem~\ref{thm:global-reinj}.

For notational simplicity, write
$$
p^\star:=p_K^\star(\mu,\eta).
$$

At equilibrium, the data-update equation gives
$$
p^\star(x)
=
(1-\eta)\mu(x)
+
\eta\,p^\star(x)H^K_{(p^\star)_{w_K^\star}}(x).
$$
Define the curated equilibrium distribution
$$
q^\star(x)
:=
p^\star(x)H^K_{(p^\star)_{w_K^\star}}(x)
$$
By Proposition~\ref{prop:pure-basic}(ii), $q^\star$ is a probability distribution. Hence the equilibrium identity can be written as
$$
p^\star
=
(1-\eta)\mu+\eta q^\star.
$$

Applying \(\bar\varphi(\cdot)\) to both sides and using linearity,
$$
\bar\varphi(p^\star)
=
(1-\eta)\bar\varphi(\mu)+\eta\bar\varphi(q^\star).
$$
Fix any protected direction \(v_\ell\). Taking inner products with \(v_\ell\), we obtain
$$
\bigl\langle v_\ell,\bar\varphi(p^\star)\bigr\rangle
=
(1-\eta)\bigl\langle v_\ell,\bar\varphi(\mu)\bigr\rangle
+
\eta\bigl\langle v_\ell,\bar\varphi(q^\star)\bigr\rangle.
$$

Now
$$
\bigl\langle v_\ell,\bar\varphi(\mu)\bigr\rangle
=
\sum_{i=1}^n \mu_i\langle v_\ell,\varphi_i\rangle.
$$

Define
$$
m_\ell:=\min_{i=1,\dots,n}\langle v_\ell,\varphi_i\rangle.
$$
Since \(q^\star\in\mathcal P(\mathcal X)\), we have
$$
\bigl\langle v_\ell,\bar\varphi(q^\star)\bigr\rangle
=
\sum_{i=1}^n q^\star(x_i)\langle v_\ell,\varphi_i\rangle
\ge
\sum_{i=1}^n q^\star(x_i)\,m_\ell
=
m_\ell.
$$
Therefore,
$$
\bigl\langle v_\ell,\bar\varphi(p^\star)\bigr\rangle
\ge
(1-\eta)\sum_{i=1}^n \mu_i\langle v_\ell,\varphi_i\rangle
+
\eta m_\ell.
$$

Now assume
$$
\sum_{i=1}^n \mu_i\langle v_\ell,\varphi_i\rangle
\ge
\frac{\theta_\ell-\eta m_\ell}{1-\eta}.
$$
Then
$$
\bigl\langle v_\ell,\bar\varphi(p^\star)\bigr\rangle
\ge
(1-\eta)\cdot \frac{\theta_\ell-\eta m_\ell}{1-\eta}+\eta m_\ell
=
\theta_\ell.
$$
Since this holds for every \(\ell=1,\dots,L\), the equilibrium satisfies
$$
\bigl\langle v_\ell,\bar\varphi(p_K^\star(\mu,\eta))\bigr\rangle
\ge
\theta_\ell,
\qquad
\ell=1,\dots,L.
$$
This proves the claim.
\end{proof}
\section{Extension to Continuous Instance Spaces}
\label{sec:continuous}
In the main text, we present the model and results in the discrete setting with finite $\mathcal{X} = \{x_1, \dots, x_n\}$. This appendix extends every result to the case where $\mathcal{X} \subset \mathbb{R}^d$ is \emph{compact} instance space.

We denote by $\mathcal{P}(\mathcal{X})$ the set of Borel probability measures on $\mathcal{X}$. The reward $r(x; w) = \langle w, \varphi(x) \rangle$ with $\|\varphi(x)\| = 1$ for every $x \in \mathcal{X}$, feature map, preference state $w \in \mathcal{W} := \{w \in \mathbb{R}^d : \|w\| \leq 1\}$, and mean feature operator $\bar{\varphi}(p) := \int_{\mathcal{X}} \varphi(x)\, p(dx)$ carry over unchanged.
The coupled dynamics are
\begin{equation}
    p_{t+1} = (1 - \eta) p_{\mathrm{ref}} + \eta p_t\, H^K_{p_t, w_t}
    \qquad
    w_{t+1} = \mathrm{Proj}_{\mathcal{W}}\!\bigl((1 - \beta)\, w_t + \beta\, \bar{\varphi}(p_{t+1})\bigr)
\label{eq:cont_dyn}
\end{equation}
where $p_{\mathrm{ref}} \in \mathcal{P}(\mathcal{X})$ is a fixed reference measure, $\eta \in [0,1]$, and the curation operator $H^K_{p,w}$ is defined as in Def.~\ref{def:curation-operator}:
$$H^K_{p,w}(x) = \mathbb{E}_{X_1,\dots,X_{K-1} \overset{\mathrm{iid}}{\sim} p}
\!\Bigl[\frac{K\, e^{\tau r_w(x)}}{e^{\tau r_w(x)} + \sum_{i=1}^{K-1} e^{\tau r_w(X_i)}}\Bigr].
$$
All results below hold for every $K \geq 2$ as well as in the limit $K \to \infty$.

\subsection{Retraining with Purely Synthetic User-Curated Data}
 
This appendix extends the results to a compact instance space $X \subset \mathbb{R}^m$. Throughout Appendix C, $\mathcal P(X)$ denotes
the set of Borel probability measures on $X$. We assume that the feature map $\phi:X\to \mathbb{R}^d$ is continuous and satisfies $\|\phi(x)\|=1$ for all $x\in X$. Hence $r_w(x)=\langle w,\phi(x)\rangle$ and $a_w(x):=\exp(\tau r_w(x))$ are continuous and satisfy $e^{-\tau}\le a_w(x)\le e^\tau$.

\begin{proposition}[Basic properties of the curation operator on compact spaces]
\label{prop:cont-curation} For any $p \in \mathcal{P}(\mathcal{X})$,
$w \in \mathcal{W}$, and $K \geq 2$, the curation operator satisfies:
\begin{enumerate}[leftmargin=*,topsep=-0.1cm,itemsep=-0.05cm]
  \item \textbf{Strict monotonicity.}
    For any $x, y \in \mathcal{X}$, $r_w(x) > r_w(y) \;\Longleftrightarrow\; H^K_{p,w}(x) > H^K_{p,w}(y)$.
  \item \textbf{Normalization.}$\int_{\mathcal{X}} H^K_{p,w}(x)\, p(dx) = 1$.
  \item \textbf{Odds contraction.}
    If $r_w(x) - r_w(y) \geq r_\Delta > 0$, then
    $
      \frac{H^K_{p,w}(y)}{H^K_{p,w}(x)}
      \leq \lambda_K(r_\Delta)
      := 1 - \frac{(1 - e^{-\tau r_\Delta})(K-1)}{e^{2\tau} + K - 1}
      \in (0,1),
    $
    and $\lambda_K(r_\Delta) \xrightarrow{K \to \infty} e^{-\tau r_\Delta}$.
\end{enumerate}
\end{proposition}
 
\begin{proof}
Each property is proved by arguments that are entirely pointwise and do not rely on finiteness of $\mathcal{X}$; they use only the bound $e^{-\tau} \leq e^{\tau r_w(x)} \leq e^{\tau}$.
 
\textbf{Strict monotonicity.}
For every fixed realization $S := \sum_{\ell=1}^{K-1} e^{\tau r_w(X_\ell)} > 0$, the map $a \mapsto Ka/(a + S)$ is strictly increasing on $(0, \infty)$, and $a(x) := e^{\tau r_w(x)}$ is strictly increasing in $r_w(x)$.
Hence, for every sample path,
$$
  r_w(x) > r_w(y)
  \;\Longrightarrow\;
  \frac{Ka(x)}{a(x) + S} > \frac{Ka(y)}{a(y) + S}.
$$
Taking expectations preserves the strict inequality (both sides are continuous functions of the sample path, and the inequality is strict on a set of full measure).

\textbf{Normalization.}
Let $X_1, \dots, X_K \overset{\mathrm{iid}}{\sim} p$. By Fubini's theorem,
$$
  \int_{\mathcal{X}} H^K_{p,w}(x)\, p(dx)
  = \mathbb{E}\!\left[
    \sum_{j=1}^{K} \frac{e^{\tau r_w(X_j)}}{\sum_{i=1}^{K} e^{\tau r_w(X_i)}}
  \right]
  = \mathbb{E}[1] = 1.
$$
The first equality uses the symmetry of the $K$ i.i.d.\ draws:
each of the $K$ positions contributes identically, so integrating $H^K_{p,w}(x)$
against $p(dx)$ (which accounts for one position) and multiplying by the
factor $K$ inside $H^K_{p,w}$ reconstructs the full sum over all $K$ positions.

The second equality holds because the Luce probabilities sum to~$1$
for every realization $(X_1, \dots, X_K)$.
 
\textbf{Odds contraction.}
Write $a(x) = e^{\tau r_w(x)}$, $a(y) = e^{\tau r_w(y)}$, $S = \sum_{\ell=1}^{K-1} a(X_\ell)$.

Since $r_w(x) - r_w(y) \geq r_\Delta$, we have $a(x) - a(y) \geq (1 - e^{-\tau r_\Delta})\, a(x)$.
For any fixed $S \geq 0$,
$$
  \frac{a(y)/(a(y) + S)}{a(x)/(a(x) + S)}
  = 1 - \frac{S\,(a(x) - a(y))}{a(x)\,(a(y) + S)}.
$$
Using $a(x) - a(y) \geq (1 - e^{-\tau r_\Delta})\, a(x)$,
$a(y) \leq e^{\tau}$, and $S \geq (K-1)e^{-\tau}$,
the ratio $S/(e^{\tau} + S)$ is bounded below by $(K-1)/(e^{2\tau} + K - 1)$.
Taking expectations yields the stated bound.
\end{proof}

\begin{theorem}[Equilibrium characterization on compact spaces]
\label{thm:cont-eq-char}
When $\eta = 1$, a distribution $p \in \mathcal{P}(\mathcal{X})$ is a fixed point of
the retraining step under preference $w$ if and only if the reward $r_w$
is constant on $\mathrm{supp}(p)$:
$$
  x, x' \in \mathrm{supp}(p)
  \quad\Longrightarrow\quad
  r_w(x) = r_w(x').
$$
Consequently, for every $x^\star \in \mathcal{X}$, the pair
$(p^\star = \delta_{x^\star},\; w^\star = \varphi(x^\star))$ is an equilibrium.
\end{theorem}
 
\begin{proof}
\textbf{Fixed-point characterization.}
Let $\mu_p^K(dx):=H_{p,w}^K(x)p(dx)$. The fixed-point equation is
$p=\mu_p^K$, equivalently
\[
\int_A (H_{p,w}^K(x)-1)p(dx)=0,\qquad \forall A\in\mathcal B(X).
\]
Thus $H_{p,w}^K(x)=1$ for $p$-a.e. $x$. Since $\phi$ is continuous,
$H_{p,w}^K$ is continuous in $x$ by dominated convergence. Therefore
$H_{p,w}^K(x)=1$ on $\operatorname{supp}(p)$.

If $x,y\in\operatorname{supp}(p)$ and $r_w(x)\ne r_w(y)$, strict
monotonicity from Proposition C.1 gives
$H_{p,w}^K(x)\ne H_{p,w}^K(y)$, contradicting the fact that both equal
$1$. Hence $r_w$ is constant on $\operatorname{supp}(p)$.
Conversely, if $r_w$ is constant on $\operatorname{supp}(p)$, then
$p$-a.s. all sampled points have the same reward as every
$x\in\operatorname{supp}(p)$, so directly
$H_{p,w}^K(x)=1$ on $\operatorname{supp}(p)$ and hence $pH_{p,w}^K=p$.

Conversely, if all support points share a common reward~$c$,
then $e^{\tau r_w(x)} = e^{\tau c}$ for every $x \in \mathrm{supp}(p)$,
and by direct computation $H^K_{p,w}(x) = 1$ for every such $x$.
 
\textbf{Singleton equilibria.}
For any $x^\star \in \mathcal{X}$, $\delta_{x^\star}$ is trivially supported on
a single point, so the equal-reward condition holds.
Moreover,
$\bar{\varphi}(\delta_{x^\star}) = \varphi(x^\star)$,
and
$\mathrm{Proj}_{\mathcal{W}}\!\bigl((1-\beta)\,\varphi(x^\star) + \beta\,\varphi(x^\star)\bigr) = \varphi(x^\star)$
since $\|\varphi(x^\star)\| = 1 \leq 1$.
Hence $(\delta_{x^\star}, \varphi(x^\star))$ is an equilibrium.
\end{proof}

\begin{assumption}[Multiscale reward gap]
\label{asp:multiscale}
Fix $x^\star \in \mathcal{X}$.
There exist a radius $r_0 > 0$ and sequences $(a_k)_{k \geq 0}$, $(b_k)_{k \geq 0}$, $(\gamma_k)_{k \geq 0}$ with $a_k \downarrow 0$, $b_k \downarrow 0$, $\gamma_k > 0$, such that the nested balls $U_k := B_{2^{-k} r_0}(x^\star)$,  $k = 0, 1, 2, \dots$ (a sequence of shrinking neighborhoods of $x^\star$  with radii halving at each scale) satisfy, for every $k \geq 0$:
\begin{enumerate}[leftmargin=*, topsep=-0.1cm, itemsep=-0.05cm]
  \item \textbf{Feature continuity:}
    $\displaystyle \sup_{x \in U_k} \|\varphi(x) - \varphi(x^\star)\| \leq a_k$.
  \item \textbf{Reward gap:}
    $\|w - \varphi(x^\star)\| \leq b_k \;\Longrightarrow\;
      \displaystyle\inf_{y \in U_{k+1}} \langle w, \varphi(y) \rangle
      - \sup_{x \notin U_k} \langle w, \varphi(x) \rangle
      \geq \gamma_k$.
  \item \textbf{Scale compatibility:}
    $a_k < b_{k+1}$ for all $k \geq 0$, and $a_0 < b_0$.
\end{enumerate}
\end{assumption}
 
\begin{remark}[Interpretation]
\label{rem:multiscale}
Assumption~\ref{asp:multiscale} is the continuous analogue of the
feature-separation condition $\varphi(x_j) \neq \varphi(x_i)$ for all $j \neq i$
in Theorem~3.3.
In the discrete case, the finite set of competitors admits a single uniform
reward gap $\gamma_i = \min_{j \neq i}\langle \varphi(x_i), \varphi(x_i) - \varphi(x_j)\rangle > 0$,
and the multiscale conditions are automatically satisfied with
$a_k = 0$, $b_k = \gamma_i / 4$, $\gamma_k = \gamma_i / 2$.
In the continuous case, competitors can be arbitrarily close to $x^\star$
with arbitrarily small reward disadvantage, so the conditions must hold
at every scale separately.
Specifically:
(i)~controls how well the feature mean approximates $\varphi(x^\star)$
after mass concentrates at scale $k$;
(ii)~provides the reward advantage that drives mass from $U_k^c$ into $U_{k+1}$;
(iii)~ensures that the feature accuracy gained at scale $k$ is sufficient
to activate the reward gap at scale $k+1$.
These conditions are not implied by ordinary feature separation alone.
They hold, for example, whenever the chosen neighborhoods $(U_k)$ satisfy
a uniform annular reward separation condition
\[
\inf_{y\in U_{k+1}}\langle \phi(x^\star),\phi(y)\rangle
>
\sup_{x\notin U_k}\langle \phi(x^\star),\phi(x)\rangle
\]
with margins large enough to dominate perturbations
$\|w-\phi(x^\star)\| \le b_k$.
\end{remark}

\begin{theorem}[Local convergence to singleton equilibria on compact spaces]
\label{thm:cont-lockin}
Let $x^\star \in \mathcal{X}$ satisfy Assumption~\ref{asp:multiscale}. Then there exists $\delta^\star > 0$ such that if
$$
  x^\star \in \mathrm{supp}(p_0), \qquad
  \|w_0 - \varphi(x^\star)\| \leq \delta^\star, \qquad
  p_0(U_1) \geq 1 - \delta^\star,
$$
then $p_t(B_\varepsilon(x^\star)) \to 1$ for every $\varepsilon > 0$ and $w_t \to \varphi(x^\star)$.
\end{theorem}

\begin{proof}
Since $a_0 < b_0$, choose $\delta^\star > 0$ with
\begin{equation}\label{eq:delta-star}
  0 < \delta^\star \leq \min\!\bigl\{\tfrac{1}{2},\; b_0\bigr\}, \qquad
  a_0 + \frac{2\delta^\star}{1 - \delta^\star} \leq b_0
\end{equation}
Define
$q_{k,t} := p_t(U_k^c)$ (the mass outside scale $k$) and
$s_{k,t} := p_t(U_{k+1})$ (the mass inside scale $k+1$).

Assume $x^\star \in \mathrm{supp}(p_0)$, $\|w_0 - \varphi(x^\star)\| \leq \delta^\star$, and $p_0(U_1) \geq 1 - \delta^\star$.
 
\textbf{Step~1: The trajectory stays inside the initial preference neighborhood.}

We prove $\|w_t - \varphi(x^\star)\| \leq b_0$ for all $t$ by induction.

The base case is $\|w_0 - \varphi(x^\star)\| \leq \delta^\star \leq b_0$.

Suppose $\|w_t - \varphi(x^\star)\| \leq b_0$. Then reward gap with $k = 0$ gives
$$
\inf_{y \in U_1} r(y; w_t) - \sup_{x \notin U_0} r(x; w_t) \geq \gamma_0.
$$

When $\eta = 1$, the retraining step gives $p_{t+1}(A) = \int_A H^K_{p_t,w_t}(x)\, p_t(dx)$.

For every $y \in U_0^c$ and $x \in U_1$,it yields $r_{w_t}(x) - r_{w_t}(y) \geq \gamma_0$,
so Proposition~\ref{prop:cont-curation} (odds contraction) gives
$$
  H^K_{p_t,w_t}(y) \leq \lambda_K(\gamma_0)\; H^K_{p_t,w_t}(x).
$$
Since this holds for \emph{every} $x \in U_1$,
$H^K_{p_t,w_t}(y) \leq \lambda_K(\gamma_0)\, \inf_{x \in U_1} H^K_{p_t,w_t}(x)$
for all $y \in U_0^c$.

Therefore
$$
 q_{0,t+1} = \int_{U_0^c} H^K_{p_t,w_t}(y)\, p_t(dy)
  \leq \lambda_K(\gamma_0)\, \inf_{x \in U_1} H^K_{p_t,w_t}(x) \cdot q_{0,t}
$$
and
$$
  s_{0,t+1} = \int_{U_1} H^K_{p_t,w_t}(x)\, p_t(dx)
  \geq \inf_{x \in U_1} H^K_{p_t,w_t}(x) \cdot s_{0,t},
$$
whence
\begin{equation}\label{eq:ratio0}
  \frac{q_{0,t+1}}{s_{0,t+1}} \leq \lambda_K(\gamma_0)\;\frac{q_{0,t}}{s_{0,t}}.
\end{equation}
Iterating,
$$\frac{q_{0,t}}{s_{0,t}} \leq \lambda_K(\gamma_0)^t \cdot \frac{\delta^\star}{(1 - \delta^\star)}$$
and in particular $q_{0,t} \leq \frac{\delta^\star}{(1 - \delta^\star)}$ for all $t$.
 
By feature continuity and $\|\varphi(x)\| = 1$,
$$
  \|\bar{\varphi}(p_{t+1}) - \varphi^\star\|
  \leq a_0 + 2\, q_{0,t+1}.
$$
Since projection is nonexpansive and $\varphi^\star \in \mathcal{W}$,
$$
  \|w_{t+1} - \varphi(x^\star)\| 
  \leq (1-\beta)\, \|w_t - \varphi(x^\star)\| + \beta\,(a_0 + 2q_{0,t+1})
  \leq (1-\beta)\, b_0 + \beta\Bigl(a_0 + \frac{2\delta^\star}{1 - \delta^\star}\Bigr)
  \leq b_0.
$$
 
\textbf{Step~2: First-scale outer leakage vanishes.}

Since $\|w_t - \varphi(x^\star)\|  \leq b_0$ for all $t$, inequality~\eqref{eq:ratio0} holds for all $t$, so $\frac{q_{0,t}}{s_{0,t}} \to 0$ geometrically.

Since $q_{0,t} \leq \frac{q_{0,t}}{s_{0,t}}$, we get $q_{0,t} \to 0$.

\noindent\textbf{Step~3: Scale-by-scale bootstrap.}

We prove by induction: for every $k \geq 0$ there exists $T_k < \infty$ with
$\|w_t - \varphi(x^\star)\|  \leq b_k$ for all $t \geq T_k$.

The case $k = 0$ is Step~1. Suppose the claim holds for~$k$.

\emph{Support preservation.}
Since $H^K_{p,w}(x) > 0$ for all $x$ (as $e^{\tau r_w(x)} > 0$ and the denominator
in Definition~2.1 is finite and positive), the pure update preserves the support:
$\mathrm{supp}(p_t) = \mathrm{supp}(p_0)$ for all $t$.
Because $x^\star \in \mathrm{supp}(p_0)$ and $U_{k+1}$ is open,
$s_{k,t} = p_t(U_{k+1}) > 0$ for all $t$.
 
\emph{Mass concentration.}
For $t \geq T_k$, reward gap gives at scale $k$.

Exactly as in Step~1, Proposition~\ref{prop:cont-curation} (odds contraction) yields
$$
  \frac{q_{k,t+1}}{s_{k,t+1}} \leq \lambda_K(\gamma_k)\;\frac{q_{k,t}}{s_{k,t}},
$$
and iterating gives $q_{k,t} \to 0$.
 
\emph{Preference bootstrap.}
$\|\bar{\varphi}(p_t) - \varphi^\star\| \leq a_k + 2q_{k,t}$, so
$\|w_{t+1} - \varphi(x^\star)\|  \leq (1-\beta)\, \|w_t - \varphi(x^\star)\|  + \beta\,(a_k + 2q_{k,t+1})$.

Since $q_{k,t} \to 0$, $\limsup_{t} \|w_t - \varphi(x^\star)\|  \leq a_k < b_{k+1}$,
giving $T_{k+1}$ with $\|w_t - \varphi(x^\star)\|  \leq b_{k+1}$ for $t \geq T_{k+1}$.

\textbf{Step~4: Convergence.}
For every $k$, $p_t(U_k^c) \to 0$.

Given $\varepsilon > 0$, choose $k$ with $r_k < \varepsilon$; then$p_t(B_\varepsilon(x^\star)^c) \leq p_t(U_k^c) \to 0$.

The induction gives $\limsup_t d_t \leq a_k \to 0$, so $w_t \to \varphi^\star$.
\end{proof}

\begin{corollary}[Failure of local stability on compact spaces]
\label{cor:cont-failure}
Fix $x^\star \in \mathcal{X}$.
If there exists $x' \neq x^\star$ with $\varphi(x') = \varphi(x^\star)$,
then the singleton equilibrium $(\delta_{x^\star}, \varphi(x^\star))$ is not
locally asymptotically stable.
\end{corollary}
 
\begin{proof}
For any $\varepsilon \in (0,1)$, define $p_\varepsilon := (1-\varepsilon)\,\delta_{x^\star} + \varepsilon\,\delta_{x'}$ and $w_\varepsilon := \varphi(x^\star)$.

Since $\varphi(x') = \varphi(x^\star)$, the two support points share the same reward under $w_\varepsilon$, so $H^K_{p_\varepsilon, w_\varepsilon}(x^\star) = H^K_{p_\varepsilon, w_\varepsilon}(x') = 1$ by the equal-reward characterization in Theorem~\ref{thm:cont-eq-char}.

Moreover, $\bar{\varphi}(p_\varepsilon) = \varphi(x^\star) = w_\varepsilon$.

Hence $(p_\varepsilon, w_\varepsilon)$ is an equilibrium for every $\varepsilon > 0$.

Since $(p_\varepsilon, w_\varepsilon) \to (\delta_{x^\star}, \varphi(x^\star))$as $\varepsilon \to 0$, the singleton equilibrium is not isolated and therefore not locally asymptotically stable.
\end{proof}

\subsection{Retraining on Mixed Synthetic and Reference Data}
\label{app:cont-reinjection}
 
We now consider $\eta \in (0,1)$, paralleling Section~3.2.
 
\begin{lemma}[Sensitivity of the retraining step on compact spaces]
\label{lem:cont-sensitivity}
For $p\in\mathcal P(\mathcal X)$, define the probability measure
$pH^K_{p,w}$ by
\[
  (pH^K_{p,w})(A)
  :=
  \int_A H^K_{p,w}(x)\,p(dx),
  \qquad A\in\mathcal B(\mathcal X).
\]
Then for any $p,q\in\mathcal P(\mathcal X)$, $w,v\in\mathcal W$, and
$K\ge2$,
\[
  d_{\rm TV}\bigl(pH^K_{p,w},qH^K_{q,w}\bigr)
  \le L_p d_{\rm TV}(p,q),
\]
\[
  d_{\rm TV}\bigl(pH^K_{p,w},pH^K_{p,v}\bigr)
  \le L_w \|w-v\|,
\]
and
\[
  \|\bar\phi(p)-\bar\phi(q)\|
  \le 2d_{\rm TV}(p,q),
\]
where $L_p:=2e^{2\tau}(2+e^{2\tau})$ and
$L_w:=2\tau e^{4\tau}$.
\end{lemma}
 
\begin{proof}
The proof of Lemma~3.6 (Sec.~B.6) relies on three ingredients:
\begin{enumerate}[leftmargin=*, topsep=-0.1cm, itemsep=-0.05cm]
\item the uniform bounds $e^{-\tau} \leq e^{\tau r_w(x)} \leq e^{\tau}$;
\item the variational formula $\bigl|\int f d(p - q)\bigr| \leq 2\|f\|_\infty\, d_{\mathrm{TV}}(p, q)$,
valid on any measurable space;
\item the coordinate-swapping telescoping argument bounding
$|H^K_{p,w}(x) - H^K_{q,w}(x)|$ (Part~1)
and $|H^K_{p,w}(x) - H^K_{p,v}(x)|$ (Part~2).
\end{enumerate}
None of these steps use finiteness of $\mathcal{X}$:
the sums in the discrete proof become integrals, and all bounds hold
for arbitrary Borel probability measures.
 
\textbf{Part~1: $p$-Lipschitz bound.}

For any measurable $A$,
\[
(pH^K_{p,w}-qH^K_{q,w})(A)
=
\int_A H^K_{p,w}\,d(p-q)
+
\int_A (H^K_{p,w}-H^K_{q,w})\,dq.
\]
Taking the supremum over measurable $A$ gives
\[
2d_{\rm TV}(pH^K_{p,w},qH^K_{q,w})
\le
2\|H^K_{p,w}\|_\infty d_{\rm TV}(p,q)
+
\|H^K_{p,w}-H^K_{q,w}\|_\infty.
\]

The first term contributes at most $4e^{2\tau}\, d_{\mathrm{TV}}(p,q)$ since $0 \leq H^K_{p,w}(x) \leq 2e^{2\tau}$.
The second term is bounded via the coordinate-swapping argument: the effect of replacing each of the $K-1$ sample coordinates from $p$ to $q$ contributes at most $2e^{4\tau}/K$ in the sup-norm of the integrand, yielding a total of at most $2e^{4\tau}\, d_{\mathrm{TV}}(p,q)$.

Combining: $d_{\mathrm{TV}}(pH^K_{p,w}, qH^K_{q,w}) \leq (2e^{2\tau} + e^{4\tau})\, d_{\mathrm{TV}}(p,q) = L_p\, d_{\mathrm{TV}}(p,q)$.
 
\textbf{Part~2: $w$-Lipschitz bound.}

Using $|e^{\tau r_w(x)} - e^{\tau r_v(x)}| \leq \tau e^{\tau}\, \|w - v\|$ and the two-term decomposition of $H^K_{p,w}(x) - H^K_{p,v}(x)$ (see Sec.~B.6, Part~2), each term is bounded by $\tau e^{4\tau}\, \|w - v\|$, giving $d_{\mathrm{TV}}(pH^K_{p,w}, pH^K_{p,v}) \leq 2\tau e^{4\tau}\, \|w - v\| = L_w\, \|w - v\|$.
 
\textbf{Part~3: Mean-feature bound.}
For any unit vector $u$, $|\langle u, \varphi(x)\rangle| \leq 1$, so$|\langle u, \bar{\varphi}(p) - \bar{\varphi}(q)\rangle| \leq 2\, d_{\mathrm{TV}}(p,q)$.

Taking the supremum over $u$ gives $\|\bar{\varphi}(p) - \bar{\varphi}(q)\| \leq 2\, d_{\mathrm{TV}}(p,q)$.
\end{proof}
 
\begin{theorem}[Unique equilibrium under frozen-$w$ on compact spaces]
\label{thm:cont-frozen}
Fix $w\in\mathcal W$ and $K\ge2$ or $K=\infty$. Define
$T^K_w:\mathcal P(\mathcal X)\to\mathcal P(\mathcal X)$ by
\[
  T^K_w(p)(A)
  :=
  (1-\eta)p_{\mathrm{ref}}(A)
  +
  \eta\int_A H^K_{p,w}(x)\,p(dx),
  \qquad A\in\mathcal B(\mathcal X).
\]
If $\eta L_p<1$, then $T^K_w$ is a strict contraction on
$(\mathcal P(\mathcal X),d_{\rm TV})$ with rate $\eta L_p$.
Consequently, there exists a unique equilibrium $p^{K,\eta}_w$ and every
trajectory converges geometrically:
\[
  d_{\rm TV}(p_t,p^{K,\eta}_w)
  \le
  (\eta L_p)^t d_{\rm TV}(p_0,p^{K,\eta}_w).
\]
\end{theorem}
 
\begin{proof}
For any $p,q\in\mathcal P(\mathcal X)$ and measurable $A$,
\[
  T^K_w(p)(A)-T^K_w(q)(A)
  =
  \eta\bigl[(pH^K_{p,w})(A)-(qH^K_{q,w})(A)\bigr].
\]
Therefore, by Lemma~\ref{lem:cont-sensitivity},
\[
  d_{\rm TV}(T^K_w(p),T^K_w(q))
  =
  \eta d_{\rm TV}(pH^K_{p,w},qH^K_{q,w})
  \le
  \eta L_p d_{\rm TV}(p,q).
\]
If $\eta L_p<1$, this is a strict contraction. Since
$(\mathcal P(\mathcal X),d_{\rm TV})$ is complete, Banach's fixed-point
theorem gives a unique fixed point $p^{K,\eta}_w$ and geometric
convergence.
\end{proof}

\begin{proposition}[Closed-form equilibrium on compact spaces]
\label{prop:cont-closedform}
Assume the condition of Theorem~\ref{thm:cont-frozen}, in particular
$\eta L_p<1$. For the limiting curation operator $K=\infty$, define
\[
  a_w(x) := e^{\tau r_w(x)}, \qquad
  M_w := \sup_{x\in\mathcal X} a_w(x).
\]
Then the frozen-$w$ equilibrium is given, for every measurable
$A\subseteq\mathcal X$, by
\[
  p^{\infty,\eta}_w(A)
  =
  \int_A
  \frac{(1-\eta)\, Z^\eta_w}
       {Z^\eta_w-\eta a_w(x)}
  \;p_{\mathrm{ref}}(dx),
\]
where $Z^\eta_w>\eta M_w$ is the unique scalar solving
\[
  \int_{\mathcal X}
  \frac{(1-\eta)\, a_w(x)}
       {Z^\eta_w-\eta a_w(x)}
  \;p_{\mathrm{ref}}(dx)
  =1.
\]
\end{proposition}

\begin{proof}
For $K=\infty$, the curation operator is
\[
  H^\infty_{p,w}(x)
  =
  \frac{a_w(x)}
       {\int_{\mathcal X} a_w(z)\,p(dz)}.
\]
Let $p$ be a frozen-$w$ fixed point and set
\[
  Z := \int_{\mathcal X} a_w(z)\,p(dz).
\]
Since $e^{-\tau}\le a_w(x)\le e^\tau$, we have
$Z\ge e^{-\tau}$. Moreover, $\eta L_p<1$ implies
$\eta<e^{-2\tau}$, and hence
\[
  \eta M_w \le \eta e^\tau < e^{-\tau}\le Z.
\]
Therefore $Z-\eta a_w(x)>0$ for every $x\in\mathcal X$.

The fixed-point equation
\[
  p=(1-\eta)p_{\mathrm{ref}}+\eta pH^\infty_{p,w}
\]
is equivalently
\[
  p(dx)\left(1-\frac{\eta a_w(x)}{Z}\right)
  =
  (1-\eta)p_{\mathrm{ref}}(dx).
\]
Thus
\[
  p(dx)
  =
  \frac{(1-\eta)Z}{Z-\eta a_w(x)}
  \;p_{\mathrm{ref}}(dx).
\]
Substituting this expression into the definition of $Z$ and dividing by
$Z>0$ gives the scalar equation
\[
  \int_{\mathcal X}
  \frac{(1-\eta)a_w(x)}
       {Z-\eta a_w(x)}
  \;p_{\mathrm{ref}}(dx)
  =1.
\]

It remains to show that this equation has a unique solution on
$(\eta M_w,\infty)$. Define
\[
  F(Z)
  :=
  \int_{\mathcal X}
  \frac{(1-\eta)a_w(x)}
       {Z-\eta a_w(x)}
  \;p_{\mathrm{ref}}(dx),
  \qquad Z>\eta M_w.
\]
The function $F$ is continuous and strictly decreasing on
$(\eta M_w,\infty)$, and $F(Z)\to0$ as $Z\to\infty$.

We now check that $F$ is larger than $1$ near the left endpoint. Since
$a_w(x)\in[e^{-\tau},M_w]$, the map
$a\mapsto a/(Z-\eta a)$ is increasing for $Z>\eta M_w$. Hence
\[
  F(Z)
  \ge
  \frac{(1-\eta)e^{-\tau}}
       {Z-\eta e^{-\tau}}.
\]
Taking $Z\downarrow \eta M_w$ gives
\[
  \liminf_{Z\downarrow \eta M_w}F(Z)
  \ge
  \frac{(1-\eta)e^{-\tau}}
       {\eta(M_w-e^{-\tau})},
\]
with the right-hand side interpreted as $+\infty$ if
$M_w=e^{-\tau}$. Since $M_w\le e^\tau$ and
$\eta<e^{-2\tau}$, we have $\eta M_w<e^{-\tau}$, which implies
\[
  \frac{(1-\eta)e^{-\tau}}
       {\eta(M_w-e^{-\tau})}
  >1
\]
whenever $M_w>e^{-\tau}$. Therefore $F(Z)>1$ for $Z$ sufficiently close
to $\eta M_w$, while $F(Z)\to0$ as $Z\to\infty$. By the intermediate
value theorem, there exists a solution $Z^\eta_w>\eta M_w$, and strict
monotonicity gives uniqueness.

Finally, for this unique $Z^\eta_w$, define $p^{\infty,\eta}_w$ by the
displayed formula. The scalar equation implies
\[
  \int_{\mathcal X} a_w(x)\,p^{\infty,\eta}_w(dx)=Z^\eta_w,
\]
and also
\[
  p^{\infty,\eta}_w(\mathcal X)
  =
  (1-\eta)
  +
  \eta
  \int_{\mathcal X}
  \frac{(1-\eta)a_w(x)}
       {Z^\eta_w-\eta a_w(x)}
  \;p_{\mathrm{ref}}(dx)
  =1.
\]
Substituting the formula back into the frozen-$w$ update verifies that
$p^{\infty,\eta}_w$ is a fixed point. Uniqueness follows from the
uniqueness of $Z^\eta_w$, equivalently from the contraction in
Theorem~\ref{thm:cont-frozen}.
\end{proof}

\begin{theorem}[Convergence of coupled dynamics on compact spaces]
\label{thm:cont-coupled}
If $\eta(L_p + 2L_w) < 1$, then there exists
$\xi \in \bigl(\frac{\eta L_w}{\beta(1-2\eta L_w)},\;
               \frac{1-\eta L_p}{2\beta\eta L_p}\bigr)$
such that the coupled update is a strict contraction on $(\mathcal{P}(\mathcal{X}) \times \mathcal{W},\; d_\xi)$ with $d_\xi\!\bigl((p,w),(q,v)\bigr) := d_{\mathrm{TV}}(p,q) + \xi\,\|w - v\|$.
Consequently, there exists a unique equilibrium $(p^\star_K, w^\star_K)$
and every trajectory converges geometrically.
\end{theorem}
 
\begin{proof}
Let $(p^+, w^+) := F(p, w)$ and $(q^+, v^+) := F(q, v)$,
where $F$ denotes the joint update.
 
\textbf{Data coordinate.}
By Lemma~\ref{lem:cont-sensitivity},
$$
  d_{\mathrm{TV}}(p^+, q^+)
  \leq \eta L_p\, d_{\mathrm{TV}}(p, q) + \eta L_w\, \|w - v\|.
$$
 
\textbf{Preference coordinate.}
Using nonexpansiveness of projection and
$\|\bar{\varphi}(p^+) - \bar{\varphi}(q^+)\| \leq 2\, d_{\mathrm{TV}}(p^+, q^+)$,
\begin{align*}
  \|w^+ - v^+\|
  &\leq (1-\beta)\, \|w - v\| + 2\beta\, d_{\mathrm{TV}}(p^+, q^+) \\
  &\leq 2\beta\eta L_p\, d_{\mathrm{TV}}(p, q)
       + (1 - \beta + 2\beta\eta L_w)\, \|w - v\|.
\end{align*}

\textbf{Combined metric.}
Multiplying the preference bound by $\xi$ and adding,
$$
  d_\xi(F(p,w), F(q,v))
  \leq A(\xi)\, d_{\mathrm{TV}}(p,q) + \xi\, B(\xi)\, \|w - v\|,
$$
where
$A(\xi) := \eta L_p(1 + 2\xi\beta)$
and
$B(\xi) := 1 - \beta + 2\beta\eta L_w + \eta L_w/\xi$.
The map is a contraction when $A(\xi) < 1$ and $B(\xi) < 1$:
\begin{align*}
  A(\xi) < 1 &\;\Longleftrightarrow\;
    \xi < \frac{1 - \eta L_p}{2\beta\eta L_p}, \\
  B(\xi) < 1 &\;\Longleftrightarrow\;
    \xi > \frac{\eta L_w}{\beta(1 - 2\eta L_w)}.
\end{align*}
Both requirements are implied by $\eta(L_p + 2L_w) < 1$, and under this condition
the window is nonempty:
$$
  \frac{\eta L_w}{\beta(1 - 2\eta L_w)} < \frac{1 - \eta L_p}{2\beta\eta L_p}
  \;\Longleftrightarrow\;
  \eta(L_p + 2L_w) < 1.
$$
 
\textbf{Completeness.}
$\mathcal{P}(\mathcal{X})$ is complete under $d_{\mathrm{TV}}$ when $\mathcal{X}$ is compact,
and $\mathcal{W}$ is closed in $\mathbb{R}^d$, so
$(\mathcal{P}(\mathcal{X}) \times \mathcal{W},\, d_\xi)$ is complete.
The Banach fixed-point theorem gives the unique equilibrium and geometric convergence.
\end{proof}

\section{Additional Experimental Results}
\label{sec:addexperiments}
We conduct all the experiments on a server which has two Intel Xeon 6326 CPU and 6 Nvidia A6000 GPU. 
\subsection{Additional CFM Experiments on CIFAR-10}
\label{subsec:cfm}

\subsubsection{Effect of Key Parameters}
\label{subsubsec:cfm_para}
We present the effect of key parameters for OT-CFM on CIFAR-10. The default configuration is $\eta = 1.0$, $\tau = 5.0$, $K = \infty$, $\beta = 0.01$, $w_0 = [1,0,0,0,\dots]$. We vary one parameter at a time.

\textbf{Mixing Ratio $\eta$}. Fig.~\ref{SubFig:cfm_eta} shows the effect of $\eta$. At $\eta = 1.0$ (pure synthetic), airplane reaches $>$95\% by round 20. As $\eta$ decreases, reference data increasingly stabilizes the distribution: at $\eta = 0.8$, the concentration is slightly mitigated; at $\eta = 0.5$, airplane stabilizes around 50\%; at $\eta = 0.2$, it stays below 50\% with other classes remaining above $\theta = 0.08$.The preference vector $w_t$ evolves similarly across all settings because $\beta = 0.01$ is small, confirming that the mode shift is driven by $\eta$
rather than by changes in $w_t$.

\begin{figure}[t]
\centering

\begin{minipage}{\textwidth}
    \centering
    \includegraphics[width=\textwidth]{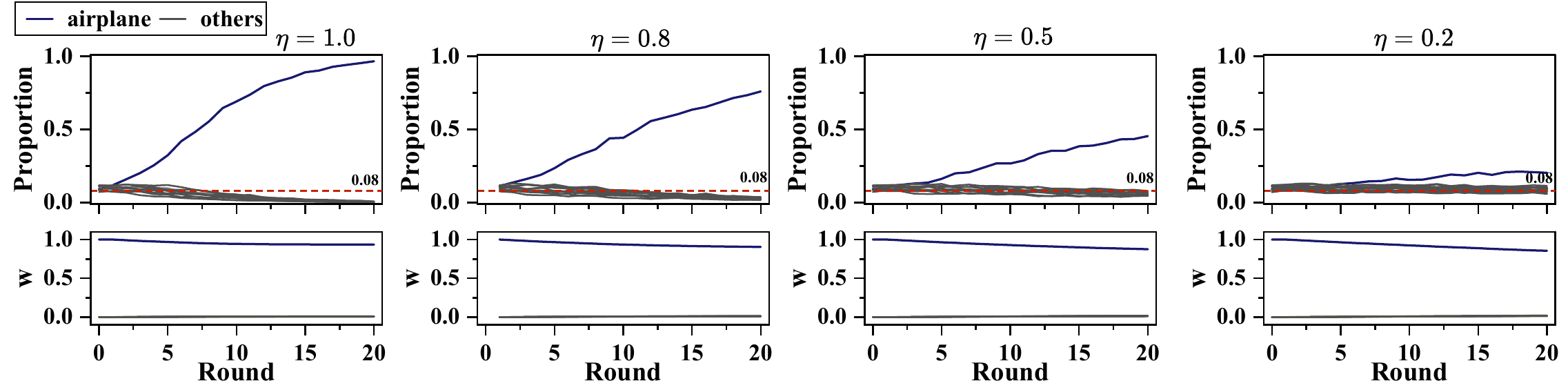}
    \caption{Effect of mixing ratio $\eta$ ($\tau=5.0$, $K=\infty$, $\beta=0.01$).\textit{Top}: per-class proportions. \textit{Bottom}: preference vector $w_t$. Dashed line marks $\theta = 0.08$.}
\label{SubFig:cfm_eta}
\end{minipage}
\vspace{4pt}

\begin{minipage}{\textwidth}
    \centering
    \includegraphics[width=\textwidth]{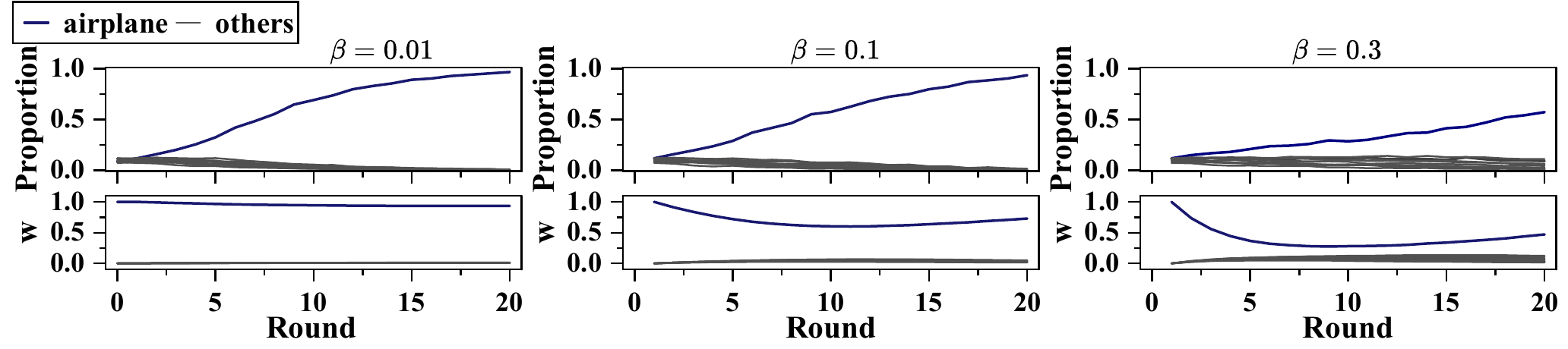}
    \caption{Effect of preference change rate $\beta$ ($\eta=1.0$, $\tau=5.0$, $K=\infty$).}
    \label{SubFig:cfm_beta}
\end{minipage}
\vspace{4pt}

\begin{minipage}{\textwidth}
    \centering
    \includegraphics[width=\textwidth]{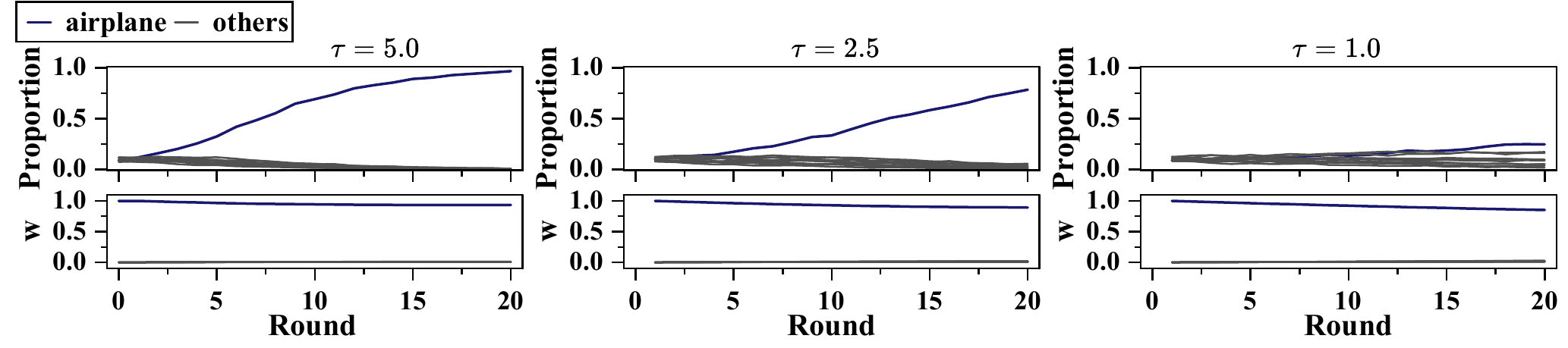}
    \caption{Effect of curation temperature $\tau$ ($\eta=1.0$, $K=\infty$, $\beta=0.01$).}
    \label{SubFig:cfm_tau}
\end{minipage}
\vspace{4pt}

\begin{minipage}{\textwidth}
    \centering
    \includegraphics[width=\textwidth]{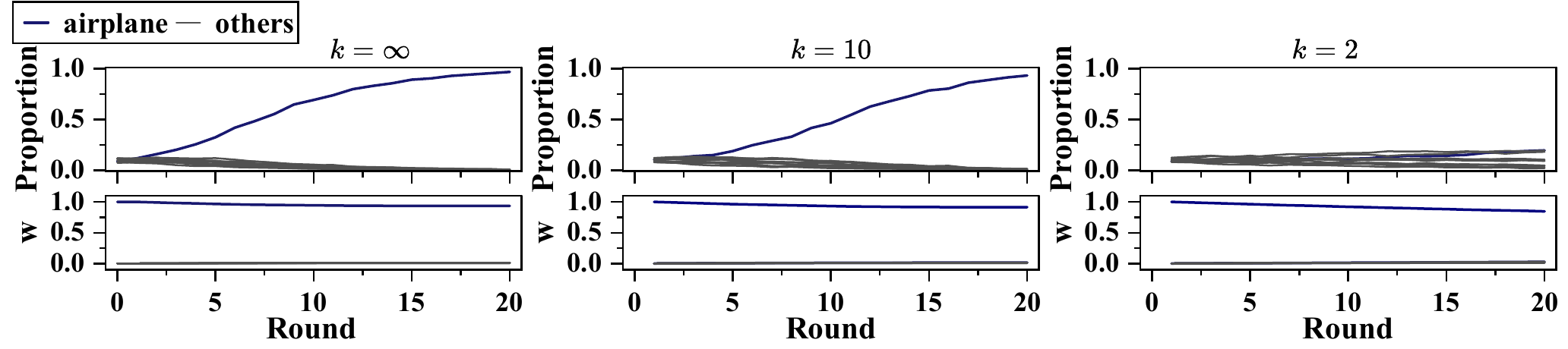}
    \caption{ Effect of selection group size $K$ ($\eta=1.0$, $\tau=5.0$, $\beta=0.01$).}
    \label{SubFig:cfm_k}
\end{minipage}
\end{figure}

\textbf{Preference change rate $\beta$.} Fig.~\ref{SubFig:cfm_beta} shows the effect of $\beta$. At $\beta = 0.01$, $w_t$ changes slowly and the curation direction remains coherent, producing the strongest concentration ($>$95\% airplane). At $\beta = 0.3$, $w_t$ tracks the current distribution too closely, diluting the preference toward a uniform direction and weakening the curation signal. However, as mode concentration progresses, $\bar{\varphi}(p_t)$ itself becomes dominated by the airplane component, which in turn pulls $w_t$ back toward $[1,0,\dots]$. This restores the curation signal and drives further concentration, although it takes effect more slowly compared to small $\beta$.

\textbf{Curation Temperature $\tau$.} Fig.~\ref{SubFig:cfm_tau} shows the effect of $\tau$. At $\tau = 5.0$, curation concentrates sharply: even small reward differences are amplified exponentially, driving airplane above 95\% by round 20. At $\tau = 2.5$, the concentration is slower. At $\tau = 1.0$, the softmax weights are nearly uniform and the class distribution remains close to balanced. The speed of mode concentration increases monotonically with $\tau$, consistent with the Lipschitz constants $L_p$ growing rapidly in $\tau$.

\textbf{Selection Group Size $K$.} Fig.~\ref{SubFig:cfm_k} shows the effect of $K$. At $K = \infty$ and $K = 10$, the concentration trajectories are nearly identical, with airplane exceeding 90\% by round 20. At $K = 2$, each Luce group contains only two candidates, limiting the selector's ability to distinguish high-reward samples.
This is consistent with the discussion of the role of $K$.

\subsubsection{Reference data design}
\label{subsubsec:cfm_pref}
\textbf{Setting A: Uniform thresholds.} As described in Section~\ref{sec:experiments}, all four protected classes share the same threshold $\theta = 0.08$ with  $c = (1, 1, 11, 12, 13, 16, 30, 28, 27, 25)$. The feasibility boundary is at $\eta^* = 0.67$. Since the full table is too long, we present Table~\ref{subtab:alg1-full-A} here to show the partial  LP search results.

\begin{table}[th]
\centering
\caption{Partial results of searching over 50 evenly spaced points in $\eta \in [0.2, 0.9]$. (Setting A: protect frog/horse/ship/truck, $\theta_\ell = 0.08$, $c = (1, 1, 11, 12, 13, 16, 30, 28, 27, 25)$)}
\label{subtab:alg1-full-A}
\small
\resizebox{\linewidth}{!}{
\begin{tabular}{c|cccccccccc|rr|c}
\toprule
$\eta$ & air & auto & bird & cat & deer & dog & frog & horse & ship & truck
  & $C^*$ & $J$ & status \\
\midrule
0.20 & .000 & .600 & .000 & .000 & .000 & .000 & .100 & .100 & .100 & .100 & 11.59 & 46,378 & \cmark \\
0.30 & .000 & .543 & .000 & .000 & .000 & .000 & .114 & .114 & .114 & .114 & 13.11 & 45,883 & \cmark \\
0.40 & .000 & .467 & .000 & .000 & .000 & .000 & .133 & .133 & .133 & .133 & 15.13 & 45,388 & \cmark \\
0.50 & .000 & .360 & .000 & .000 & .000 & .000 & .160 & .160 & .160 & .160 & 17.96 & 44,892 & \cmark \\
0.60 & .000 & .200 & .000 & .000 & .000 & .000 & .200 & .200 & .200 & .200 & 22.20 & 44,397 & \cmark \\
\textbf{0.67} & \textbf{.000} & \textbf{.026} & \textbf{.000} & \textbf{.000} & \textbf{.000} & \textbf{.000} & \textbf{.243} & \textbf{.243} & \textbf{.243} & \textbf{.243} & \textbf{26.81} & \textbf{44,043} & ${\eta^*}$ \\
\midrule
0.69 & \multicolumn{10}{c|}{infeasible} & -- & -- & \xmark \\
0.90 & \multicolumn{10}{c|}{infeasible} & -- & -- & \xmark \\
\bottomrule
\end{tabular}
}
\end{table}

\textbf{Setting B: Heterogeneous thresholds.} The four protected classes have different thresholds reflecting varying preservation priorities: $\theta_{\mathrm{frog}} = 0.05$, $\theta_{\mathrm{horse}} = 0.08$, $\theta_{\mathrm{ship}} = 0.10$, $\theta_{\mathrm{truck}} = 0.15$. And the  $c = (1, 1, 11, 12, 13, 16, 30, 28, 27, 25)$ keep same.

\begin{table}[th]
\centering
\caption{Partial results of searching over 50 evenly spaced points in $\eta \in [0.2, 0.9]$ (Setting B: protect frog/horse/ship/truck,$\theta = (0.05, 0.08, 0.10, 0.15)$,
$c = (1, 1, 11, 12, 13, 16, 30, 28, 27, 25)$).}
\label{tab:alg1-full-B}
\small
\resizebox{\linewidth}{!}{
\begin{tabular}{c|cccccccccc|rr|c}
\toprule
$\eta$ & air & auto & bird & cat & deer & dog & frog & horse & ship & truck
  & $C^*$ & $J$ & status \\
\midrule
0.20 & .000 & .525 & .000 & .000 & .000 & .000 & .062 & .100 & .125 & .187 & 13.26 & 53,030 & \cmark \\
0.30 & .000 & .457 & .000 & .000 & .000 & .000 & .071 & .114 & .143 & .214 & 15.01 & 52,534 & \cmark \\
0.40 & .000 & .367 & .000 & .000 & .000 & .000 & .083 & .133 & .167 & .250 & 17.35 & 52,039 & \cmark \\
0.50 & .000 & .240 & .000 & .000 & .000 & .000 & .100 & .160 & .200 & .300 & 20.62 & 51,544 & \cmark \\
0.60 & .000 & .050 & .000 & .000 & .000 & .000 & .125 & .200 & .250 & .375 & 25.52 & 51,049 & \cmark \\
\textbf{0.61} & \textbf{.000} & \textbf{.015} & \textbf{.000} & \textbf{.000} & \textbf{.000} & \textbf{.000} & \textbf{.130} & \textbf{.207} & \textbf{.259} & \textbf{.389} & \textbf{26.43} & \textbf{50,978} & $\eta^*$ \\
\midrule
0.63 & \multicolumn{10}{c|}{infeasible} & -- & -- & \xmark \\
0.90 & \multicolumn{10}{c|}{infeasible} & -- & -- & \xmark \\
\bottomrule
\end{tabular}
}
\end{table}

Table~\ref{tab:alg1-full-B} shows the partial  LP search results.  The feasibility boundary is at $\eta^* = 0.61$. The heterogeneous thresholds shift the feasibility boundary and change the optimal allocation: the LP assigns more reference mass to classes with higher thresholds (ship, truck) and less to those with lower thresholds (frog, horse). Table~\ref{subtab:alg2-compare} compares our designed reference with uniform reference  and the actual performance of the reference data is shown in the Fig.~\ref{SubFig:cfm_pre2}.

\begin{figure}[th]
\centering
\begin{minipage}[c]{0.67\textwidth}
\centering
  \captionof{table}{
    Comparison of reference data strategies.
    $C^*$: per-sample cost; $J$: total cost $(1-\eta) \cdot N \cdot C^*$.
    Last four columns: proportion of each protected class at round 20. 
  }
  \label{subtab:alg2-compare}
  \small
  \begin{tabular}{llcccccc}
    \toprule
    $p_{\mathrm{ref}}$ & $\eta$ & $C^*$ & $J$
      & frog & horse & ship & truck \\
    \midrule
    Uniform & 0.50 & 16.4 & 41,000
      & .067 & .077 & .059 & .067 \\
    Uniform & 0.20 & 16.4 & 65,600
      & .112 & .113 & .097 & .105\\
    \textbf{Ours} & \textbf{0.61} & \textbf{26.43} & \textbf{50,978}
      & \textbf{.059} & \textbf{.101} & \textbf{.116} & \textbf{.168} \\
    \bottomrule
  \end{tabular}
\end{minipage}
\hfill
\begin{minipage}[c]{0.28\textwidth}
    \centering
    \includegraphics[width=\linewidth]{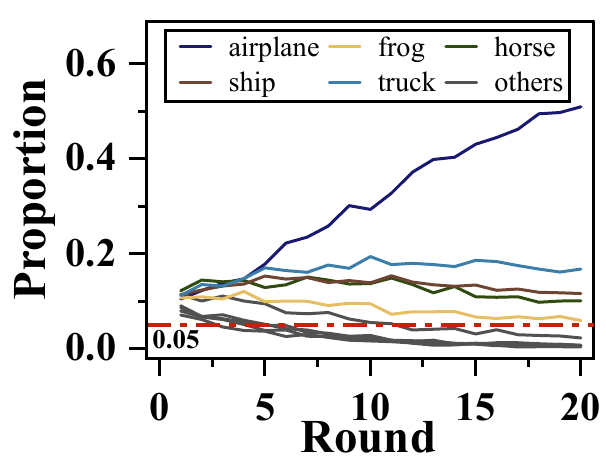}
    \caption{Class proportions under designed reference.}
    \label{SubFig:cfm_pre2}
\end{minipage}
\end{figure}

\subsubsection{Generated Samples}
To complement the quantitative results, we provide generated samples at selected rounds in Fig.~\ref{Appfig:cfm_airplane_samples} and Fig.~\ref{Appfig:cfm_cat_samples}, offering a visual illustration of how the model's output evolves under different reference data strategies.

\begin{figure}[htbp]
    \centering
     \begin{subfigure}[c]{0.1\linewidth}
        \centering
    \end{subfigure}
    \hfill
    \begin{subfigure}[c]{0.21\linewidth}
        \centering \textbf{Round 1}
    \end{subfigure}
    \begin{subfigure}[c]{0.21\linewidth}
        \centering \textbf{Round 7}
    \end{subfigure}
    \begin{subfigure}[c]{0.21\linewidth}
        \centering \textbf{Round 14}
    \end{subfigure}
    \begin{subfigure}[c]{0.21\linewidth}
        \centering \textbf{Round 20}
    \end{subfigure}

    \begin{subfigure}[c]{0.1\linewidth}
        \centering \textbf{Pure $\eta=1.0$}
    \end{subfigure}
    \hfill
    \begin{subfigure}[c]{0.21\linewidth}
        \centering
        \includegraphics[width=\linewidth]{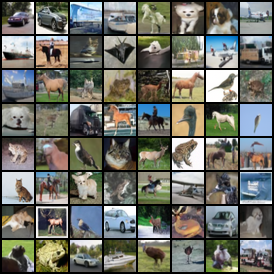}
    \end{subfigure}
    \begin{subfigure}[c]{0.21\linewidth}
        \centering
        \includegraphics[width=\linewidth]{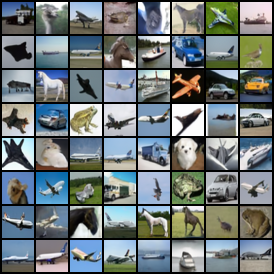}
    \end{subfigure}
    \begin{subfigure}[c]{0.21\linewidth}
        \centering
        \includegraphics[width=\linewidth]{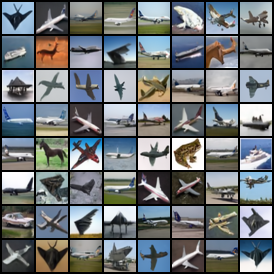}
    \end{subfigure}
    \begin{subfigure}[c]{0.21\linewidth}
        \centering
        \includegraphics[width=\linewidth]{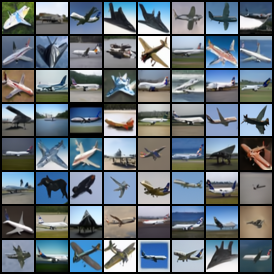}
    \end{subfigure}
    
    \begin{subfigure}[c]{0.1\linewidth}
        \centering \textbf{Uniform reference $\eta=0.5$}
    \end{subfigure}  
    \hfill
    \begin{subfigure}[c]{0.21\linewidth}
        \centering
        \includegraphics[width=\linewidth]{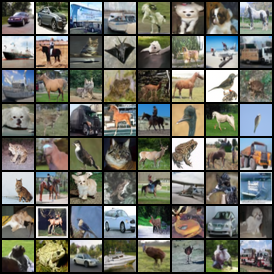}
    \end{subfigure}
    \begin{subfigure}[c]{0.21\linewidth}
        \centering
        \includegraphics[width=\linewidth]{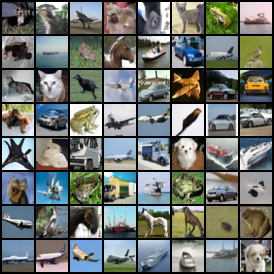}
    \end{subfigure}
    \begin{subfigure}[c]{0.21\linewidth}
        \centering
        \includegraphics[width=\linewidth]{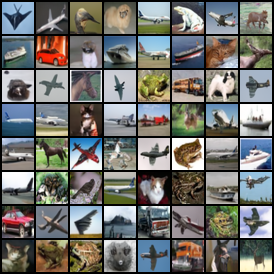}
    \end{subfigure}
    \begin{subfigure}[c]{0.21\linewidth}
        \centering
        \includegraphics[width=\linewidth]{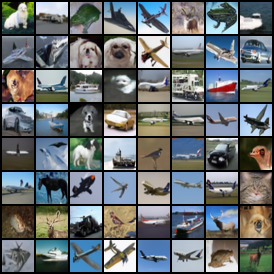}
    \end{subfigure}

    \begin{subfigure}[c]{0.1\linewidth}
        \centering  \textbf{Uniform reference $\eta=0.1$}
    \end{subfigure}
    \hfill
        \begin{subfigure}[c]{0.21\linewidth}
        \centering
        \includegraphics[width=\linewidth]{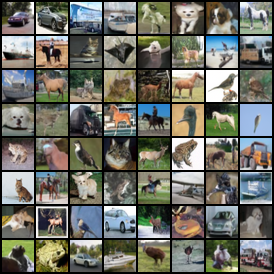}
    \end{subfigure}
    \begin{subfigure}[c]{0.21\linewidth}
        \centering
        \includegraphics[width=\linewidth]{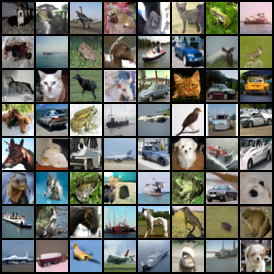}
    \end{subfigure}
    \begin{subfigure}[c]{0.21\linewidth}
        \centering
        \includegraphics[width=\linewidth]{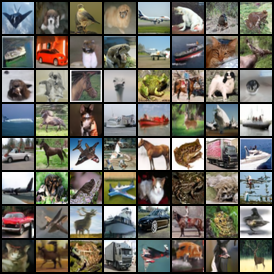}
    \end{subfigure}
    \begin{subfigure}[c]{0.21\linewidth}
        \centering
        \includegraphics[width=\linewidth]{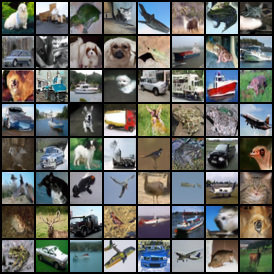}
    \end{subfigure}

    \begin{subfigure}[c]{0.1\linewidth}
        \centering \textbf{Designed reference $\eta=0.67$}
    \end{subfigure}
    \hfill
        \begin{subfigure}[c]{0.21\linewidth}
        \centering
        \includegraphics[width=\linewidth]{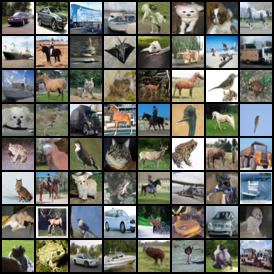}
    \end{subfigure}
    \begin{subfigure}[c]{0.21\linewidth}
        \centering
        \includegraphics[width=\linewidth]{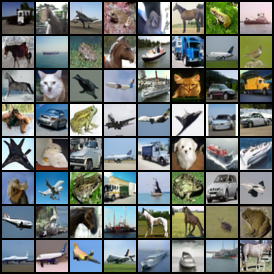}
    \end{subfigure}
    \begin{subfigure}[c]{0.21\linewidth}
        \centering
        \includegraphics[width=\linewidth]{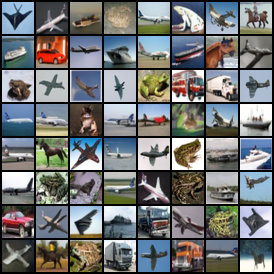}
    \end{subfigure}
    \begin{subfigure}[c]{0.21\linewidth}
        \centering
        \includegraphics[width=\linewidth]{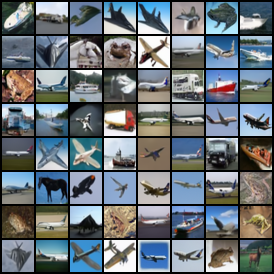}
    \end{subfigure}

    \caption{Generated samples from OT-CFM under $w_0 = [1,0,0,0,\dots]$ (airplane preference), $\tau=5.0$, $\beta=0.01$. Each row corresponds to a different reference data strategy; each column shows samples at the indicated round. Under pure synthetic training ($\eta=1.0$), the model collapses to airplane images. Mixing with uniform reference data increasingly preserves diversity. The LP-designed reference with $\theta_{\ell}=0.08$ for last four classes (bottom row) maintains diversity at $\eta=0.67$. }
    \label{Appfig:cfm_airplane_samples}
\end{figure}

\begin{figure}[htbp]
    \centering
     \begin{subfigure}[c]{0.1\linewidth}
        \centering
    \end{subfigure}
    \hfill
    \begin{subfigure}[c]{0.21\linewidth}
        \centering \textbf{Round 1}
    \end{subfigure}
    \begin{subfigure}[c]{0.21\linewidth}
        \centering \textbf{Round 7}
    \end{subfigure}
    \begin{subfigure}[c]{0.21\linewidth}
        \centering \textbf{Round 14}
    \end{subfigure}
    \begin{subfigure}[c]{0.21\linewidth}
        \centering \textbf{Round 20}
    \end{subfigure}

    \begin{subfigure}[c]{0.1\linewidth}
        \centering \textbf{Pure $\eta=1.0$}
    \end{subfigure}
    \hfill
    \begin{subfigure}[c]{0.21\linewidth}
        \centering
        \includegraphics[width=\linewidth]{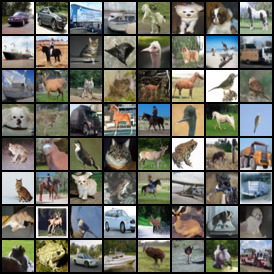}
    \end{subfigure}
    \begin{subfigure}[c]{0.21\linewidth}
        \centering
        \includegraphics[width=\linewidth]{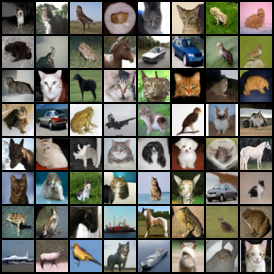}
    \end{subfigure}
    \begin{subfigure}[c]{0.21\linewidth}
        \centering
        \includegraphics[width=\linewidth]{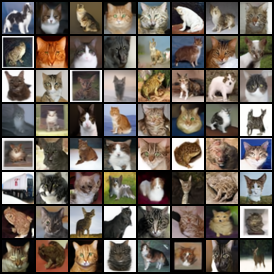}
    \end{subfigure}
    \begin{subfigure}[c]{0.21\linewidth}
        \centering
        \includegraphics[width=\linewidth]{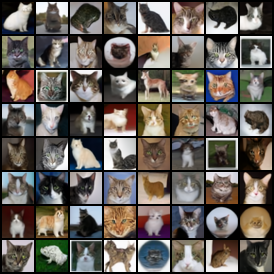}
    \end{subfigure}

    \begin{subfigure}[c]{0.1\linewidth}
        \centering  \textbf{Uniform reference $\eta=0.1$}
    \end{subfigure}
    \hfill
        \begin{subfigure}[c]{0.21\linewidth}
        \centering
        \includegraphics[width=\linewidth]{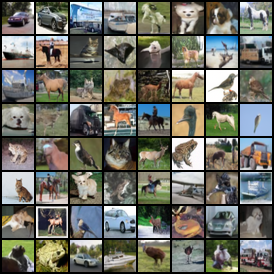}
    \end{subfigure}
    \begin{subfigure}[c]{0.21\linewidth}
        \centering
        \includegraphics[width=\linewidth]{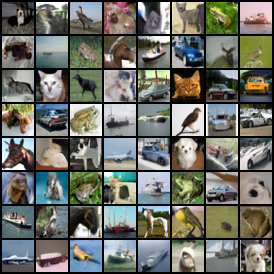}
    \end{subfigure}
    \begin{subfigure}[c]{0.21\linewidth}
        \centering
        \includegraphics[width=\linewidth]{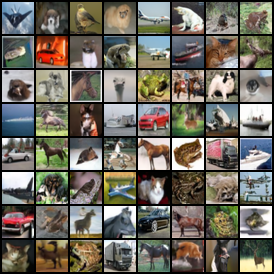}
    \end{subfigure}
    \begin{subfigure}[c]{0.21\linewidth}
        \centering
        \includegraphics[width=\linewidth]{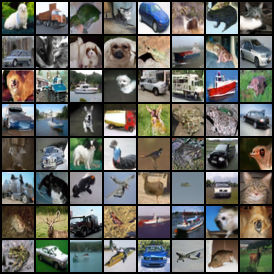}
    \end{subfigure}

    \caption{Generated samples from OT-CFM under $w_0 = [0,0,0,1,\dots]$ (cat preference), $\tau=5.0$, $\beta=0.01$. Each row corresponds to a different reference data strategy; each column shows samples at the indicated round. Under pure synthetic training ($\eta=1.0$), the model collapses to airplane images. Mixing with uniform reference data increasingly preserves diversity.}
    \label{Appfig:cfm_cat_samples}
\end{figure}

\subsection{DDPM Experiments on CIFAR-10}
\label{subsec:ddpm}
\textbf{Setup.} The model is DDPM~\citep{ddpm,ImprovedDDpm} with the VLB objective (50M parameter UNet), initialized from the released pretrained checkpoint on CIFAR-10. The feature map $\varphi(x) \in \mathbb{R}^{10}$ is the normalized softmax output of a pretrained ResNet-56 classifier (accuracy $\geq 93\%$), where $\varphi(x)_c$ represents the predicted probability of class $c$. Each round generates 20,000 candidate images, curates 10,000 via $K$-way Luce selection, mixes with real CIFAR-10 at ratio $\eta$ to form a training set of size 10,000, and fine-tunes for 1,000 steps. After training, 10,000 fresh samples are generated to compute per-class proportions and update $w_t$. We run each experiment for 10 rounds with $\tau = 5.0$, $K = \infty$, and $\beta = 0.01$.

\textbf{Co-evolution dynamics.} Fig.~\ref{SubFig:ddpm-main} shows results under two initializations
and two mixing ratios. Under $\eta = 1.0$, DDPM exhibits the same initialization-dependent lock-in as CFM: $w_0 = [1,0,0,0,\dots]$ drives concentration toward airplane and $w_0 = [0,0,0,1,0]$ toward cats. Under $\eta = 0.1$, both initializations converge to a similar near-uniform equilibrium, consistent with the uniqueness result. Compared to CFM, DDPM shows slower concentration under $\eta = 1.0$, reaching 60\% by round 10. This is expected: DDPM suffers from greater image quality degradation over successive retraining rounds, causing the classifier to assign less confident predictions. More broadly, the theoretical model assumes the retrained distribution exactly matches the training mixture, while in practice finite capacity and training error introduce deviations. Nevertheless, the qualitative predictions of the theory hold: lock-in occurs under pure synthetic training and stabilization occurs under mixture with reference data.

\begin{figure}[ht]
\centering
\includegraphics[width=\textwidth]{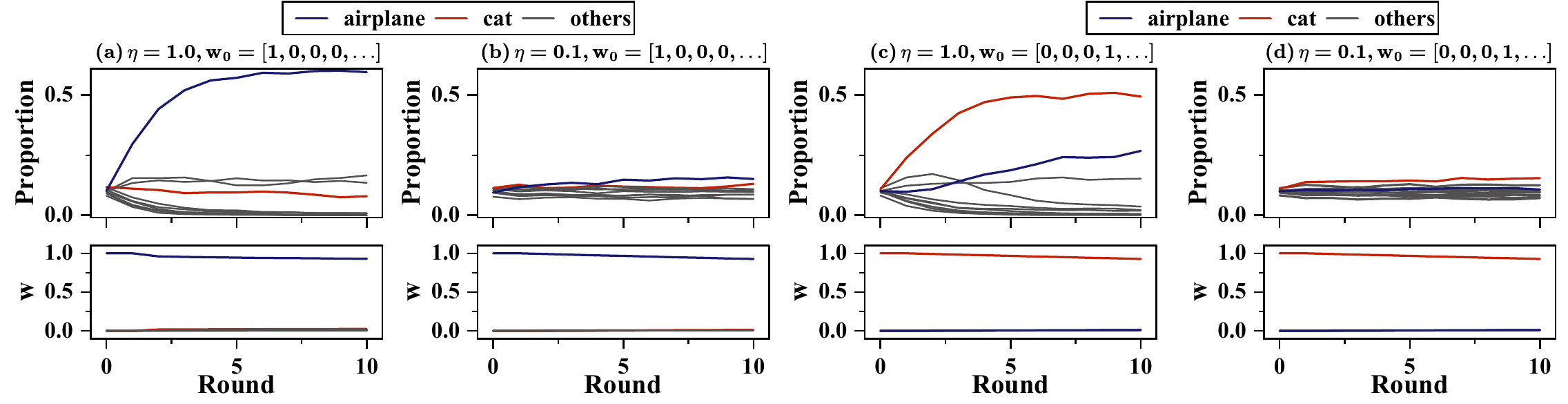}
\caption{Co-evolution dynamics under different initializations and mixing ratios for DDPM model. \textbf{Top}: per-class generation proportions. \textbf{Bottom}: preference vector $w_t$.
}
\label{SubFig:ddpm-main}
\end{figure}

\textbf{Reproducibility across seeds.} To verify that the observed dynamics are not artifacts of a particular random seed, we repeat the experiment ($\eta = 1.0$, $w_0 = [1,0,0,0,\dots]$) with three independent seeds. Fig.~\ref{fig:ddpm-seeds} shows that all three runs produce consistent trajectories: airplane proportion reaches 58.29\%-59.02\% by round 10 and the preference vector follows a similar path. The small variation across seeds confirms that the co-evolution dynamics are driven by the systematic interaction between curation and retraining, rather than by stochastic fluctuations.

\begin{figure}[ht]
\centering
\includegraphics[width=\textwidth]{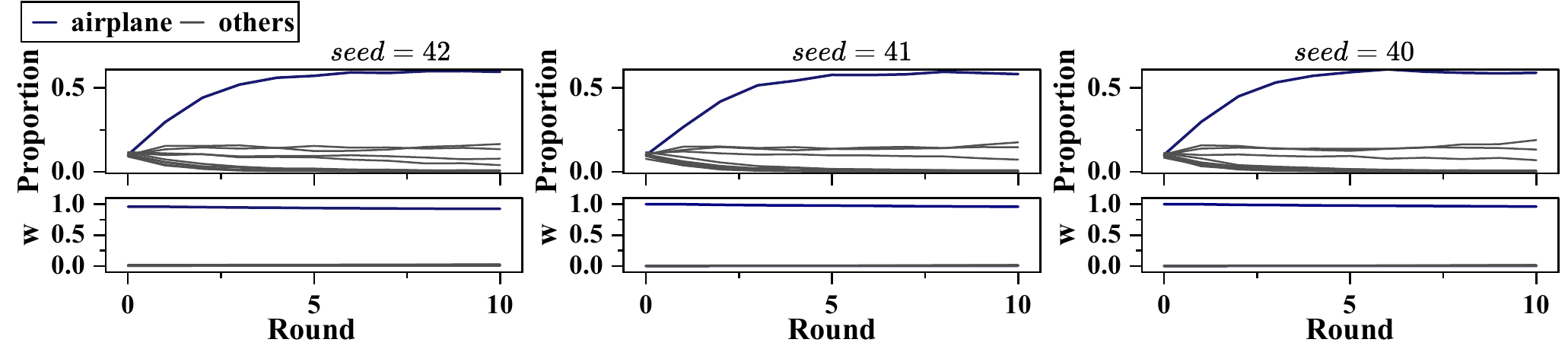}
\caption{
  Reproducibility of DDPM co-evolution across three independent seeds ($\eta = 1.0$, $w_0 = [1,0,0,0,\dots]$, $\tau = 5.0$, $K = \infty$, $\beta = 0.01$).
}
\label{fig:ddpm-seeds}
\end{figure}

\textbf{Generated Samples.} Fig.~\ref{Appfig:ddpm_airplane_samples} shows the corresponding generated samples for DDPM. Compared to OT-CFM, DDPM exhibits noticeable image quality degradation over successive rounds: textures become blurry and fine details are lost, reflecting the accumulation of training error across generations. This degradation is most pronounced under $\eta = 1.0$, where the model retrains entirely on its own imperfect outputs. Nevertheless, the same structural patterns persist: concentration toward a single class under pure synthetic training, and preserved diversity under reference data mixing.

\begin{figure}[htbp]
    \centering
     \begin{subfigure}[c]{0.1\linewidth}
        \centering
    \end{subfigure}
    \hfill
    \begin{subfigure}[c]{0.21\linewidth}
        \centering \textbf{Round 1}
    \end{subfigure}
    \begin{subfigure}[c]{0.21\linewidth}
        \centering \textbf{Round 7}
    \end{subfigure}
    \begin{subfigure}[c]{0.21\linewidth}
        \centering \textbf{Round 14}
    \end{subfigure}
    \begin{subfigure}[c]{0.21\linewidth}
        \centering \textbf{Round 20}
    \end{subfigure}

    \begin{subfigure}[c]{0.1\linewidth}
        \centering \textbf{Pure $\eta=1.0$}
    \end{subfigure}
    \hfill
    \begin{subfigure}[c]{0.21\linewidth}
        \centering
        \includegraphics[width=\linewidth]{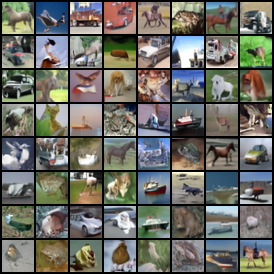}
    \end{subfigure}
    \begin{subfigure}[c]{0.21\linewidth}
        \centering
        \includegraphics[width=\linewidth]{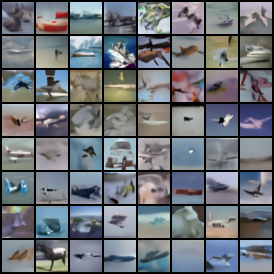}
    \end{subfigure}
    \begin{subfigure}[c]{0.21\linewidth}
        \centering
        \includegraphics[width=\linewidth]{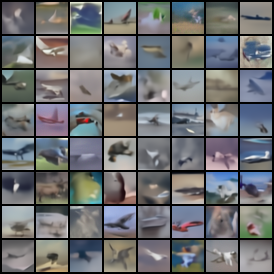}
    \end{subfigure}
    \begin{subfigure}[c]{0.21\linewidth}
        \centering
        \includegraphics[width=\linewidth]{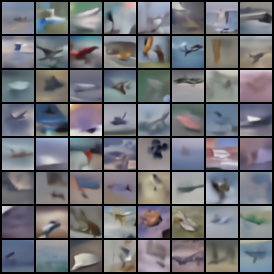}
    \end{subfigure}

    \begin{subfigure}[c]{0.1\linewidth}
        \centering  \textbf{Uniform reference $\eta=0.1$}
    \end{subfigure}
    \hfill
        \begin{subfigure}[c]{0.21\linewidth}
        \centering
        \includegraphics[width=\linewidth]{Figs/ddpm/airplane_samples/eta0.1_round0.png}
    \end{subfigure}
    \begin{subfigure}[c]{0.21\linewidth}
        \centering
        \includegraphics[width=\linewidth]{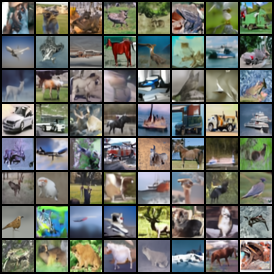}
    \end{subfigure}
    \begin{subfigure}[c]{0.21\linewidth}
        \centering
        \includegraphics[width=\linewidth]{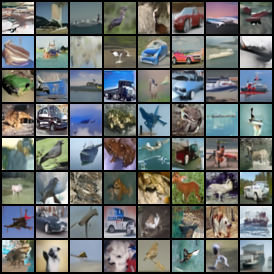}
    \end{subfigure}
    \begin{subfigure}[c]{0.21\linewidth}
        \centering
        \includegraphics[width=\linewidth]{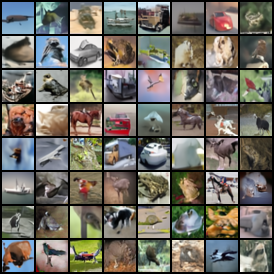}
    \end{subfigure}

    \caption{Generated samples from DDPM under $w_0 = [1,0,0,0,\dots]$ (airplane preference), $\tau=5.0$, $\beta=0.01$. Each row corresponds to a different reference data strategy; each column shows samples at the indicated round. Under pure synthetic training ($\eta=1.0$), the model collapses to airplane images. Mixing with uniform reference data increasingly preserves diversity.}
    \label{Appfig:ddpm_airplane_samples}
\end{figure}

\begin{figure}[thbp!]
\centering
\begin{minipage}{\textwidth}
    \centering
    \includegraphics[width=\textwidth]{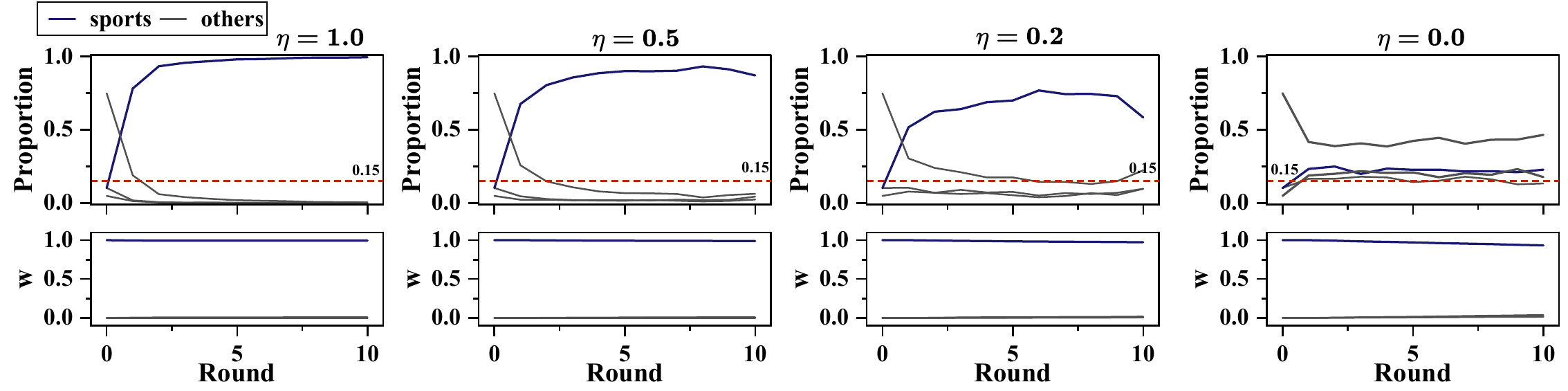}
    \caption{ GPT-2 on AG~News with $w_0 = [0,1,0,0]$ (Sports). \textit{Top}: per-class proportions. \textit{Bottom}: preference vector $w_t$.}
    \label{fig:text-sports}
\end{minipage}
\vspace{4pt}

\begin{minipage}{\textwidth}
    \centering
    \includegraphics[width=\textwidth]{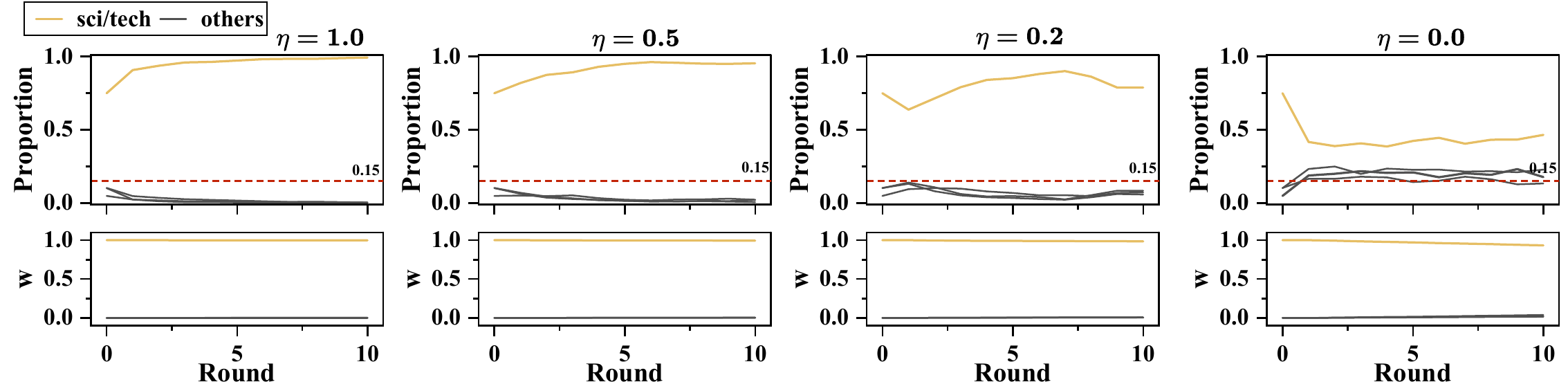}
    \caption{
  GPT-2 on AG~News with $w_0 = [0,0,0,1]$ (SciTech). \textit{Top}: per-class proportions. \textit{Bottom}: preference vector $w_t$.}
\label{fig:text-scitech}
\end{minipage}
\end{figure}

\subsection{GPT-2 Experiments on AG~News}
\label{subsec:cpt2}
We replicate the co-evolution experiments on text generation.

\textbf{Setup.} The model used is GPT-2-small (124M parameters)~\citep{gpt2}, initialized from the HuggingFace pretrained checkpoint. The feature map $\varphi(x) \in \mathbb{R}^4$ is the normalized softmax output of a pretrained BERT~\citep{bert} classifier (accuracy $\geq 94\%$), where $\varphi(x)_c$ represents the predicted probability of class $c \in \{\text{world}, \text{sports}, \text{business}, \text{sciTech}\}$.

\textbf{Iterative retraining.} Each round generates 5,000 candidate texts, curates 2,500 via $K$-way Luce selection, mixes with real AG~News at ratio $\eta$ to form a training set of size 2,500, and fine-tunes for 1,000 gradient steps. After training, 2,000 fresh samples are generated to compute per-class proportions and update $w_t$. We run each experiment for 10 rounds with $\tau = 5.0$, $K = \infty$, and $\beta = 0.01$.

\textbf{Co-evolution dynamics.} Fig.~\ref{fig:text-sports} shows results with $w_0 = [0,1,0,0]$ (Sports) under four mixing ratios. At $\eta = 1.0$, the Sports proportion exceeds 90\% by round 10, mirroring the lock-in observed in the image experiments. As $\eta$ decreases, reference data stabilizes the distribution. Fig.\ref{fig:text-scitech} shows the same sweep with $w_0 = [0,0,0,1]$ (SciTech). The dynamics are qualitatively identical:
the system locks into SciTech under $\eta = 1.0$ and stabilizes under small $\eta$. This confirms that the initialization-dependent lock-in and the stabilizing effect
of reference data hold across both modalities and class structures.

Unlike the image models, GPT-2 deviates more noticeably from the idealized assumptions of our theory. Even when trained entirely on uniformly distributed reference data ($\eta = 0.0$), the generated class proportions are not uniform, indicating that GPT-2 is far from a perfect learner: the retrained distribution does not exactly match the training mixture. In addition, the class proportions fluctuate across rounds rather than converging smoothly, likely due to the instability of fine-tuning a pretrained language model on small datasets. Despite these deviations, the qualitative theoretical predictions remain: lock-in under pure synthetic training and stabilization under reference data mixing.

\paragraph{Reference data design.}

We apply Alg.~\ref{alg:ref-design} to the AG~News setting.  We protect the class world with threshold $\theta = 0.15$. To account for the gap between the theoretical model and GPT-2's imperfect learning, we calibrate the threshold used in the LP: since GPT-2 achieves roughly $37.5\%$ of the theoretically predicted level under uniform reference data, we scale the LP input threshold to $\theta_{LP} \approx \frac{0.15}{0.375}=0.4$,  so that the actual achieved proportion approximates the desired $\theta = 0.15$. The costs are $c = (20, 5, 5, 5)$, reflecting that World news requires more effort to source than the other categories.

Table.~\ref{subtab:alg2-text_compare} compares our designed reference with uniform reference and the actual performance of the reference data is shown in the Fig.~\ref{SubFig:text_pref}.

\begin{figure}[thbp!]
\centering
\begin{minipage}[c]{0.67\textwidth}
\centering
  \captionof{table}{
    Comparison of reference data strategies.
    $C^*$: per-sample cost; $J$: total cost $(1-\eta) \cdot N \cdot C^*$.
    Last columns: proportion of the protected class at round 20. 
  }
  \label{subtab:alg2-text_compare}
  \small
  \begin{tabular}{llcccccc}
    \toprule
    $p_{\mathrm{ref}}$ & $\eta$ & $C^*$ & $J$
      & world \\
    \midrule
    Uniform & 0.50 & 8.75 & 10,937.50
      & .1211 \\
    Uniform & 0.10 & 8.75 & 19,687.50
      & 0.12 \\
    \textbf{Ours} & \textbf{0.10} & \textbf{11.67} & \textbf{26,250.00}
      & \textbf{.1404} \\
    \bottomrule
  \end{tabular}
\end{minipage}
\hfill
\begin{minipage}[c]{0.3\textwidth}
    \centering
    \includegraphics[width=\linewidth]{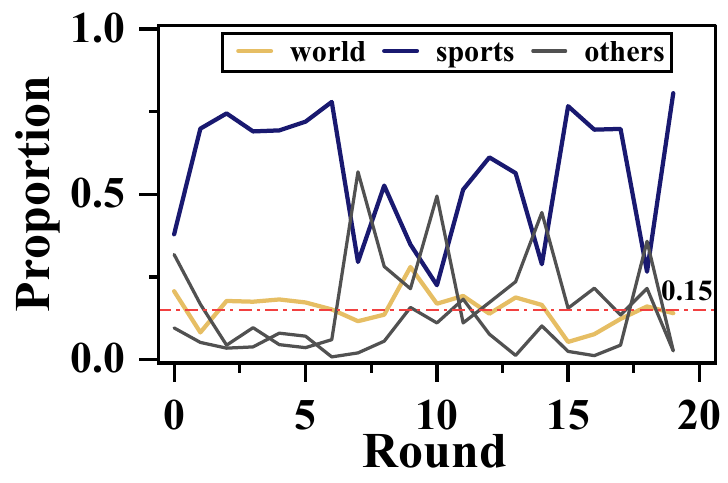}
    \caption{Class proportions under designed reference.}
    \label{SubFig:text_pref}
\end{minipage}
\end{figure}

\section{Discussion and Extensions}
\label{sec:discussion}
Our framework focuses on a basic co-evolutionary setting in which a single aggregate preference state interacts with a self-consuming generative model through preference-guided curation. Building on this
framework, several natural extensions arise by enriching the preference dynamics and population structure while retaining the same closed-loop interaction between user feedback and generative retraining. We discuss three such directions below.

\subsection{Nonlinear Preference Responses}
\label{app:nonlinear_preference_dynamics}
The preference update in Eq.~\ref{eq:pref-update} moves the current preference toward the mean feature vector of the updated model through linear averaging. This form is related to first-order models of preference and opinion formation. More generally, users may respond nonlinearly to the discrepancy between their current
preference and the content to which they are exposed. We consider such responses while retaining the reward
$r_w(x)=\langle w,\varphi(x)\rangle$, the curation operator, and the model retraining step in Eq.~\ref{eq:data-update}.

\textit{A gap-response formulation.}
Let $\Delta_t:=\bar\varphi(p_{t+1})-w_t$ denote the mean feature vector after retraining and its discrepancy from the current preference. For a response function $F:\mathbb{R}^d\to\mathbb{R}^d$, define
\begin{equation}
    w_{t+1}
    =G_F(w_t,\bar\varphi(p_{t+1}))
    :=\operatorname{Proj}_W\!
      \left[w_t+\beta F(\bar\varphi(p_{t+1})-w_t)\right].
    \label{eq:app_gap_response}
\end{equation}
The linear update is recovered by $F(\Delta)=\Delta$. We additionally consider a saturating response and a state-dependent scalar gate. 

\textit{Saturating response.} For $F(\Delta)=\tanh(g\Delta)$ with $g>0$ and $\tanh$ applied componentwise,
the update is
\begin{equation}
    w_{t+1}
    =\operatorname{Proj}_W\!\left[
        w_t+\beta\tanh\!\left(
            g\bigl(\bar\varphi(p_{t+1})-w_t\bigr)
        \right)
    \right]
    \label{eq:app_tanh_response}
\end{equation}
This response is smooth and saturates as the feature--preference gap grows, limiting the magnitude of the response in each coordinate.

\textit{Scalar-gated response.}
Let $\gamma_t:=\gamma\bigl(w_t,\bar\varphi(p_{t+1})\bigr)\in(0,1]$ be a state-dependent response weight. Using
$F_t(\Delta_t)=\gamma_t\Delta_t$ in the same update template gives
\begin{equation}
    w_{t+1}
    =\operatorname{Proj}_W\!\left[
        (1-\beta\gamma_t)w_t
        +\beta\gamma_t\bar\varphi(p_{t+1})
    \right].
    \label{eq:app_gated_response}
\end{equation}
The gate acts as a similarity or confidence weight. It changes the adaptation rate according to the current interaction between preferences and model outputs. In the discussion below, $G_F$ denotes the resulting preference-update map, including this state-dependent gated form.

\textbf{Components of the analysis that remain unchanged.} Changing the preference response leaves the model-update branch unchanged. Consequently, the curation properties in Prop.~\ref{prop:pure-basic}, the retraining sensitivity constants $L_p$ and $L_w$ in Lem~\ref{lem:cont-sensitivity}, and the equal-reward characterization of fixed points of the retraining step at a given preference in Thm.~\ref{thm:pure-structure} continue to apply. These components depend on the reward and feature representation, rather than
on the rule used to update preferences.

The singleton equilibria also remain fixed whenever the response vanishes at a zero gap. Indeed, at $p^\star=\delta_{x_i}$, $w^\star=\varphi(x_i)$, we have $\bar\varphi(p^\star)-w^\star=0$. Thus $F(0)=0$ implies that the preference update leaves $w^\star$ unchanged, while the model update is already fixed in the purely synthetic regime. All three response families above vanish at a zero gap.

\textbf{Response-dependent sensitivity and coupled dynamics.}
The response function enters the joint analysis through the preference coordinate. Suppose that the resulting map satisfies
$$
    \bigl\|G_F(w,m)-G_F(v,\widetilde m)\bigr\|
    \leq a_F\|w-v\|+b_F\|m-\widetilde m\|
$$
for admissible preference vectors and mean feature vectors, with $a_F,b_F\geq0$. For two trajectories $(p_t,w_t)$ and $(q_t,v_t)$, the preference-coordinate estimate becomes
\begin{equation*}
\begin{aligned}
    \|w_{t+1}-v_{t+1}\|
    &\leq a_F\|w_t-v_t\|
       +b_F\bigl\|\bar\varphi(p_{t+1})
                    -\bar\varphi(q_{t+1})\bigr\|
\end{aligned}
\end{equation*}
For the linear response, nonexpansiveness of projection gives $(a_F,b_F)=(1-\beta,\beta)$, recovering the estimate used in the proof of Thm.~3.9. For another response, the same proof strategy replaces $(1-\beta,\beta)$ with $(a_F,b_F)$ and combines the corresponding preference-coordinate estimate with the unchanged retraining bounds. The admissible weight in the product metric, the contraction condition, and the convergence rate consequently depend on the response. Quantitative lock-in neighborhoods likewise require response-specific control of the preference update. This formulation identifies the part of the analysis that must be adapted, while preserving the curation and retraining components. Deriving explicit bounds for particular nonlinear responses is a direction for further analysis.

\subsection{Stochastic Preference Adaptation}
Another natural extension is to model preference changes at the individual level as stochastic rather than deterministic. For example, suppose a user updates toward the feature vector of a randomly observed
model output, $X_{t+1}\sim p_{t+1}$, $w_{t+1}= \operatorname{Proj}_W\left[(1-\beta)w_t+\beta\varphi(X_{t+1})\right]$. Under the assumptions $\|w_t\|\leq 1$, $\|\varphi(x)\|=1$, and
$\beta\in[0,1]$, the point before projection already lies in $W$.
Therefore,
\begin{equation*}
    \mathbb{E}
    \left[
        w_{t+1}
        \mid
        p_{t+1},w_t
    \right]
    =
    (1-\beta)w_t
    +
    \beta\bar{\varphi}(p_{t+1}),
\end{equation*}
which recovers Eq.~\ref{eq:pref-update} as the conditional-mean dynamics. This provides a mean-field interpretation of the deterministic preference update studied in the paper.

The stochastic formulation raises additional questions that do not arise at the mean level. In particular, one may ask whether stochastic trajectories concentrate around the deterministic dynamics, how random
exposure affects the probability of entering different lock-in basins under pure self-consumption, and whether reference mixing induces stochastic stability or a concentrated stationary distribution around
the deterministic equilibrium. More general noise models, including exogenous perturbations to the preference state, would further require quantifying how random fluctuations propagate through the coupled model--preference feedback loop.

\subsection{Heterogeneous User Populations} The current framework represents the population by a single aggregate
preference vector. A natural extension is to consider $G$ user groups with group-specific preferences
$\{w_{g,t}\}_{g=1}^G$ and population weights $\{\alpha_g\}_{g=1}^G$, where $\alpha_g\geq 0$ and $\sum_{g=1}^G\alpha_g=1$. One possible model replaces the single curated distribution by a mixture
of group-specific curation responses,
\begin{equation*}
    p_{t+1}(x)
    =
    (1-\eta)p_{\mathrm{ref}}(x)
    +
    \eta p_t(x)
    \sum_{g=1}^G
    \alpha_g
    H^K_{p_t,w_{g,t}}(x),
\end{equation*}
while each group updates according to
\begin{equation*}
    w_{g,t+1}
    =
    \operatorname{Proj}_W
    \left[
        (1-\beta_g)w_{g,t}
        +
        \beta_g\bar{\varphi}(p_{t+1})
    \right].
    \label{eq:discussion_heterogeneous_pref}
\end{equation*}

This extension introduces qualitatively new interactions because all groups affect, and are simultaneously affected by, the same evolving generative distribution. It therefore raises questions about whether group preferences converge or remain separated, how the feedback of one group indirectly changes the preferences of others through the shared model, and how reference-data design should be adapted when different groups impose distinct preservation objectives. Characterizing the equilibria and stability of such multi-group dynamics, as well as designing reference distributions that balance group-specific objectives, is a promising direction for future work.

\section{Broader impacts}
This work studies feedback loops in self-consuming generative systems, where model outputs are curated by users and reused for future training, while user preferences also evolve over time. This work can help diagnose and mitigate long-term instability, preference drift, and diversity collapse by our theoretical and empirical tools. However, the same steering mechanism is value-neutral with respect to the choice of preserved attributes, and could in principle be used to filter out specific content. Thus, transparent selection of preserved attributes is important in any deployment.


\end{document}